\documentclass{article}

\usepackage[preprint]{neurips_2022}

\usepackage[utf8]{inputenc} 
\usepackage[T1]{fontenc}    
\usepackage{hyperref}       
\usepackage{url}            
\usepackage{booktabs}       
\usepackage{amsfonts}       
\usepackage{nicefrac}       
\usepackage{microtype}      
\usepackage{xcolor}         

\usepackage{graphicx}
\usepackage{booktabs}
\usepackage[font={small,it}]{caption}

\newcommand{\titre}{Contextual Bandits with Knapsacks \\ for a Conversion Model}
\title{\titre}

\author{Zhen Li \\
BNP Paribas, 16 boulevard des Italiens, 75009 Paris, France \\
\texttt{zhen.li@bnpparibas.com} \\
\And
Gilles Stoltz \\
Universit{\'e} Paris-Saclay, CNRS, Laboratoire de mathématiques d'Orsay, 91405, Orsay, France \\
\texttt{gilles.stoltz@universite-paris-saclay.fr} \\
HEC Paris, 1 rue de la Lib{\'e}ration, 78350 Jouy-en-Josas, France \\
\texttt{stoltz@hec.fr}
}

\usepackage{amsmath,amssymb,amsthm,mathtools}
\usepackage{dsfont}
\usepackage{enumitem}

\usepackage[most]{tcolorbox}
\newtcolorbox{nbox}[1][]{
  enhanced,
  fonttitle=\scshape,
  #1
}

\renewcommand{\leq}{\leqslant}
\renewcommand{\geq}{\geqslant}

\renewcommand{\phi}{\varphi}
\renewcommand{\epsilon}{\varepsilon}
\renewcommand{\hat}{\widehat}
\newcommand{\defeq}{\stackrel{\mbox{\scriptsize \rm def}}{=}}
\newcommand{\cA}{\mathcal{A}}

\newcommand{\cE}{\mathcal{E}}
\newcommand{\cF}{\mathcal{F}}

\newcommand{\cH}{\mathcal{H}}
\newcommand{\cL}{\mathcal{L}}

\newcommand{\cX}{\mathcal{X}}

\newcommand{\cP}{\mathcal{P}}
\newcommand{\E}{\mathbb{E}}

\newcommand{\R}{\mathbb{R}}
\newcommand{\btheta}{\boldsymbol{\theta}}
\newcommand{\bbeta}{\boldsymbol{\beta}}
\newcommand{\bmu}{\boldsymbol{\mu}}
\newcommand{\bx}{\boldsymbol{x}}
\newcommand{\bX}{\boldsymbol{X}}
\newcommand{\bc}{\boldsymbol{c}}
\newcommand{\bC}{\boldsymbol{C}}
\newcommand{\bL}{\boldsymbol{L}}
\newcommand{\bv}{\boldsymbol{v}}
\newcommand{\bu}{\boldsymbol{u}}
\newcommand{\bU}{\boldsymbol{U\!}}

\newcommand{\bp}{\boldsymbol{p}}
\newcommand{\bq}{\boldsymbol{q}}

\newcommand{\Ber}{\mathrm{Ber}}
\newcommand{\sig}{\eta}
\newcommand{\bphi}{\boldsymbol{\phi}}
\newcommand{\transp}{\!\mbox{\rm \tiny T}\,}
\newcommand{\e}{\mathrm{e}}
\newcommand{\anull}{a_{\mbox{\tiny null}}}
\newcommand{\nocost}{{\mbox{\tiny no-cost}}}

\newcommand{\opt}{\mathrm{\textsc{opt}}}

\renewcommand{\d}{\,\mathrm{d}}
\newcommand{\budg}{{\mbox{\rm \tiny budg}}}
\newcommand{\psum}{{\mbox{\rm \tiny p-sum}}}
\newcommand{\ppos}{{\mbox{\rm \tiny p-pos}}}
\newcommand{\Inotanull}[1]{\mathds{1}_{\{a_{#1} \not= \anull\}}}
\newcommand{\maxn}{}
\newcommand{\prob}{{\mbox{\rm \tiny prob}}}
\newcommand{\haz}{{\mbox{\rm \tiny HAz}}}
\newcommand{\cEnu}{\hat{\cE}_{{\mbox{\rm \tiny unif},\delta}}}
\newcommand{\am}{{\mbox{\rm \tiny am}}}
\newcommand{\dur}{{\mbox{\rm \tiny dur}}}
\newcommand{\ir}{\mathrm{ir}}
\newcommand{\out}{\mathrm{out}}
\newcommand{\outm}{{\mbox{\rm \tiny out}}}
\newcommand{\olr}{\overline{r}}
\newcommand{\olbc}{\overline{\bc}}
\newcommand{\olc}{\overline{c}}
\newcommand{\lin}{{\mbox{\rm \tiny lin}}}

\newcommand{\var}{\textit}
\newcommand{\proj}{\Pi_{\mbox{\rm \tiny unit}}}
\newcommand{\bzeta}{\boldsymbol{\zeta}}
\newcommand{\conv}{{\mbox{\rm \tiny conv}}}
\newcommand{\disc}{{\mbox{\rm \tiny disc}}}
\newcommand{\irs}{{\mbox{\rm \tiny ir}}}

\DeclareMathOperator{\id}{I}
\DeclareMathOperator*{\argmax}{argmax}
\DeclareMathOperator*{\argmin}{argmin}
\newtheorem{theorem}{Theorem}
\newtheorem{lemma}{Lemma}

\setlist{noitemsep,topsep=0pt,parsep=3.5pt,partopsep=0pt}

\begin{document}

\maketitle

\begin{abstract}
We consider contextual bandits with knapsacks, with an underlying structure
between rewards generated and cost vectors suffered.
We do so motivated by sales with commercial discounts.
At each round, given the stochastic i.i.d.\ context $\bx_t$ and the arm picked $a_t$ (corresponding, e.g., to a discount level),
a customer conversion may be obtained,
in which case a reward $r(a,\bx_t)$ is gained and vector costs $\bc(a_t,\bx_t)$ are suffered
(corresponding, e.g., to losses of earnings).
Otherwise, in the absence of a conversion, the reward and costs are null.
The reward and costs achieved are thus coupled through the binary variable measuring
conversion or the absence thereof.
This underlying structure between rewards and costs is different from the linear structures
considered by~\citet{Agrawal2016LinearCB} (but we show that the techniques introduced in the present article
may also be applied to the case of these linear structures). The adaptive policies exhibited
solve at each round a linear program based on upper-confidence estimates of the probabilities of conversion
given $a$ and $\bx$. This kind of policy is most natural and achieves a regret bound
of the typical order $(\opt/B) \smash{\sqrt{T}}$, where $B$ is the total budget allowed,
$\opt$ is the optimal expected reward achievable by a static policy, and
$T$ is the number of rounds.
\end{abstract}

\section{Introduction and Literature Review}

We consider the framework of stochastic multi-armed bandits,
which has been extensively studied since the early works by~\citet{Thompson1933MAB}
and~\citet{Robbins1952}. Two recent (and complementary) surveys summarizing the latest research
in the field were written by~\citet{Lattimore2020BA} and \citet{Aleksanders2019IntroBandits}.
On the one hand, we are particularly interested in the setting of \emph{contextual} stochastic multi-armed bandits,
preferably with some structural assumptions on the dependency between rewards and contexts:
linear models (again, a rich literature, see, among many others,
\citet{Chu2011ContextualBW} and~\citet{AbbasiYadkori2011ImprovedAF}, whose work marked
a turning point), and, for $[0,1]$-valued rewards, logistic models
(\citet{Filippi2010GLM} and \citet{Faury2020ImprovedOA}).
On the other hand, we are also particularly interested in stochastic multi-armed bandits
\emph{with knapsacks}, i.e., with cumulative vector-cost constraints to be abided by
on top of maximizing the accumulated rewards. The setting was introduced by
\citet{Badanidiyuru2013BanditsWK,Badanidiyuru2018BanditsWK} and a comprehensive
summary of the results achieved since then may be found in \citet[Chapter~10]{Aleksanders2019IntroBandits}.
The intersection of these two frameworks of interest is called
\emph{contextual bandits with knapsacks} [CBwK] and is the focus
of the present article.

\textbf{Literature review on CBwK.}
The first approach to CBwK, by~\citet{Badanidiyuru2014ResourcefulCB} and~\citet{Agrawal2016AnEA},
assumes a joint stochastic generation of triplets of contexts-rewards-costs,
with no specific underlying structure, and makes the problem tractable by
using as a benchmark a finite set of static policies.
As noted by~\citet{Agrawal2016LinearCB}, picking this finite set may be uneasy, which
is why they introduce instead a structural assumption of linear modeling:
the (unknown) expected rewards and cost vectors depend linearly on the contexts.

We consider a different modeling assumption, motivated
by sales with commercial discounts (see Appendix~\ref{app:industrialmotivation}):
general (known) reward and cost functions
are considered but they are coupled via a 0/1--valued factor, called a (customer) conversion, obtained as
the realization of a Bernoulli variable with parameter $P(a,\bx)$
depending on the context $\bx$ observed (customer's characteristics) and the action $a$ taken (discount level offered).
The probabilities $P(a,\bx)$ are themselves modeled by a logistic regression,
whose parameters may be learned through an adaptation of the techniques by
\citet{Filippi2010GLM} and \citet{Faury2020ImprovedOA}. We do so in the first
phase of the adaptive policy introduced in this article.
More details on the comparison of the new setting considered
to known settings of CBwK may be found in Section~\ref{sec:notcovered}.

\textbf{Primal-dual approach.}
The second phase of the adaptive policy exhibited uses the primal-dual approach
to a convex optimization problem---actually, a simple optimization problem given by a linear program.
This approach was already used in various ways for
bandits with knapsacks, including CBwK, to
define policies based on the dual problem:
this is explicit in the LagrangeBwK policy of
\citet{Nicole2019AdversarialBwK} and is implicit in the
reward-minus-weighted-cost approach of \citet{Agrawal2016LinearCB} and~\citet{Agrawal2016AnEA},
as we underline in the proof sketch of Section~\ref{sec:analysis-known}
as well as in the discussion of Section~\ref{sec:LinearCBwL}.
However, we only use the primal-dual approach in the analysis
and state our adaptive policy directly in terms of the primal problem,
where we substituted upper-confidence estimates of the probabilities~$P(a,\bx)$.
We therefore end up with a most natural adaptive policy, which mimics
the optimal static policy used as a benchmark. This direct primal statement
of the policy actually also works for the setting of
linear CBwK studied by \citet{Agrawal2016LinearCB}, as we show in Section~\ref{sec:LinearCBwL}.
Policies based on such direct primal statements were already considered
for bandits with knapsacks (see \citet{Xiaocheng2022PrimalDual} and references therein)
but do not seem easily extendable to CBwK.

\textbf{Outline and main contributions.}
The first contribution of this article is a new structured setting of CBwK,
based on a coupling between general rewards and cost vectors through conversions modeled based
on a logistic regression; we present and discuss it in Section~\ref{sec:protocol} (and explain its origins
in Appendix~\ref{app:industrialmotivation} of the supplementary material).
The adaptive policy introduced is described in Section~\ref{sec:policy}.
Its first phase consists of learning the parameter of logistic regression and is
adapted from \citet{Faury2020ImprovedOA}. Its second phase---and this is the second
contribution of this article---directly solves a primal problem with optimistic
conversion probabilities. The analysis, which we believe is concise, elegant, and natural,
is provided in Sections~\ref{sec:analysis-known} (when the context distribution $\nu$ is known)
and~\ref{sec:analysis-unknown} (when $\nu$ is unknown). As mentioned above,
Section~\ref{sec:LinearCBwL} draws the consequences of our second contribution
for linear CBwK.

\textbf{Notation.} Throughout the article, vectors are denoted with bold symbols.
In particular, $\mathbf{0}$ and $\mathbf{1}$ denote the vectors with all components equal to~$0$ and $1$,
respectively. With no additional subscript, $\Arrowvert\mathbf{v}\Arrowvert$ denotes the Euclidean
norm of a vector $\mathbf{v}$, while a subscript given by a non-negative symmetric matrix $M$ refers to
$\Arrowvert \mathbf{v} \Arrowvert_M = \sqrt{\mathbf{v}^{\transp} M \mathbf{v}}$.

\section{Learning Protocol and Motivation}

We describe the learning protocol and objectives considered (Section~\ref{sec:protocol}) and
explain why it is not covered by earlier works (Section~\ref{sec:notcovered}). We also
detail (Appendix~\ref{app:industrialmotivation} in the supplementary material) how this learning protocol was defined
based on an industrial motivation in the banking sector: market share expansion for loans by granting discounts,
under commercial budget constraints.

\subsection{Learning Protocol and Modeling Assumptions}
\label{sec:protocol}

We consider a finite action set $\cA$, including a special action $\anull$ called no-op, and a finite context set $\cX \subseteq \R^{n}$.
(We discuss and mitigate finiteness of $\cX$ in Section~\ref{sec:notcovered}.)
A scalar reward function $r : \cA \times \cX \to [0, 1]$ and a vector-valued cost function $\bc : \cA \times \cX \to [0, 1]^d$
evaluate the performance of actions given the contexts. There are several sources of costs to control: each corresponds to a component of $\bc$.
We assume that these function are known, and (with no loss of
generality) that their ranges are $[0,1]$ and $[0,1]^d$. The no-op action induces null reward and costs: $r(\anull,\bx) = 0$
and $\bc(\anull,\bx) = \mathbf{0}$ for all $\bx \in \cX$.

Contexts---which correspond, for instance, to customers' characteristics, see Appendix~\ref{app:industrialmotivation}---are
drawn sequentially according to some distribution~$\nu$, which may be known or unknown (we will deal with both cases).
At each round $t \geq 1$, upon observing the context $\bx_t \in \cX$ drawn, the learner picks an action $a_t \in \cA$,
which corresponds, for instance, to an offer made to the customer $t$. If the latter accepts the offer, an event
which we denote $y_t = 1$, then the learner obtains a reward $r(a_t,\bx_t)$ and suffers some costs $\bc(a_t,\bx_t)$.
When the customer declines the offer, we set $y_t = 0$, and null reward and costs are obtained. Thus, in both cases,
the reward and costs may be written as $r(a_t,\bx_t) \, y_t$ and $\bc(a_t,\bx_t) \, y_t$. We call $y_t$ the conversion
and now explain how it is modeled.

\textbf{Modeling conversions.} We model each conversion $y_t$ as an independent Bernoulli random drawn, with
parameter $P(a_t,\bx_t)$ depending on the context $\bx_t$ and action $a_t \ne \anull$. We further assume that these probabilities
may be written as a logistic regression model, i.e.,
there exists a known transfer function $\bphi : \cA \setminus \{ \anull \} \times \cX \to \R^m$
and some unknown parameter $\btheta_\star \in \R^m$ such that
\begin{equation}
\label{eq:P-logistic}
\forall \bx \in \cX, \ \forall a \in \cA \setminus \{ \anull \}, \qquad
P(a,\bx) = \sig \bigl( \bphi(a, \bx)^{\transp} \btheta_\star \bigr)\,,
\qquad \mbox{where} \quad \sig(x) = 1/(1+\e^{-x})\,.
\end{equation}
We assume that $\bphi$ is normalized in a way that its Euclidean norm satisfies $\Arrowvert \bphi \Arrowvert \leq 1$
and that a bounded convex set $\Theta$ containing $\btheta_\star$ is known.
Such a modeling is natural and opens the toolbox of logistic bandits; see \citet{Faury2020ImprovedOA}
and references cited therein. We however note (and discuss this fact in Appendix~\ref{app:Logistic-UCB1})
that the logistic regression model above is slightly different from the one by~\citet{Faury2020ImprovedOA}.

The concept of a conversion $y$
for a round when the no-op action $\anull$ is played is void, and thus,
we leave the probabilities $P(\anull,\bx)$ undefined,
though by an abuse of notation, these quantities might appear
but always multiplied by a $0$, given, e.g., by indicator functions like $\Inotanull{}$,
null rewards $r(\anull,\bx)$, or null costs $\bc(\anull,\bx)$.

\textbf{Policies: static vs.\ adaptive.}
The learner is given a number of rounds $T$ and a maximal budget $B$ (the same for all cost components, with no loss of generality: up
to some normalization). A static policy is a function $\pi : \cX \to \cP(\cA)$, where $\cP(\cA)$ is the set of
probability distributions over $\cA$. As is traditional in the literature of CBwK (we recall below why this is the case),
we take as benchmark the static policy $\pi^\star$ with largest expected
cumulative rewards under the condition that its cumulative costs abide by the budget constraints in expectation.
More formally,
$\pi^\star$ achieves the maximum defining
\begin{equation}
\label{eq:OPTpistar}
\begin{split}
\opt(\nu, P, B) =
&  \max_{\pi : \cX \to \cP(\cA)}
T \, \E_{\bX \sim \nu} \! \left[ \sum_{a \in \cA} r(a,\bX) \, P(a, \bX) \, \pi_a(\bX) \right] \\
&  \text{under} \qquad T \, \E_{\bX \sim \nu} \! \left[ \sum_{a \in \cA} \bc(a, \bX) \, P(a,\bX) \, \pi_a(\bX) \right]
\leq B \mathbf{1} \,,
\end{split}
\end{equation}
where $\E_{\bX \sim \nu}$ denotes an expectation solely over random contexts $\bX$ following
distribution $\nu$, where $\pi_a(\bX)$ denotes the probability mass put by $\pi(\bX)$ on $a \in \cA$,
and where $\leq$ is understood component-wise. Of course, the sums in the two expectations above are taken indifferently
over $\cA$ or $\cA \setminus \{ \anull \}$.

The learner uses an adaptive policy, i.e., a sequence of measurable
functions $\bp_t : \cH^{t-1} \times \cX \to \cP(\cA)$ indexed by $t \geq 1$, where $\cH = \cX \times \cA \times \{0,1\}$.
Indeed, the history available to the learner at the beginning of the round $t \geq 2$ is summarized by
$h_{t-1} = (\bx_s, a_s, \,y_s)_{s \leq t-1}$, and we define $h_0$ as the empty vector. Such a policy draws the action $a_t$
for round $t \geq 1$ independently at random according to $\bp_t(h_{t-1},\bx_t)$. We impose hard budget constraints on
adaptive policies: they must satisfy
\[
\sum_{t \leq T} \bc(a_t,\bx_t) y_t \leq B \mathbf{1} \quad \mbox{a.s.}
\]
Such adaptive policies are called feasible in the literature.
To abide by these hard constraints, we may restrict our attention to adaptive policies that pick Dirac masses on $\anull$
whenever one component of the cumulative costs is larger than $B-1$. At the same time, an adaptive policy should
maximize the cumulative rewards obtained or, equivalently, minimize its regret:
\[
R_T = \opt(\nu,P,B) - \sum_{t \leq T} r(a_t,\bx_t) y_t\,.
\]

It may be proved (along the same lines as \citet[Appendix~B]{Agrawal2016LinearCB} do for a different model)
that the optimal static policy $\pi^\star$ obtains, on average and in expectation,
a cumulative reward at least as good as the best feasible adaptive policy.

\textbf{Summary.} A summary of the learning protocol and of the goals is provided in Box~A.
We note here that rewards gained and vector costs suffered at round~$t$ in the case $y_t = 1$
of a conversion could be stochastic with expectations $r(a_t,\bx_t)$ and $\bc(a_t,\bx_t)$:
our analysis and the regret bounds would be unchanged, as long as the
expectation functions $r$ and $\bc$ are known.

\begin{figure}[t]
\begin{nbox}[title={Box A: Contextual bandits with knapsacks [CBwK] for a conversion model}]
\textbf{Known parameters:}
finite action set $\cA$ including a no-op action $\anull$; finite context set $\cX \subseteq \R^{n}$;
scalar reward function $r : \cA \times \cX \to [0, 1]$;
vector-valued cost function $\bc : \cA \times \cX \to [0, 1]^d$;
number $T$ of rounds; total budget constraint $B > 0$. \medskip

\textbf{Possibly unknown parameters:} context distribution $\nu$ on $\cX$;
probability of conversion given action and context $P : \cA \setminus \{ \anull \} \times \cX \to [0, 1]$,
modeled as $P(a,\bx) = \sig \bigl( \bphi(a, \bx)^{\transp} \btheta_\star \bigr)$
for some known transfer function $\bphi : \cA \setminus \{ \anull \} \times \cX \to \R^m$,
with $\Arrowvert \bphi \Arrowvert \leq 1$,
and some unknown parameter $\btheta_\star \in \R^m$, lying in a known bounded convex
set $\Theta$. \medskip

\textbf{For rounds} $t = 1, 2, 3, \ldots, T$:
\begin{enumerate}
  \item Context $\bx_t \sim \nu$ is drawn independently of the past;
  \item Learner observes $\bx_t$ and picks an action $a_t \in \cA$;
  \item Conversion $y_t \in \{0, 1\}$ is drawn according to $\Ber\bigl(P(a_t,\bx_t)\bigr)$;
  \item Learner observes $y_t$, gets reward $r(a_t,\bx_t) \,y_t$,
        and suffers costs $\bc(a_t,\bx_t) \, y_t$. \smallskip
\end{enumerate}

\textbf{Goals:} \ \ Maximize \ \ $\displaystyle{\sum_{t \leq T} r(a_t,\bx_t) \,y_t}$ \quad while controlling
\ \ $\displaystyle{\sum_{t \leq T} \bc(a_t,\bx_t) \, y_t \leq B \mathbf{1}}$ \vspace{-.125cm}
\end{nbox}
\vspace{-.5cm}
\end{figure}

\subsection{Discussion and Comparison to Existing Learning Protocols}
\label{sec:notcovered}

The setting described above may be reduced to the general setting of CBwK,
as introduced by \citet{Badanidiyuru2014ResourcefulCB} and~\citet{Agrawal2016AnEA}.
Indeed, introduce independent Bernoulli variables $y_{t,a}$ with parameters $P(a,\bx_t)$,
for all $a \in \cA \setminus \{\anull\}$, and set $y_{t,\anull} = 0$.
The vectors
\[
\smash{\Bigl( \bx_t, \, \bigl( r_t(a) \bigr)_{a \in \cA}, \, \bigl( \bc_t(a) \bigr)_{a \in \cA} \Bigr)}\,,
\qquad \mbox{where} \quad r_t(a) = r(a,\bx_t)\,y_{t,a}
\quad \mbox{and} \quad
\bc_t(a) = \bc(a,\bx_t)\,y_{t,a}
\]
are i.i.d., and upon picking action $a_t \in \cA$, the obtained and observed
rewards and cost vectors equal $r_t(a_t)$ and $\bc_t(a_t)$.
When $\cX$ is discrete, we may consider the set $\Pi$ of base policies
that map $\cX$ to $\{ \delta_a : a \in \cA \}$, the set of Dirac masses at some $a \in \cA$.
The convex hull of $\Pi$ is the set of all static policies $\cX \to \cP(\cA)$, against which
we would like our policy to compete; but the adaptive
policies by \citet{Badanidiyuru2014ResourcefulCB} and~\citet{Agrawal2016AnEA}
only compete with respect to the best single element in $\Pi$, not the
best convex combination of elements of $\Pi$.

The setting of linear CBwK (\citet{Agrawal2016LinearCB}) provides a structural link
between contexts and expected rewards and cost vectors, but in a linear way
that is incomparable to the setting of CBwK for a conversion model introduced above.
More details are given in Section~\ref{sec:LinearCBwL}.
We also mention that linear and logistic structural links between
contexts (prices) and rewards or costs were studied in a non-contextual setting (i.e., not in CBwK)
by~\citet{Maio2021}. Their strategy bears some resemblance to the one by~\citet{Agrawal2016LinearCB},
in particular, both consider an online convex optimization strategy as a subroutine.

All mentioned references consider a no-op action $\anull$.
(It could be replaced by the existence of a standard action $a_{\nocost}$
always achieving null costs and possibly some positive rewards.)

On the contrary, none of the mentioned references assumes that the context $\cX$ set is finite.
This is a technical necessity for a part of the adaptive policy introduced; see the
discussion of computational complexity at the end of Section~\ref{sec:policy}.
But somehow, considering a finite set $\Pi$ of policies,
as in \citet{Badanidiyuru2014ResourcefulCB} and~\citet{Agrawal2016AnEA}, is a counterpart
to assuming finiteness of~$\cX$.
Also, Appendix~\ref{app:simu} actually mitigates this restriction that $\cX$ is finite:
learning the logistic parameter $\btheta_\star$ may be achieved with continuous contexts
(see Phase~1 in Section~\ref{sec:policy});
only the subsequent optimization part (Phase~2 in Section~\ref{sec:policy})
requires finiteness of $\cX$. We may thus well discretize only $\cX$
for this Phase~2, which is exactly what Appendix~\ref{app:simu} performs.
This mitigation comes with possible theoretical guarantees as
Sections~\ref{sec:analysis-known} and~\ref{sec:analysis-unknown}
reveal that the errors $\varepsilon_t(a,\bx)$ for learning $\theta_\star$ and $P$,
obtained as outcomes of the first step of the analyses, are carried over in the subsequent steps,
where the optimization part is evaluated.

\section{Description of the Adaptive Policy Considered}
\label{sec:policy}

At each stage $t \geq 1$, the policy first updates an estimator $\hat{\btheta}_{t-1}$ of $\btheta_\star$
based on the history $h_{t-1}$ available so far, based on an adaptation of the Logistic-UCB1 algorithm by~\citet{Faury2020ImprovedOA},
and deduces estimators $\hat{P}_{t-1}(a,\bx)$ and upper confidence bounds $U_{t-1}(a,\bx)$ of the probabilities $P(a,\bx)$. The policy
then solves the corresponding estimated version of the optimization problem~\eqref{eq:OPTpistar}. We now
describe the corresponding two steps.
In the description below, quantities that depend on information available at round $t-1$ (respectively, $t$) are
indexed by $t-1$ (respectively, $t$).

\textbf{Phase 0: In case the cost constraints are about to be violated.}
To make sure cost constraints are never violated, whenever at least one of the components of the current
cumulative costs
is larger than $B-1$ and could possibly be larger than $B$ at the end of round $t$,
we play $\anull$ (and we actually do so for the rest of the rounds). This corresponds to defining
$\bp_t(h_{t-1},\bx) = \delta_{\anull}$ for all $\bx \in \cX$, where $\delta_{\anull}$
denotes the Dirac mass on $\anull$. Otherwise,
we proceed as described below in Phase~1 and Phase~2.

\textbf{Phase 1: Learning $\btheta_\star$ via an adapted Logistic-UCB1.} This first phase depends on a regularization parameter $\lambda > 0$
and on upper-confidence bonuses $\varepsilon_{t}(a,\bx) > 0$, both to be specified by the analysis.

At rounds $t \geq 2$, we first maximize a regularized log-likelihood of the history $h_{t-1}$:
\begin{equation}
\label{eq:besttheta}
\tilde{\btheta}_{t-1} \in \argmax_{\btheta \in \R^{m}}
\sum_{s=1}^{t-1} \Inotanull{s} \biggl( y_{s} \ln \sig \bigl( \bphi(a_s, \bx_s)^{\transp} \btheta \bigr)
+ (1- y_{s}) \ln \Bigl( 1 - \sig \bigl( \bphi(a_s, \bx_s)^{\transp} \btheta \bigr) \Bigr) \biggr)
- \frac{\lambda}{2} \Arrowvert \btheta \Arrowvert^2\,.
\end{equation}
In the expression above, we read that we only gather information about $\btheta_\star$
at those rounds $s$ when $a_s \ne \anull$.
When $\tilde{\btheta}_{t-1}$ does not belong to~$\Theta$, an ad hoc projection step
corrects for this, if needed:
\begin{align}
\label{eq:projtheta}
\hat{\btheta}_{t-1} & \in \argmin_{\btheta \in \Theta}
\Bigl\Arrowvert \Psi_{t-1}(\btheta) - \Psi_{t-1}\bigl(\tilde{\btheta}_{t-1}\bigr) \Bigr\Arrowvert_{W_{t-1}(\btheta)^{-1}}\,, \\
\nonumber
\mbox{where} \qquad
\Psi_{t-1}(\btheta) & =
\smash{\sum_{s=1}^{t-1} \Inotanull{s} \, \sig \bigl(\bphi(a_s, \bx_s)^{\transp} \btheta \bigr) \, \bphi(a_s, \bx_s) + \lambda \btheta} \phantom{\sum^t} \\
\label{eq:Wttheta}
\mbox{and} \qquad W_{t-1}(\btheta) & = \lambda \id_m +
\sum_{s=1}^{t-1} \Inotanull{s} \, \dot{\sig}\bigl( \bphi(a_s, \bx_s)^{\transp} \btheta \bigr) \,
\bphi(a_s, \bx_s) \bphi(a_s, \bx_s)^{\transp} \,.
\end{align}
We recall that the function $\dot{\sig}$ denotes the derivative of $\sig$, i.e., $\dot{\sig}(x) = \e^{-x}/\bigl( 1 + \e^{-x} \bigr)^2$.
We have $\dot{\sig} = \eta(1-\eta)$.

By plug-in, we finally define estimators and upper-confidence bounds of the probabilities $P(a,\bx)$
for $a \ne \anull$ and all $\bx \in \cX$:
\[
\hat{P}_{t-1}(a,\bx) = \sig \bigl( \bphi(a, \bx)^{\transp} \hat{\btheta}_{t-1} \bigr)
\qquad \mbox{and} \qquad
U_{t-1}(a,\bx) = \min \bigl\{ \hat{P}_{t-1}(a,\bx) + \varepsilon_{t-1}(a,\bx), \, 1 \bigr\}\,.
\]
For $\anull$, no estimators or upper-confidence bounds
need to be defined, as the quantities $P(\anull,\bx)$ are actually undefined.

\textbf{Phase~2: Sampling, via solving an optimization problem with expected constraints.}
This phase relies on a conservative-budget parameter denoted by $B_T$,
which is only slightly smaller than $B$ and whose exact value is to be specified by the analysis.

We start with the case of a known context distribution $\nu$.
At round $t = 1$, we play an arbitrary action in $\cA \setminus \{ \anull \}$.
At rounds $t \geq 2$, if at least one component of the cumulative vector costs suffered
so far is larger than $B - 1$, we pick $a_t = \anull$. Otherwise,
we pick for $\bp_t(h_{t-1},\,\cdot\,)$ the solution of the
optimization problem $\opt(\nu, U_{t-1}, B_T)$
defined\footnote{In the definition~\eqref{eq:OPTpistar} of $\opt(\nu, U_{t-1}, B_T)$,
expectations are only over $\bX \sim \nu$ and not over the random variable $U_{t-1}$;
more comments and explanations on this fact may be found in Appendix~\ref{app:vhpc}.}
in~\eqref{eq:OPTpistar}, and draw $a_t$ according to $\bp_t(h_{t-1},\bx_t)$.

When the context distribution is unknown, we rather pick for $\bp_t(h_{t-1},\,\cdot\,)$ the solution of the
optimization problem $\opt\bigl(\hat{\nu}_{t}, U_{t-1}, B_T \bigr)$,
where
\begin{equation}
\label{eq:empestnu}
\hat{\nu}_{t} = \frac{1}{t} \sum_{s=1}^{t} \delta_{\bx_s}\,,
\end{equation}
with $\delta_{\bx}$ denoting the Dirac mass at $\bx \in \cX$. Since $\bx_t$ is revealed at the beginning of
round $t$, before we pick an action, we may indeed use $\hat{\nu}_{t}$ at round $t$.

\textbf{Summary and discussion of the computational complexity.}
We summarize the considered adaptive policy in Box~B and now discuss its computational complexity.

As $\ln \varphi$ and $\ln (1-\varphi)$ are strictly concave and smooth,
the maximum-likelihood step \eqref{eq:besttheta} of Phase~1 consists
of maximizing a strictly concave and smooth function over $\R^m$,
which may be performed efficiently. The projection step~\eqref{eq:projtheta}
of Phase~1 is however an issue, both with the version of Logistic-UCB1
discussed here and with the earlier approach by \citet[Section~3]{Filippi2010GLM}.
The latter and \citet[Section~4.1]{Faury2020ImprovedOA} both underline that
the projection step~\eqref{eq:projtheta} is a complex optimization problem
that however does not often need to be solved in practice, as
they usually observe $\tilde{\btheta}_{t-1} \in \Theta$. Our numerical experiments
concur with this statement (but admittedly, they rely on choosing a rather large value of $\Theta$).

On the contrary, Phase 2 of the adaptive policy consists of solving a linear
program with $|\cX| \times |\cA|$ constraints, where
where $|\cX|$ and $\cA$ denote the cardinality of $\cX$ and $\cA$, respectively---see
the detailed rewriting~\eqref{eq:primal-problem} in the supplementary material.
Therefore, the computational complexity of Phase~2 is polynomial
(of weak order) in $|\cX| \times |\cA|$. To achieve this acceptable
complexity we had however to restrict our attention to finite
sets of contexts $\cX$, which requires in practice
segmenting countable or continuous
context sets into finitely many clusters, for instance. We
do so in our experiments.

\begin{figure}[t]
\begin{nbox}[title={Box B: Logistic-UCB1 for direct solutions to OPT problems}]
\textbf{Parameters:} regularization parameter $\lambda > 0$;
conservative-budget parameter $B_T$;
upper-confidence bonuses $\varepsilon_{s}(a,\bx) > 0$, for $s \geq 1$
and $(a,\bx) \in \bigl( \cA \setminus \{\anull\} \bigr) \times \cX$. \medskip

\textbf{Round} $t=1$: play an arbitrary action $a_1 \in \cA \setminus \{\anull\}$  \medskip

\textbf{At rounds} $t \geq 2$: \smallskip
\begin{enumerate}
\item[\underline{Phase 0}] If $\displaystyle{\sum_{s \leq t-1} \bc(a_s,\bx_s) \, y_s \leq (B-1) \mathbf{1}}$ is violated, then
$\bp_t(h_{t-1},\bx) = \delta_{\anull}$ for all $\bx$ \smallskip
\item[\underline{Phase 1}] Otherwise, compute a maximum-likelihood estimator $\tilde{\btheta}_{t-1}$ of $\btheta_\star$ according to~\eqref{eq:besttheta},
compute its projection $\hat{\btheta}_{t-1}$ onto~$\Theta$ according to~\eqref{eq:projtheta}, and define, for $a \ne \anull$:
\[
\hspace{-.75cm}
\hat{P}_{t-1}(a,\bx) = \sig \bigl( \bphi(a, \bx)^{\transp} \hat{\btheta}_{t-1} \bigr) \quad \mbox{and} \quad
U_{t-1}(a,\bx) = \min \Bigl\{ \hat{P}_{t-1}(a,\bx) + \varepsilon_{t-1}(a,\bx), \,\, 1 \Bigr\} \vspace{-.1cm}
\]
\item[\underline{Phase 2}] Compute the solution $\bp_t(h_{t-1},\,\cdot\,)$ of
\begin{align*}
\hspace{-.75cm}
\opt\bigl(\tilde{\nu}, U_{t-1}, B_T\bigr) =
&  \max_{\pi : \cX \to \cP(\cA)}
T \, \E_{\bX \sim \tilde{\nu}} \! \left[ \sum_{a \in \cA} r(a,\bX) \, U_{t-1}(a, \bX) \, \pi_a(\bX) \right] \\
&  \text{under} \qquad T \, \E_{\bX \sim \tilde{\nu}} \! \left[ \sum_{a \in \cA} \bc(a, \bX) \, U_{t-1}(a,\bX) \, \pi_a(\bX) \right]
\leq B_T \mathbf{1} \,,
\end{align*}
where $\tilde{\nu}$ denotes either $\nu$ (when it is known) or its empirical estimate~$\hat{\nu}_{t}$ in~\eqref{eq:empestnu} \smallskip \\
Draw an arm $a_t \sim \bp_t(h_{t-1},\bx_t)$.
\end{enumerate}
\end{nbox}
\vspace{-.5cm}
\end{figure}

\paragraph{Simulation study.}
A simulation study on partially simulated but realistic data may be found in Appendix~\ref{app:simu}.
The underlying dataset is the standard ``default of credit card clients''
dataset of~\citet{UCI2016DefaultCR}, initially provided by~\citet{UCIarticle}.
(It may be used under a Creative Commons Attribution 4.0 International [CC BY 4.0] license.)

\section{Analysis for a Known Context Distribution $\nu$}
\label{sec:analysis-known}

Since $\Theta$ is bounded, the following quantity, standardly introduced in the context of
logistic bandits (see \citet{Faury2020ImprovedOA} and references therein), is finite, though possibly large:
\[
\kappa = \sup \! \left\{ \frac{1}{\dot{\sig} \bigl(\bphi(a, \bx)^{\transp} \btheta \bigr)} :
\bx \in \cX, \ a \in \cA \setminus \{ \anull \}, \ \btheta \in \Theta \right\} < +\infty\,.
\]
We denote by $\Arrowvert \Theta \Arrowvert_{\maxn} = \max \bigl\{ \Arrowvert \btheta \Arrowvert : \ \btheta \in \Theta \bigr\}$
the maximal Euclidean norm of an element in $\Theta$.

By construction, given that individual cost vectors lie in $[0,1]^d$ and due to its ``Phase~0'', the adaptive policy considered
always satisfies the budget constraints. The bound on rewards reads as follows.

\begin{theorem}
\label{th:main}
In the setting of Box~A of Section~\ref{sec:protocol},
we consider the adaptive policy of Box~B of Section~\ref{sec:policy}
assuming that the distribution of the contexts is known, i.e., with $\tilde{\nu} = \nu$.
We set a confidence level $1-\delta \in (0,1)$
and use parameters $\lambda = m \ln(1+T/m)$,
\[
B_T = B - 2 - \sqrt{2T \ln(4d/\delta)}\,,
\]
and $\epsilon_t(a,\bx)$ stated in~\eqref{eq:def:vareps} of the supplementary material.
Then, provided that $T \geq 2m$ and $B > 4 + 2\sqrt{2T \ln (4d/\delta)}$,
we have, with probability at least $1 - 2 \delta$,
\[
\opt(\nu,P,B) -
\sum_{t \leq T} r(a_t,\bx_t) \, y_t \leq
\left( 4 + 2\sqrt{2T \ln \frac{4d}{\delta}} \right) \frac{\opt(\nu,P,B)}{B}
+ E_T + \sqrt{2T \ln \frac{4}{\delta}} + 1\,,
\]
where the closed-form expression of $E_T = \mathcal{O} \bigl( m \sqrt{T} \ln T \bigr)$
is in~\eqref{eq:bdETpart3} of the supplementary material.
\end{theorem}

We will rather discuss the bound of the more general Theorem~\ref{th:main-unknown}
(to be stated and proved in Section~\ref{sec:analysis-unknown}) than the one
of Theorem~\ref{th:main}.
We provide a proof sketch in Section~\ref{sec:proofsketch} and
discuss the main technical novelty in Section~\ref{sec:maintechnicalnovelty}.

\subsection{Proof Sketch for Theorem~\ref{th:main}}
\label{sec:proofsketch}

The detailed proof of Theorem~\ref{th:main} may be found in Appendix~\ref{app:proof-nu-known}. We provide here an overview thereof,
highlighting the four main ingredients. The third and fourth steps benefited from some inspiration
drawn from the proof techniques of~\citet{Agrawal2016LinearCB}. The first step
is an adaptation of Lemmas~1 and~2 by \citet{Faury2020ImprovedOA}.

\textbf{First,} the mentioned adaptation
provides values of the parameters $\epsilon_t(a,\bx)$ such that, with
probability at least $1-\delta$,
\begin{align*}
\nonumber
& & \forall t \geq 1, \ \forall a \in \cA \setminus \{ \anull \}, \ \forall \bx \in \cX, \qquad &
\bigl| \hat{P}_{t}(a,\bx) - P(a,\bx) \bigr| \leq \epsilon_t(a,\bx)\,, \\
\mbox{hence} & \ & & U_t(a,\bx) - 2 \epsilon_t(a,\bx) \leq P(a,\bx) \leq U_t(a,\bx)\,,
\end{align*}
while $\displaystyle{\sum_{t \leq T} \epsilon_{t-1}(a_t,\bx_t) \Inotanull{t}}$ is of order $\sqrt{T}$ up to poly-logarithmic terms.

\textbf{Second,} the Phase~2 formulation of the strategy, in a primal form, is equivalently restated in a dual form.
For each round $t \geq 2$,
strong duality holds and entails the existence of a vector $\bbeta^{\budg,\star}_t \in \R^d$ such that
$\bp_t(h_{t-1},\,\cdot\,)$ may be identified as the argmax over $\pi: \cX \to \cP(\cA)$ of
\[
\E_{\bX \sim \nu} \! \left[ T \sum_{a \in \cA} \Bigl( r(a,\bX)  - \bigl(\bbeta^{\budg,\star}_t\bigr)^{\, \transp} \bc(a, \bX) \Bigr) \, U_{t-1}(a,\bX) \, \pi_{a}(\bX)  +  \sum_{\bx \in \cX} \sum_{a \in \cA} {\beta_{\bx, a}^{\ppos,\star}} \, \pi_{a}(\bx) \right].
\]
By exploiting the KKT conditions, we are able to get rid of the double sum above and finally
get a $\cX$--pointwise characterization of $\bp_t(h_{t-1},\,\cdot\,)$:
for all $\bx \in \cX$,
\begin{align*}
\bp_t(h_{t-1},\bx) \in
& \argmax_{\bq \in \cP(\cA)} \sum_{a \in \cA} \bigl( r(a,\bx) - \bigl(\bbeta^{\budg,\star}_t\bigr)^{\, \transp} \bc(a, \bx) \bigr)
\, U_{t-1}(a,\bx) \, q_{a} \\
& = \argmax_{\bq \in \cP(\cA)} \sum_{a \in \cA} \bigl( r(a,\bx) - \bigl(\bbeta^{\budg,\star}_t\bigr)^{\, \transp} \bc(a, \bx) \bigr)_+
\, U_{t-1}(a,\bx) \, q_{a}\,.
\end{align*}
Non-negative parts $(\,\cdot\,)_+$ may be introduced thanks to the existence of the no-op action $\anull$.
The distributions $\bp_t(h_{t-1},\bx)$ may therefore be interpreted as maximizing some upper-confidence bound
on penalized gains (rewards minus some scalarized costs);
the dual variables $\bbeta^{\budg,\star}_t$ play a role similar to the $Z$ parameter of \citet[Section~3.3]{Agrawal2016LinearCB} in
terms of weighing gains versus costs.
In passing, we also prove
\[
\smash{\opt(\nu,U_{t-1},B_T) \geq B_T (\bbeta^{\budg,\star}_t )^{\, \transp} \mathbf{1}}
\]
based on the KKT conditions.
The latter inequality is comparable in spirit to the bound of \citet[Corollary 3]{Agrawal2016LinearCB}, relating $Z$ to $\opt(\nu,P,B)/B$.

\textbf{Third,} for $t \geq 2$,
whenever the policy $\bp_t(h_{t-1},\,\cdot\,)$ is obtained
by solving the optimization problem $\opt(\nu,U_{t-1},B_T)$ of Phase~2
and by independence of $\bx_t$ and $h_{t-1}$, we have
\begin{align*}
\frac{\opt(\nu,U_{t-1},B_T)}{T}
& = \E_{\bX \sim \nu} \!\left[ \sum_{a \in \cA} r(a, \bX) \, U_{t-1}(a, \bX) \, p_{t,a}(h_{t-1},\bX) \right] \\
& = \E\bigl[ r(a_t, \bx_t) \, U_{t-1}(a_t, \bx_t) \,\big|\, h_{t-1} \bigr]\,.
\end{align*}
Therefore, repeated applications of the Hoeffding-Azuma inequality and the inequalities of the first step
entail that, up to quantities of the order of $\sqrt{T}$,
\begin{align*}
\sum_{t=2}^T \frac{\opt(\nu,U_{t-1},B_T)}{T} & \approx
\sum_{t=2}^T r(a_t, \bx_t) U_{t-1}(a_t, \bx_t) \\
& \lesssim \sum_{t=2}^T \epsilon_{t-1}(a_t,\bx_t) \Inotanull{t} + \sum_{t=2}^T r(a_t, \bx_t) P(a_t, \bx_t)
\lesssim \sum_{t=2}^T r(a_t, \bx_t) \, y_t\,.
\end{align*}
We thus only need to control $\displaystyle{\opt(\nu,P,B) - \sum_{t=2}^T \frac{\opt(\nu,U_{t-1},B_T)}{T}}$,
which may be assumed $\geq 0$.

The value $B_T = B - 2 - \sqrt{2 T \ln (4d/\delta)}$ and
similar Hoeffding-Azuma-based arguments show that with high probability, the budget limit $B-1$ is indeed never reached and that
we always compute $\bp_t(h_{t-1},\,\cdot\,)$ in the way indicated by Phase~2.

\textbf{Fourth,} we collect all bounds together. We start with
\[
\sum_{t=2}^T \frac{B_T}{T} (\bbeta^{\budg,\star}_t )^{\, \transp} \mathbf{1}
\leq \sum_{t=2}^T \frac{\opt(\nu,U_{t-1},B_T)}{T} \leq \opt(\nu,P,B)\,.
\]
We the exploit the dual characterization of $\bp_t(h_{t-1},\,\cdot\,)$ and the control
$P \leq U_{t-1}$ to get that with high probability, for all $\bx \in \cX$,
\begin{multline*}
\sum_{a \in \cA} \bigl( r(a,\bx) - \bigl(\bbeta^{\budg,\star}_t\bigr)^{\, \transp} \bc(a, \bx) \bigr) \, U_{t-1}(a,\bx) \, p_{t, a}(h_{t-1},\bx) \\
\geq \sum_{a \in \cA} \bigl( r(a,\bx) - \bigl(\bbeta^{\budg,\star}_t\bigr)^{\, \transp} \bc(a, \bx) \bigr) \, P(a,\bx) \, \pi_a^{\star}(\bx)\,.
\end{multline*}
After integration over $\bX \sim \nu$ and substituting of the definitions of $\pi^\star$ and $\bp_{t, a}(h_{t-1},\,\cdot\,)$,
as well as the equality stemming from the KKT conditions, we have
\begin{align*}
& \overbrace{\E_{\bX \sim \nu} \left[ \sum_{a \in \cA} r(a,\bX)  \, U_{t-1}(a,\bX) \, \bp_{t, a}(h_{t-1},\bX) \right]}^{= \opt(\nu,U_{t-1},B_T)/T} \\
& \ \ \ \ - \underbrace{\E_{\bX \sim \nu} \left[ \sum_{a \in \cA}  \bigl(\bbeta^{\budg,\star}_t\bigr)^{\, \transp} \bc(a, \bX) \, U_{t-1}(a,\bX) \, \bp_{t, a}(h_{t-1},\bX)\right]}_{(B_T/T) (\bbeta^{\budg,\star}_t)^{\, \transp} \mathbf{1}} \\
\geq \ \ & \underbrace{\E_{\bX \sim \nu} \left[  \sum_{a \in \cA} r(a,\bX) \, P(a,\bX) \, \pi^{\star}_a(\bX) \right]}_{=\opt(\nu,P,B)/T} -(\bbeta^{\budg,\star}_t\bigr)^{\, \transp} \underbrace{\E_{\bX \sim \nu} \left[ \sum_{a \in \cA}  \bc(a, \bX) \, P(a,\bX) \, \pi^{\star}_a(\bX)\right]}_{\leq (B/T) \mathbf{1}}.
\end{align*}
Rearranging and summing over $2 \leq t \leq T$, we obtain
\[
\sum_{t=2}^T \frac{\opt(\nu,P,B) - \opt(\nu,U_{t-1},B_T)}{T}
\leq \sum_{t=2}^T \frac{B - B_T}{T} (\bbeta^{\budg,\star}_t )^{\, \transp} \mathbf{1}
\leq \left( \frac{B}{B_T} - 1 \right) \opt(\nu,P,B)\,,
\]
where we substituted the first inequality stated in this fourth step. This concludes the proof.

\subsection{Discussion on the Main Technical Novelties}
\label{sec:maintechnicalnovelty}

As should be clear from the comments at the beginning of Section~\ref{sec:proofsketch},
the technical novelties lies in the second step of the proof of Theorem~\ref{th:main}.
On the one hand, we are able to directly analyze a strategy stated in a
primal form, which is a more natural formulation. On the other hand, doing so,
we are also able to avoid the issues that come with dual formulations, relying, e.g., on some critical parameter $Z$, as in
\citet[Theorem~3]{Agrawal2016LinearCB}, to trade off rewards and costs.
This parameter $Z$ should be of order $\opt/B$ and has to be learned, e.g., through
$\sqrt{T}$ initial exploration rounds.
(More details are to be found in Section~\ref{sec:th-lin}.)
In our analysis, this parameter $Z$ is superseded
by dual optimal variables $\bbeta^{\budg,\star}_t \geq \mathbf{0}$, that are only used in the analysis
and not to state the policy, unlike in \citet{Agrawal2016LinearCB}.
Put differently, the clever use in this context of KKT conditions is the main technical novelty.
On a side note,
we are also able to take care in an explicit and detailed fashion of the no-op action $\anull$,
whose specific treatment is often unaddressed in the literature.

\section{Analysis for an Unknown Context Distribution~$\nu$}
\label{sec:analysis-unknown}

When the context distribution $\nu$ is unknown, we simply estimate it through its empirical frequencies~\eqref{eq:empestnu}.
The regret bound is almost unchanged:
an additional mild factor of, e.g., $2 |\cX| \sqrt{2T \ln (2 T |\cX|/\delta)}$ appears
in the $\sqrt{T}$ term multiplying $\opt(\nu,P,B)/B$.
This term comes from some uniform deviation argument
stated in~\eqref{eq:unifctrl} and~\ref{eq:tobeimproved}.

\begin{theorem}
\label{th:main-unknown}
In the setting of Box~A of Section~\ref{sec:protocol},
we consider the adaptive policy of Box~B of Section~\ref{sec:policy}
with $\tilde{\nu} = \hat{\nu}_t$ at rounds $t \geq 2$.
We set a confidence level $1-\delta \in (0,1)$
and use parameters $\lambda = m \ln(1+T/m)$, a working budget of
\[
\smash{B - b_T\,, \qquad \mbox{where} \qquad b_T = 2 + \sqrt{2T \ln(4d/\delta)}
+ |\cX| \sqrt{2T \ln\bigl(2 T |\cX|/\delta\bigr)}\,,}
\]
and $\epsilon_t(a,\bx)$ stated in~\eqref{eq:def:vareps} of the supplementary material.
Then, provided that $T \geq 2m$ and $B > 2b_T$,
we have, with probability at least $1 - 3\delta$,
\[
\opt(\nu,P,B) -
\sum_{t \leq T} r(a_t,\bx_t) \, y_t \leq
2 b_T \left( 1 + \frac{\opt(\nu,P,B)}{B} \right) + E_T\,, \vspace{-.15cm}
\]
where the expression of $E_T = \mathcal{O} \bigl( m \sqrt{T} \ln T \bigr)$
may be found in~\eqref{eq:bdETpart3} of the supplementary material.
\end{theorem}

The order of magnitude of the regret bound is
$\bigl( m + |\cX| \opt(\nu,P,B)/B \big) \sqrt{T} \ln T$,
which is reminiscent of all known regret upper bounds for CBwK (e.g., the ones by
\citet{Badanidiyuru2014ResourcefulCB} and~\citet{Agrawal2016AnEA}, for general CBwK,
and \citet{Agrawal2016LinearCB} for linear CBwK, see Section~\ref{sec:LinearCBwL}).
The factor $|\cX|$ may be improved, see below, but this is a detail.
A discussion on exhibiting corresponding lower bounds is to be found at the end of
Section~\ref{sec:LinearCBwL}.

A detailed proof of Theorem~\ref{th:main-unknown} is provided in Appendix~\ref{app:proof-nu-unknown} of the supplementary material.
It follows closely the proof of Theorem~\ref{th:main}, with modifications mostly consisting
of relating quantities of the form
\[
\smash{\E_{\bX \sim \hat{\nu}_t} \bigl[ f(\bX) \bigr]
\ \  \mbox{vs.} \ \
\E_{\bX \sim \nu} \bigl[ f(\bX) \bigr]\,,
\ \  \mbox{where, e.g.,} \ \
f(\bX) = \sum_{a \in \cA} r(a,\bX)  \, U_{t-1}(a,\bX) \, \bp_{t, a}(h_{t-1},\bX)\,.} \vspace{.15cm}
\]
To do so, we use that
for all functions $f : \cX \to [0,1]$,
\begin{equation}
\label{eq:unifctrl}
\forall \, t \leq T, \qquad
\Bigl| \E_{\bX \sim \hat{\nu}_t} \bigl[ f(\bX) \bigr]
- \E_{\bX \sim \nu} \bigl[ f(\bX) \bigr] \Bigr|
\leq \sum_{x \in \cX} \bigl| \hat{\nu}_t(\bx) - \nu(\bx) \bigr|
\defeq \bigl\Arrowvert \hat{\nu}_t - \nu \bigr\Arrowvert_1
\,,
\end{equation}
where $\bigl\Arrowvert \hat{\nu}_t - \nu \bigr\Arrowvert_1$ is the total variation distance
between $\hat{\nu}_t$ and $\nu$.
In Appendix~\ref{app:proof-nu-unknown},
we upper bound the latter, for the sake of simplicity,
in a crude way by applying $T |\cX|$ times the Hoeffding-Azuma inequality
(once for each $1 \leq t \leq T$ and $\bx \in \cX$) and obtain that
with probability at least $1-\delta$,
\begin{equation}
\label{eq:tobeimproved}
\forall \, t \leq T, \qquad
\bigl\Arrowvert \hat{\nu}_t - \nu \bigr\Arrowvert_1
\leq |\cX| \sqrt{\frac{1}{2t} \ln \frac{2 T |\cX|}{\delta}}\,.
\end{equation}
The $|\cX| \sqrt{2T \ln (2 T |\cX|/\delta)}$ term in the regret bound of Theorem~\ref{th:main-unknown}
appears as the sum over $t \leq T$ of the deviation bounds~\eqref{eq:tobeimproved}.
The bounds~\eqref{eq:tobeimproved} may actually be improved
into bounds of the order of $\sqrt{|\cX|/t}$,
via some Cauchy-Schwarz bound and a deviation argument in Banach spaces by \cite{Pinelis},
or by more direct techniques described by \citet[Lemma~3]{Dev} and \citet{BK12}.
In any case, the regret bound of Theorem~\ref{th:main-unknown} automatically benefits from such
improvements, by replacing the $2 |\cX| \sqrt{2T \ln (2 T |\cX|/\delta)}$ term in the bound therein by the (sum over $t \leq T$ of the) better
uniform deviation bounds.

\section{Extension to Linear Contextual Bandits with Knapsacks}
\label{sec:LinearCBwL}

This section is a brief summary of Appendix~\ref{app:linCBwK}.
We explain therein how the adaptive policy of Box~B may be adapted to the
setting of linear CBwK, introduced by \citet{Agrawal2016LinearCB},
where the bounded rewards $r_t$ and vector costs $\bc_t$ are independently generated at each round
according to bounded distributions with respective
expectations $\olr(a_t,\bx_t)$ and $\olbc(a_t,\bx_t)$, depending linearly on (a transfer function $\varphi$ of) the contexts:
for all $a \ne \anull$ and $\bx \in \cX$, for all components $i$ of $\olbc$,
\[
\olr(a,\bx) = \bphi(a,\bx)^{\transp} \bmu_{\star}
\qquad \mbox{and} \qquad
\olc_i(a,\bx) = \bphi(a,\bx)^{\transp} \btheta_{\star,i}\,.
\]
We consider the same benchmark $\opt(\nu,\olr,\olbc,B)$ as \citet{Agrawal2016LinearCB} and are able to
exhibit a similar $\bigl( \opt(\nu,\olr,\olbc,B) / B \bigr) m \sqrt{T} \ln T$
regret bound, with however a slight relaxation on the order of magnitude required for $B$.
We do so with a strategy that we deem more direct and natural, inspired from the one of Box~B,
where in Phase~1 a LinUCB-type (\citet{AbbasiYadkori2011ImprovedAF}) estimation of the parameters
is performed, and where in Phase~2, a direct solution to an $\opt$ problem with estimated
parameters is performed. The parameters are upper-confidence functions $U_{t-1}$ on $\olr$
and lower-confidence vector functions $\bL_{t-1}$ on $\olbc$.

The main advantage of our approach in the case of linear contextual bandits
is exactly as described in Section~\ref{sec:maintechnicalnovelty}:
avoiding the critical parameter $Z$ of
\citet[Theorem~3]{Agrawal2016LinearCB}, which is used to trade off rewards and costs.
The main limitation of our approach is the assumption of a finite context set $\cX$,
which is required to make the Phase-2 linear program tractable.

\section{Future Work}

We conclude this article with a list of issues to be further investigated.

\emph{First}, as discussed in Section~\ref{sec:notcovered}, the restrictions of finiteness
should be alleviated: finiteness of the context set in the setting of this
article, or finiteness of the set of benchmark policies in other settings
(see \citealp{Badanidiyuru2014ResourcefulCB} and~\citealp{Agrawal2016AnEA}).

\emph{Second}, we only dealt with $\leq$ budget constraints (and do does the literature so far).
Direct approaches to constraints of the form $\geq$ remain to be further investigated.

\emph{A third} series of questions to be clarifies concerns regret lower bounds, and more generally,
the tightness of the results---in particular, the required conditions on budget sizes.
Earlier references for contextual bandits with knapsacks
did also not provide lower bounds statements that were simultaneously
optimal (i.e., matching the obtained upper bounds) and general (i.e., valid for all problems
with a given number of rounds~$T$,
a given budget $B$, and a given value $\opt$ for the optimal expected reward achievable by a static policy).
\citet[comments after Theorem~1]{Badanidiyuru2014ResourcefulCB} merely indicates
that the obtained regret upper bound is optimal in some regimes, e.g., when the budget $B$
grows linearly with the number of rounds $T$.
\citet[comments after Theorem~1]{Agrawal2016LinearCB} only compares the obtained regret upper bound
to the case of no budget constraints.
In particular, as far as the orders of magnitude in~$T$ are concerned,
it is unclear whether the $(\opt/B) \sqrt{T}$ rates achieved (up to poly-logarithmic factors) in
the present article and in the two mentioned references are optimal.
These rates do not match the optimal rates in the case of no contexts, which were stated
and proved by~\citet{Badanidiyuru2013BanditsWK}.

\begin{ack}
Zhen Li and Gilles Stoltz have no direct funding to acknowledge
other than the salaries paid by their employers, BNP Paribas,
CNRS, and HEC Paris. They have no competing interests to declare.
\end{ack}

\bibliographystyle{plainnat}
\bibliography{CBwK-conversion--bib}

\clearpage
\appendix

\begin{center}
  {\Large\bf Supplementary Material for \medskip \\
  	``{\titre}''}
\end{center}

\ \\

The appendices of this article contain the following elements.
\begin{itemize}
\item Appendix~\ref{app:industrialmotivation} provides the industrial motivation
behind the setting described in Section~\ref{sec:protocol}, namely, within the banking industry,
a market share expansion for loans.
\item Appendix~\ref{app:proof-nu-known} contains the proof of Theorem~\ref{th:main},
i.e., the regret bound in case of a known distribution~$\nu$, except
for a lemma on learning the parameter of logistic bandits, provided
in Appendix~\ref{app:Logistic-UCB1}.
\item Appendix~\ref{app:Logistic-UCB1} states and proves
the indicated lemma; both the statement and the proof are mere adaptations
of \citet[Lemmas~1 and~2]{Faury2020ImprovedOA}.
\item Appendix~\ref{app:proof-nu-unknown} contains the proof of
Theorem~\ref{th:main-unknown}, i.e., explains how to adapt the proof of
Appendix~\ref{app:proof-nu-known} to the case of an unknown distribution~$\nu$.
\item Appendix~\ref{app:linCBwK} details the claims of Section~\ref{sec:LinearCBwL}, i.e.,
the extension of the techniques introduced to the setting of linear CBwK.
\item Appendix~\ref{app:simu} reports a simulation study based on realistic data.
\end{itemize}

\clearpage
\section{Industrial Motivation: Market Share Expansion for Loans}
\label{app:industrialmotivation}

We describe the industrial problem we faced and which led to
the setting described in Section~\ref{sec:protocol}.

Incentives and discounts are common practices in many industries to achieve some business objectives; however, there is usually a limit on the number or/and total volume of discounts that can be granted, so companies need to select carefully who should receive them. We consider, for instance, the banking industry, with a business objective of market share expansion:
achieving the highest possible volume of loan subscription (total subscribed amount). Note that in practice, all loan applications need to go first through a risk-assessment process, and
offers are only made if the bank considers that the loan will not put the customer's solvability at risk.
We assume that all the clients concerned here have already gone through this process and
are eligible for getting a loan from a given bank. We formulate the problem in a sequential manner as follows.

At each round $t \geq 1$, a client asks for a credit product. Her/his characteristics are denoted by $\bx_t \in \R^n$,
and encompass the socio-demographic profile, the loan request (amount $x_{\am,t}$, duration $x_{\dur,t}$), etc.
It is reasonable to model these characteristics as independent draws
from a common (possibly unknown) distribution $\nu$. Based on $\bx_t$, the bank
will suggest some standard interest rate $\ir(\bx_t)$ based on its pricing rules; the detail of the rules is not relevant and we assume
that the underlying function $\ir$ is given.
If client $t$ accepts the offered rate $\ir(\bx_t)$ and subscribes to the loan, an event which we denote by $y_t = 1$
and call a conversion,
the bank gets a sales performance (gain on volume) $x_{\am,t}$. Otherwise, the client declines the offer,
which we denote by $y_t = 0$ and the bank gets a null reward.

Actually, to improve the chances of a conversion,
the bank may also offer a discount $a_t \in (0,1]$, or lack of discount $a_t = 0$, on the interest rate.
If it offers a discount $a_t > 0$ and $y_t=1$, the bank will suffer a loss of earnings, equal to $a_t \, \ir(\bx_t) \, \out(\bx_t)$,
where $\out(\bx_t)$ denotes the total outstanding amounts. This loss of earnings
is considered a promotion cost. These promotion costs are summed up to previous such costs and should usually not exceed a
fixed-in-advance budget $B_2 > 0$.
Also, there is usually a fixed-in-advance limit $B_1 > 0$ on the total number of clients who can subscribe with a discount.

Given that the customers' characteristics are i.i.d.,
it is indeed reasonable to assume that $y_t$ follows some Bernoulli distribution
with unknown probability $P(a_t, \bx_t)$. Of course, the higher the discount, the higher the probability of a conversion.

We summarize the setting with the notation of Section~\ref{sec:protocol}.
We assume that discounts are picked in a finite grid $\mathcal{D} = \bigl\{ j/D : j \in \{0,\ldots,D\} \bigr\}$,
so that the action set equals $\cA = \mathcal{D} \cup \{ \anull \}$.
At each round $t \geq 1$, given the customer's characteristics $\bx_t$
and the discount $a_t \in [0,1]$ picked by the bank,
the latter receives the following reward and suffers the following costs:
\begin{align*}
& r(a_t,\bx_t) \, y_t\,, \qquad \mbox{where} \qquad r : (a,\bx) \mapsto x_{\am}\,, \\
\mbox{and} \qquad & \bc(a_t,\bx_t) \, y_t\,, \qquad \mbox{where} \qquad
\bc : (a,\bx) \mapsto \bigl(\mathds{1}_{\{ a \ne 0 \}}, \, a \, \ir(\bx_t) \, \out(\bx_t) \bigr)\,.
\end{align*}
The first component of the cumulative cost vector
measures the total number subscriptions with discounts, and the second component reports the total promotion costs.
The bank wants to enforce
\[
\sum_{t=1}^T \bc(a_t,\bx_t) \, y_t =
\sum_{t=1}^T \bigl(\mathds{1}_{\{ a_t \ne 0 \}}, \, a_t \, \ir(\bx_t) \, \out(\bx_t) \bigr) \, y_t \leq (B_1,\,B_2)
\]
while maximizing the sum of the achieved rewards.

Normalizations in $[0,1]$ both for rewards and cost components may be achieved by
considering the maximal amount $M_{\am}$ and outstanding $M_{\outm}$ that the bank
would allow, and by considering
\[
r : (a,\bx) \mapsto x_{\am}/M_{\am}
\qquad \mbox{and} \qquad
\bc : (a,\bx) \mapsto \bigl(\mathds{1}_{\{ a \ne 0 \}}, \, a \, \ir(\bx_t) \, \out(\bx_t) / M_{\outm}\bigr)
\]
with the alternative budget $(B_1,\,B_2/M_{\outm})$. A single budget parameter $B = \min\bigl\{B_1,\,B_2/M_{\outm}\bigr\}$
may be considered by a final normalization: by dividing the first cost component by $B_1/B > 1$ if $B_1 > B$,
or the second cost component by $(B_2/M_{\outm})/B > 1$ if $B_2/M_{\outm} > B$, respectively.

\clearpage
\section{Detailed Proof of the Regret Bound in Case of a Known Distribution~$\nu$: \\ ~~~~~~~Proof of Theorem~\ref{th:main}}
\label{app:proof-nu-known}

The proof is divided into four steps.

\subsection{First Step: Defining Confidence Intervals on the Probabilities $P(a,\bx)$}
\label{sec:ICP}

The keystone of this step is the following lemma, adapted from~\citet{Faury2020ImprovedOA}:
it provides guarantees for the adapted version of Logistic-UCB1 defined in Phase~1 of
the adaptive policy studied in this article. The lemma actually holds for any sampling strategy of the arms, not just the one
used in Phase~2 of the adaptive policy.

The reasons of the adaptations, lying in different settings being considered, as well as a detailed proof,
are provided in Appendix~\ref{app:Logistic-UCB1}.
We recall that we denote
\[
\kappa = \sup \! \left\{ \frac{1}{\dot{\sig} \bigl(\bphi(a, \bx)^{\transp} \btheta \bigr)} : \bx \in \cX, \ a \in \cA \setminus \{ \anull \}, \ \btheta \in \Theta \right\},
\]
and that since $\dot{\eta} = \eta(1-\eta) \in [0,1/4]$, we always have $\kappa \geq 4$.
We also recall the notation for the maximal Euclidean norm of an element in $\Theta$:
\[
\Arrowvert \Theta \Arrowvert_{\maxn} = \max_{\btheta \in \Theta} \Arrowvert \btheta \Arrowvert\,.
\]

\begin{lemma}[combination of Lemmas~1 and~2 of~\citet{Faury2020ImprovedOA} with minor adjustments, detailed
in Appendix~\ref{app:Logistic-UCB1}]
\label{lemma:pbound_conversion}
Assume that $\kappa < +\infty$.
Fix any sampling strategy and consider the version of Logistic-UCB1
given by~\eqref{eq:besttheta}--\eqref{eq:Wttheta}.
For all $\delta \in (0,1)$, there exists an event $\cE_{\prob,\delta}$ with probability at least $1-\delta$
and such that over $\cE_{\prob,\delta}$:
\begin{multline*}
\forall t \geq 1, \ \forall a \in \cA \setminus \{ \anull \}, \ \forall \bx \in \cX,\qquad
\bigl| \hat{P}_{t}(a,\bx) - P(a,\bx) \bigr| \leq \\
\gamma_{t,\lambda,\delta} \sqrt{\kappa \bigl(\Arrowvert \Theta \Arrowvert_{\maxn} + 1/2)} \,
\bigl\Arrowvert \bphi(a, \bx) \bigr\Arrowvert_{V_{t}^{-1}}\,,
\end{multline*}
\begin{align*}
\mbox{where} \qquad \qquad &
\gamma_{t,\lambda,\delta}  = \sqrt{\lambda} \bigl( \Arrowvert \Theta \Arrowvert_{\maxn} + 1/2 \bigr) + \frac{2}{\sqrt{\lambda}} \ln
\Biggl( \frac{2^m}{\delta} \biggl(1 + \frac{t}{4 m\lambda} \biggr)^{m/2} \Biggr) \\
\mbox{and} \qquad \qquad &
V_{t} = \sum_{s=1}^{t} \bphi(a_s, \bx_s) \bphi(a_s, \bx_s)^{\transp} \Inotanull{s} + \kappa \lambda \id_m\,.
\end{align*}
\end{lemma}

\paragraph{Associated confidence intervals for the $P(a,\bx)$.}
Thanks to this lemma, we consider the upper-confidence bonuses
\begin{equation}
\label{eq:def:vareps}
\varepsilon_t(a,\bx) =
\gamma_{t,\lambda,\delta} \sqrt{\kappa \bigl(\Arrowvert \Theta \Arrowvert_{\maxn} + 1/2)} \,
\bigl\Arrowvert \bphi(a, \bx) \bigr\Arrowvert_{V_{t}^{-1}}\,.
\end{equation}
On the event $\cE_{\prob,\delta}$ of Lemma~\ref{lemma:pbound_conversion},
we have, for all $t \geq 1$, all $a \in \cA \setminus \{ \anull \}$, and all $\bx \in \cX$:
on the one hand,
\begin{equation}
\label{eq:UgepP}
U_{t}(a,\bx) = \min \Bigl\{ \hat{P}_{t}(a,\bx) + \varepsilon_{t}(a,\bx), \, 1 \Bigr\}
\geq \min \bigl\{ P(a, \bx), \, 1 \bigr\} = P(a, \bx)
\end{equation}
and on the other hand,
\begin{equation}
\label{eq:csq1LmEdelta}
U_{t}(a,\bx) = \min \Bigl\{ \hat{P}_{t}(a,\bx) + \varepsilon_{t}(a,\bx), \, 1 \Bigr\}
\leq \hat{P}_{t}(a,\bx) + \varepsilon_{t}(a,\bx) \leq P(a, \bx) + 2\varepsilon_{t}(a,\bx)\,.
\end{equation}

\paragraph{Control of the sum of upper-confidence bonuses.}
According to the proof sketch of Section~\ref{sec:analysis-known}, it
only remains to control the sum of the upper-confidence bonuses at
observed contexts $\bx_t$ and played actions $a_t$. Note that at rounds $t \geq 2$,
we use the bonuses $\varepsilon_{t-1}(a,\bx)$. We prove that
\begin{equation}
\label{eq:defET}
2 \sum_{t=2}^T \varepsilon_{t-1}(a_t,\bx_t) \Inotanull{t}
\leq \underbrace{\gamma_{T,\lambda,\delta} \sqrt{\kappa \bigl(4 \Arrowvert \Theta \Arrowvert_{\maxn} +2\bigr)}
\sqrt{2m T \max\left\{1, \,\,\frac{1}{ \kappa \lambda} \right\} \ln \biggl(1+\frac{T}{\kappa \lambda m}\biggr)}}_{\defeq E_T} \,.
\end{equation}
To that end, we first note that $\gamma_{t-1,\lambda,\delta} \leq \gamma_{T,\lambda,\delta}$:
\[
2 \sum_{t=2}^T \varepsilon_{t-1}(a_t, \bx_t) \Inotanull{t}
\leq
\gamma_{T,\lambda,\delta} \sqrt{\kappa \bigl(4 \Arrowvert \Theta \Arrowvert_{\maxn} +2\bigr)}
\sum_{t=2}^{T} \bigl\Arrowvert \bphi(a_t, \bx_t) \bigr\Arrowvert_{V_{t-1}^{-1}} \Inotanull{t}\,.
\]
Given that $\Arrowvert \bphi \Arrowvert \leq 1$ by assumption,
a direct application of Lemma~\ref{lm:ellpot} below (a classical result of basic algebra for linear bandits)
ensures that
\begin{equation*}
\begin{split}
\sum_{t=1}^{T} \bigl\Arrowvert \bphi(a_t, \bx_t) \bigr\Arrowvert^2_{V_{t-1}^{-1}} \Inotanull{t} & \leq 2 m \max\left\{1, \,\,\frac{1}{ \kappa \lambda} \right\} \ln \biggl(1+\frac{\sum_{t=1}^{T} \Inotanull{t} }{\kappa \lambda m}\biggr).
\end{split}
\end{equation*}
A Cauchy-Schwarz inequality thus entails
\[
\begin{split}
\sum_{t=1}^{T} \bigl\Arrowvert \bphi(a_t, \bx_t) \bigr\Arrowvert_{V_{t-1}^{-1}}  \Inotanull{t}  &  \leq \sqrt{\sum_{t=1}^{T} \Inotanull{t}} \sqrt{2m \max\left\{1, \,\,\frac{1}{ \kappa \lambda} \right\} \ln \biggl(1+\frac{\sum_{t=1}^{T} \Inotanull{t}}{\kappa \lambda m}\biggr)}\\
&  \leq \sqrt{T} \sqrt{2m \max\left\{1, \,\,\frac{1}{ \kappa \lambda} \right\} \ln \biggl(1+\frac{T}{\kappa \lambda m}\biggr)}\,,
\end{split}
\]
which concludes this step.

\begin{lemma}[Elliptic potential and determinant-trace inequality,
cf.\ Lemmas~10 and~11 of~\citet{AbbasiYadkori2011ImprovedAF}, Lemmas~15 and~16 of~\citet{Faury2020ImprovedOA}]
\label{lm:ellpot}
For all $\lambda > 0$ and all sequences $\bu_1,\bu_2,\ldots$ of vectors in $\R^m$ with $\Arrowvert \bu_s \Arrowvert \leq 1$,
defining $U_0 = \lambda \id_m$ and for $t \geq 1$,
\[
U_t = \lambda \id_m + \sum_{s=1}^t \bu_t (\bu_t)^{\transp}\,,
\]
we have, for all $\tau \geq 1$:
\[
\sum_{t=1}^\tau \Arrowvert \bu_t \Arrowvert^2_{U_{t-1}^{-1}}
\leq 2 m \max\left\{1, \,\,\frac{1}{\lambda} \right\} \ln \biggl(1+\frac{\tau}{\lambda m}\biggr).
\]
\end{lemma}

\subsection{Second Step: Dual Formulation of the Sampling Phase (Phase~2) and Consequences}
\label{app:second-step}

In this step, we consider a round $t \geq 2$ for which
the cost constraints of Phase~0 of the adaptive policy are not violated
and the optimization problem $\opt(\nu,\,U_{t-1},B_{t,T})$ is to be solved;
its solution is the policy $\bp_t(h_{t-1},\,\cdot\,)$ used to sample $a_t$
according to $\bp_t(h_{t-1},\bx_t)$.

We first rewrite in its dual form the optimization problem
$\opt(\nu,\,U_{t-1},B_{t,T})$
and show that strong duality holds. As a consequence,
there exists a vector $\bbeta^{\budg,\star}_t \in \R^d$ such that
$\bp_t(h_{t-1},\,\cdot\,)$ may be identified as
\[
\hspace{-.6cm}
\argmax_{\pi: \cX \to \cP(\cA)} \;  \E_{\bX \sim \nu} \! \left[ T \sum_{a \in \cA} \Bigl( r(a,\bX)  - \bigl(\bbeta^{\budg,\star}_t\bigr)^{\, \transp} \bc(a, \bX) \Bigr) \, U_{t-1}(a,\bX) \, \pi_{a}(\bX)  +  \sum_{\bx \in \cX} \sum_{a \in \cA} {\beta_{\bx, a}^{\ppos,\star}} \, \pi_{a}(\bx) \right].
\]
By exploiting the KKT conditions, we are able to get rid of the double sum above and finally
get a $\cX$--pointwise characterization of $\bp_t(h_{t-1},\,\cdot\,)$:
for all $\bx \in \cX$,
\[
\bp_t(h_{t-1},\bx) \in \argmax_{\bq \in \cP(\cA)}
\sum_{a \in \cA} \Bigl( r(a,\bx) - \bigl(\bbeta^{\budg,\star}_t\bigr)^{\, \transp} \bc(a, \bx) \Bigr) \, U_{t-1}(a,\bx) \, q_{a}\,,
\]
where, with no impact, we may replace the $r(a,\bx) - \bigl(\bbeta^{\budg,\star}_t\bigr)^{\, \transp} \bc(a, \bx)$
by their non-negative parts.
The distributions $\bp_t(h_{t-1},\bx)$ may therefore be interpreted as maximizing some upper-confidence bound
on penalized gains (rewards minus some scalarized costs).

We also prove $\opt(\nu,U_{t-1},B_T) \geq B_T (\bbeta^{\budg,\star}_t )^{\, \transp} \mathbf{1}$
based on the KKT conditions.

\paragraph{Primal form of the optimization problem $\opt(\nu,\,U_{t-1},B_{T})$.}
Since $\cX$ is a finite set (this is actually the key place where we need this assumption),
the optimization problem $\opt(\nu,\,U_{t-1},B_T)$ may be stated as
the \emph{opposite} of
\[
\begin{split}
\min_{(\pi_a(\bx)) \in \R^{\cA \times \cX}}
& - T \, \E_{\bX \sim \nu} \! \left[ \sum_{a \in \cA} r(a,\bX) \, U_{t-1}(a, \bX) \, \pi_a(\bX) \right] \vspace{.25cm} \\
\text{under} \qquad & T \, \E_{\bX \sim \nu} \! \left[ \sum_{a \in \cA} \bc(a, \bX) \, U_{t-1}(a,\bX) \, \pi_a(\bX) \right]
\leq B_{T} \mathbf{1}\,, \\
& \forall a \in \cA, \ \forall \bx \in \cX, \quad \pi_a(\bx) \geq 0\,, \\
& \forall \bx \in \cX, \quad \sum_{a \in \cA} \pi_a(\bx) = 1\,. \\
\end{split}
\]
Thanks to the no-op action $\anull \in \cA$, which is used to model abstention and results in null rewards and costs, i.e.,
$r(\anull,\bx) = 0$ and $\bc(\anull,\bx) = \mathbf{0}$ for all $\bx \in \cX$, we may relax the third constraint into
\[
\forall \bx \in \cX, \quad \sum_{a \in \cA} \pi_a(\bx) \leq 1\,.
\]
Indeed, any vector $\bigl( \pi_a(\bx) \bigr) \in \R^{\cA \times \cX}$
satisfying the constraint with $\leq 1$
can be transformed into a vector $\bigl( \pi'_a(\bx) \bigr) \in \R^{\cA \times \cX}$
for which the expected reward and the first and second constraints
remain identical while the third constraint is satisfied with $= 1$:
by adding the necessary probability mass to the components $\pi_{\anull}(\bx)$.

In the sequel, we consider this primal problem with the $\leq 1$ constraint:
\begin{equation}
\label{eq:primal-problem}
\begin{split}
- \opt(\nu,\,U_{t-1},B_T)
= \min_{(\pi_a(\bx)) \in \R^{\cA \times \cX}}
& - T \, \E_{\bX \sim \nu} \! \left[ \sum_{a \in \cA} r(a,\bX) \, U_{t-1}(a, \bX) \, \pi_a(\bX) \right] \vspace{.25cm} \\
\text{under} \qquad & T \, \E_{\bX \sim \nu} \! \left[ \sum_{a \in \cA} \bc(a, \bX) \, U_{t-1}(a,\bX) \, \pi_a(\bX) \right]
\leq B_{T} \mathbf{1}\,, \\
& \forall a \in \cA, \ \forall \bx \in \cX, \quad \pi_a(\bx) \geq 0\,, \\
& \forall \bx \in \cX, \quad \sum_{a \in \cA} \pi_a(\bx) \leq 1\,. \\
\end{split}
\end{equation}
It forms a convex optimization problem, as its objective and its constraints are all affine (\citet[Section~4.2.1]{BV04}).

\paragraph{Lagrangian (dual) form of the optimization problem.}
Denote by $\bbeta^{\budg}, \, \bbeta^{\ppos}, \, \bbeta^{\psum}$ the vector dual variables associated with the constraints on budget [budg], non-negative probability [p-pos],
and sum of probabilities [p-sum], respectively. The vectors $\bbeta^{\ppos}$ and $\bbeta^{\psum}$ have components
$\beta_{\bx}^{\psum}$ and $\beta_{\bx,a}^{\ppos}$ indexed by $\bx \in \cX$ and $a \in \cA$.

We define the Lagrangian associated with our primal problem:
\begin{equation}
\label{eq:lagrangian-main}
\begin{split}
\cL_t\Bigl( \bigl( \pi_a(\bx) \bigl)_{a,\bx}, & \,\, \bbeta^{\budg}, \bbeta^{\psum},  \bbeta^{\ppos} \Bigr) \\
= & - \, T \, \E_{\bX \sim \nu} \! \left[ \sum_{a \in \cA} r(a,\bX) \, U_{t-1}(a, \bX) \, \pi_{a}(\bX) \right] \\
& +  \bigl(\bbeta^{\budg}\bigr)^{\, \transp} \left( T \, \E_{\bX \sim \nu} \! \left[ \sum_{a \in \cA} \bc(a, \bX) \, U_{t-1}(a,\bX) \, \pi_{a} (\bX) \right]
- B_{T} \mathbf{1} \right)\\
& +  \sum_{\bx \in \cX} {\beta_{\bx}^{\psum}} \left( \sum_{a \in \cA} \pi_{a}(\bx) -1 \right)
-  \sum_{\bx \in \cX} \sum_{a \in \cA} {\beta_{\bx, a}^{\ppos}} \, \pi_{a}(\bx)\,.
\end{split}
\end{equation}
The dual problem consists of maximizing
\[
\inf_{(\pi_a(\bx)) \in \R^{\cA \times \cX}} \cL_t\Bigl( \bigl( \pi_a(\bx) \bigl)_{a,\bx}, \bbeta^{\budg}, \bbeta^{\psum},  \bbeta^{\ppos} \Bigr)
\]
under the constraints that all components of the $\bbeta^{\budg}, \bbeta^{\psum}, \bbeta^{\ppos}$ are non-negative,
which we denote in a vector-wise manner by
\[
\bbeta^{\budg} \geq \mathbf{0}\,, \qquad \bbeta^{\psum} \geq \mathbf{0}\,, \qquad \bbeta^{\ppos} \geq \mathbf{0}\,.
\]

\paragraph{Strong duality and consequences.}
We explain why the so-called Slater's condition (\citealp[Section~5.2.3]{BV04}) holds;
it entails that the value $- \opt\bigl(\nu,\,U_{t-1},B_T\bigr)$ of the primal problem
equals the value of the Lagrangian dual problem.
The primal problem is convex, all its constraints are affine, with domain $\R^{\cA \times \cX}$:
Slater's condition therefore reduces to feasibility. And feasibility of the constraints is clear
by taking Dirac masses on $\anull$, i.e., $\pi_a(\bx) = \delta_{\anull}$ for all $\bx \in \cX$.
Since the values of the primal and dual problems are clearly larger than $-T > - \infty$,
Slater's condition also implies
that the dual optimal value is achieved at a dual feasible set of parameters, i.e., that the
the constrained supremum and the infimum defining the dual problem are a minimum and a maximum, respectively.
We may therefore summarize the consequences of Slater's condition as follows:
\[
\begin{split}
- \opt(\nu,U_{t-1},B_T) = \max_{\bbeta^{\budg}, \bbeta^{\psum}, \bbeta^{\ppos}} \ \ & \min_{(\pi_a(\bx)) \in \R^{\cA \times \cX}} \
\cL_t\Bigl( \bigl( \pi_a(\bx) \bigr)_{a,\bx}, \bbeta^{\budg}, \bbeta^{\psum},  \bbeta^{\ppos} \Bigr) \\
\mbox{under} \qquad & \bbeta^{\budg} \geq \mathbf{0}\,, \qquad \bbeta^{\psum} \geq \mathbf{0}\,, \qquad \bbeta^{\ppos} \geq \mathbf{0}\,.
\end{split}
\]

Because of strong duality and the existence of a dual feasible set of parameters, the max-min above equals its min-max counterpart
(\citet[Sections~5.4.1 and~5.4.2]{BV04}):
\[
- \opt(\nu,U_{t-1},B_T) =
\min_{(\pi_a(\bx)) \in \R^{\cA \times \cX}} \
\max_{\bbeta^{\budg} \geq \mathbf{0}, \, \bbeta^{\psum} \geq \mathbf{0}, \, \bbeta^{\ppos} \geq \mathbf{0}} \
\cL_t\Bigl( \bigl(\pi_a(\bx) \bigr)_{a,\bx}, \bbeta^{\budg}, \bbeta^{\psum},  \bbeta^{\ppos} \Bigr)\,. \\
\]
We let $\bbeta^{\budg,\star}_t \geq \mathbf{0}$, $\bbeta^{\psum,\star}_{t} \geq \mathbf{0}$, and $\bbeta^{\ppos,\star}_t \geq \mathbf{0}$ be an optimal dual solution and recall that $\bp_t(h_{t-1},\,\cdot\,)$ denote an optimal primal solution, which, with no loss of generality, may be assumed
satisfying the $\leq 1$ constraints with equality. From \citet[Section~5.4.2]{BV04}, this pair of solutions forms a saddle-point for the
Lagrangian; in particular,
\begin{align}
\label{eq:OPTptLt}
- \opt(\nu,U_{t-1},B_T) & = \cL_t\bigl(\bp_t(h_{t-1},\,\cdot\,), \bbeta^{\budg,\star}_t, \bbeta^{\psum,\star}_{t}, \bbeta^{\ppos,\star}_t \bigr) \\
\nonumber
& = \min_{(\pi_a(\bx)) \in \R^{\cA \times \cX}} \cL_t\Bigl( \bigl(\pi_a(\bx)\bigr), \bbeta^{\budg,\star}_t,
\bbeta^{\psum,\star}_{t}, \bbeta^{\ppos,\star}_t \Bigr) \\
\nonumber
& = \min_{\pi : \cX \to \cP(\cA)} \cL_t(\pi, \bbeta^{\budg,\star}_t, \bbeta^{\psum,\star}_{t}, \bbeta^{\ppos,\star}_t)\,.
\end{align}
The distribution $\bp_t(h_{t-1},\,\cdot\,)$ played thus appears as the argument of the minimum above.

Substituting the definition~\eqref{eq:lagrangian-main} of $\cL_t$ into the characterization~\eqref{eq:OPTptLt},
rearranging the first two terms of $\cL_t$, noting that the third term of $\cL_t$ is null,
and discarding the constant term $B_T \bigl(\bbeta^{\budg,\star}_t\bigr)^{\, \transp} \mathbf{1}$, we get: \smallskip
\begin{align}
\label{eq:argmax-ucb-main}
\hspace{-1cm} \bp_t(h_{t-1},\,\cdot\,) & \in \\
\nonumber
\hspace{-1cm} \phantom{e} \argmin_{\pi: \cX \to \cP(\cA)} & \;  \E_{\bX \sim \nu} \! \left[ - T \sum_{a \in \cA} \Bigl( r(a,\bX)  + \bigl(\bbeta^{\budg,\star}_t\bigr)^{\, \transp} \bc(a, \bX) \Bigr) \, U_{t-1}(a,\bX) \, \pi_{a}(\bX) -  \sum_{\bx \in \cX} \sum_{a \in \cA} {\beta_{\bx, a}^{\ppos,\star}} \, \pi_{a}(\bx) \right] \\
\nonumber
\hspace{-1cm} = \argmax_{\pi: \cX \to \cP(\cA)} & \;  \E_{\bX \sim \nu} \! \left[ T \sum_{a \in \cA} \Bigl( r(a,\bX)  - \bigl(\bbeta^{\budg,\star}_t\bigr)^{\, \transp} \bc(a, \bX) \Bigr) \, U_{t-1}(a,\bX) \, \pi_{a}(\bX)  +  \sum_{\bx \in \cX} \sum_{a \in \cA} {\beta_{\bx, a}^{\ppos,\star}} \, \pi_{a}(\bx) \right].
\end{align}
We further simplify this alternative definition by showing that the sums of the ${\beta_{\bx, a}^{\ppos}} \, \pi_{a}(\bx)$
may be omitted. We do so by exploiting the KKT conditions.

\paragraph{KKT conditions: statement.}
The Karush–Kuhn–Tucker (KKT) conditions (\citet[Section~5.5.3]{BV04})
for the primal optimal $\bp_t(h_{t-1},\,\cdot\,) : \cX \to \cP(\cA)$ and the dual optimal
$\bbeta^{\budg,\star}_t \geq \mathbf{0}$, $\bbeta^{\psum,\star}_{t} \geq \mathbf{0}$, and $\bbeta^{\ppos,\star}_t \geq \mathbf{0}$
imply the following conditions: first, complementary slackness, which reads
\begin{equation}
\label{eq:kkt-complementary-slackness-proba}
\forall a \in \cA, \ \forall \bx \in \cX, \qquad \beta^{\ppos,\star}_{t, \bx, a} \, p_{t,a}(h_{t-1},\bx) = 0
\end{equation}
and
\begin{equation}
\label{eq:kkt-complementary-slackness-budget}
\bigl(\bbeta^{\budg,\star}_t\bigr)^{\, \transp}   \E_{\bX \sim \nu} \! \left[ \sum_{a \in \cA} \bc(a, \bX) \, U_{t-1}(a,\bX) \, p_{t, a}(h_{t-1},\bX) \right]
=  \frac{B_{T}}{T} \, \bigl(\bbeta^{\budg,\star}_t\bigr)^{\, \transp} \mathbf{1}\,;
\end{equation}
second, a stationary condition, based on the fact that the gradient of the Lagrangian function~\eqref{eq:lagrangian-main} with respect to the $\pi_{a}(\bx)$ vanishes. Denoting by $\nu(\bx)$ the probability mass put by $\nu$ on $\bx$, we have:
\begin{multline}
\label{eq:kkt-stationary}
\forall a \in \cA, \ \ \forall \bx \in \cX,  \qquad
T \, r(a,\bx) \, U_{t-1}(a, \bx) \, \nu(\bx)  \\
= \bigl( \bbeta^{\budg,\star}_t \bigr)^{\transp} \Bigl( T \, \bc(a, \bx) \, U_{t-1}(a,\bx) \, \nu(\bx) \Bigr)
+ \beta^{\psum,\star}_{t, \bx} - \beta^{\ppos,\star}_{t, \bx, a}\,.
\end{multline}

\paragraph{KKT conditions: first consequence---final characterization of $\bp_t(h_{t-1},\,\cdot\,)$.}
For now, we only exploit~\eqref{eq:kkt-complementary-slackness-proba}:
this equality and the fact that $\beta^{\ppos,\star}_{t, \bx, a} \, \pi_{a}(\bx)$ is always non-negative
show that, as announced, the characterization~\eqref{eq:argmax-ucb-main} may be further simplified into
\begin{equation}
\label{eq:argmax-ucb-main-simple}
\begin{split}
\bp_t(h_{t-1},\,\cdot\,) \in & \argmax_{\pi: \cX \to \cP(\cA)} \;  \E_{\bX \sim \nu} \! \left[ T \sum_{a \in \cA} \Bigl( r(a,\bX)  - \bigl(\bbeta^{\budg,\star}_t\bigr)^{\, \transp} \bc(a, \bX) \Bigr) \, U_{t-1}(a,\bX) \, \pi_{a}(\bX) \right] \\
= & \argmax_{\pi: \cX \to \cP(\cA)} \;  \E_{\bX \sim \nu} \! \left[ \sum_{a \in \cA} \Bigl( r(a,\bX)  - \bigl(\bbeta^{\budg,\star}_t\bigr)^{\, \transp} \bc(a, \bX) \Bigr) \, U_{t-1}(a,\bX) \, \pi_{a}(\bX) \right].
\end{split}
\end{equation}

In the characterization~\eqref{eq:argmax-ucb-main-simple}, as we got rid of the cross terms,
the maximization may be carried out in a $\cX$--pointwise manner,
i.e., by separately computing each probability distribution $\bp_t(h_{t-1},\bx) \in \cP(\cA)$. More formally,
for each $\bx \in \cX$,
\begin{equation}
\label{eq:charact-pt-x-1}
\bp_t(h_{t-1},\bx) \in \argmax_{\bq \in \cP(\cA)}
\sum_{a \in \cA} \Bigl( r(a,\bx) - \bigl(\bbeta^{\budg,\star}_t\bigr)^{\, \transp} \bc(a, \bx) \Bigr) \, U_{t-1}(a,\bx) \, q_{a} \,.
\end{equation}
In this characterization,
$r(a,\bx)  - \bigl(\bbeta^{\budg,\star}_t\bigr)^{\, \transp} \bc(a, \bx)$ appears as a penalized gain
in case a client with characteristics $\bx$ converts, and $\bp_t(h_{t-1},\bx)$ is obtained by combining an upper-confidence-bound estimation
$U_{t-1}(a,\bx)$ of the conversion rate $P(a,\bx)$ with this penalized gain.

Denote by $(z)_+ = \max\{z,\,0\}$ the non-negative part of $z \in \R$ and fix $\bx \in \cX$.
Given that the no-op action $\anull$ is such that $r(\anull,\bx)  - \bigl(\bbeta^{\budg,\star}_t\bigr)^{\, \transp} \bc(\anull, \bx) = 0$
and $U_{t-1}(a,\bx) \geq 0$ for all $a \in \cA$, in view of the objective,
any distribution $\bq$ should move the probability mass $q_a$ on an action $a\in \cA$ with
$r(a,\bx)  - (\bbeta^{\budg,\star}_t)^{\, \transp} \bc(a, \bx) < 0$ to $\anull$.
As a consequence,
\begin{multline}
\label{eq:charact-pt-x-1ter}
\max_{\bq \in \cP(\cA)}
\sum_{a \in \cA} \Bigl( r(a,\bx) - \bigl(\bbeta^{\budg,\star}_t\bigr)^{\, \transp} \bc(a, \bx) \Bigr) \, U_{t-1}(a,\bx) \, q_a(\bx) \\
= \max_{\bq \in \cP(\cA)}
\sum_{a \in \cA} \Bigl( r(a,\bx) - \bigl(\bbeta^{\budg,\star}_t\bigr)^{\, \transp} \bc(a, \bx) \Bigr)_+ \, U_{t-1}(a,\bx) \, q_a(\bx)
\end{multline}
and
\begin{equation}
\label{eq:charact-pt-x-1bis}
\bp_t(h_{t-1},\bx) \in \argmax_{\bq \in \cP(\cA)}
\sum_{a \in \cA} \Bigl( r(a,\bx) - \bigl(\bbeta^{\budg,\star}_t\bigr)^{\, \transp} \bc(a, \bx) \Bigr)_+ \, U_{t-1}(a,\bx) \, q_{a} \,.
\end{equation}

\paragraph{KKT conditions: a second consequence.}
We first exploit the stationary condition~\eqref{eq:kkt-stationary}.
Multiplying both sides of this equality by $p_{t, a}(h_{t-1},\bx)$, summing
over $\bx \in \cX $ and $a \in A$, we obtain an equality between expectations with respect
to $\bX \sim \nu$:
\begin{equation}
\label{eq:kkt-stationary-sum-a}
\begin{split}
& T \, \E_{\bX \sim \nu}   \! \left[  \sum_{a \in \cA}  r(a,\bX) \, U_{t-1}(a, \bX) \, p_{t, a}(h_{t-1},\bX) \right]  \\
= & \ \bigl( \bbeta^{\budg,\star}_t \bigr)^{\transp} \left( T \, \E_{\bX \sim \nu} \!
\left[  \sum_{a \in \cA} \bc(a, \bX) \, U_{t-1}(a,\bX) \, p_{t, a}(h_{t-1},\bX) \right]  \right) \\
& + \sum_{\bx \in\cX} \beta^{\psum,\star}_{t, \bx} \underbrace{\sum_{a \in \cA}  p_{t, a}(h_{t-1},\bx)}_{=1}
- \sum_{\bx \in \cX} \sum_{a \in \cA}  \underbrace{\beta^{\ppos,\star}_{t, \bx, a}  \, p_{t, a}(h_{t-1},\bx)}_{=0}\,,
\end{split}
\end{equation}
where the equality to~$0$ indicated in the right-hand side correspond to~\eqref{eq:kkt-complementary-slackness-proba}.
We now substitute~\eqref{eq:kkt-complementary-slackness-budget} into~\eqref{eq:kkt-stationary-sum-a}
and obtain
\[
T \, \E_{\bX \sim \nu}   \! \left[  \sum_{a \in \cA}  r(a,\bX) \, U_{t-1}(a, \bX) \, p_{t, a}(h_{t-1},\bX) \right]
= B_T (\bbeta^{\budg,\star}_t )^{\, \transp} \mathbf{1} + \sum_{\bx \in\cX} \beta^{\psum,\star}_{t, \bx}\,.
\]
The left-hand side equals $\opt(\nu,U_{t-1},B_T)$ since $\bp_t(h_{t-1},\,\cdot\,)$ is the solution of the primal
problem. Thus, the equality above entails, by $\bbeta^{\psum,\star}_{t} \geq \mathbf{0}$, that
\begin{equation}
\label{eq:property-strong-duality}
\opt(\nu,U_{t-1},B_T) \geq B_T (\bbeta^{\budg,\star}_t )^{\, \transp} \mathbf{1}\,.
\end{equation}

\subsection{Third Step: Various High-Probability Controls}
\label{app:vhpc}

We prove below that on the intersection between the event $\cE_{\prob,\delta}$ of Lemma~\ref{lemma:pbound_conversion}
and another event $\cE_{\haz,\delta}$, also of probability at least $1-\delta$, we have
simultaneously that for all rounds $t \geq 2$, the policy $\bp_t(h_{t-1},\,\cdot\,)$
is obtained by Phase~2, i.e., by solving $\opt(\nu,U_{t-1},B_T)$, and that
\begin{align}
\label{eq:ccl3rdstep-r}
& \sum_{t=1}^T r(a_t,\bx_t) \, y_t \geq
\sum_{t=2}^T \frac{\opt(\nu,U_{t-1},B_T)}{T} - \sqrt{2T \ln \frac{4}{\delta}} - E_T \\
\nonumber
\mbox{and} \qquad & \sum_{t=1}^T \bc(a_t,\bx_t) \, y_t \leq
\left( B_T + 1 + \sqrt{2T \ln \frac{4d}{\delta}} \right) \mathbf{1}
= (B-1) \mathbf{1} \,,
\end{align}
where the bound $E_T$ was defined in~\eqref{eq:defET} and where we used the definition of $B_T$,
namely,
\begin{equation}
\label{eq:defBT}
B_T = B - 2 - \sqrt{2T \ln \frac{4d}{\delta}}\,.
\end{equation}
This definition requires that
\begin{equation}
\label{eq:defBT-cstr}
B > c \left( 2 + \sqrt{2T \ln \frac{4d}{\delta}} \right)
\end{equation}
for $c=1$, but we will rather assume that the inequality holds with $c=2$.

\paragraph{Applications of the Hoeffding-Azuma inequality to handle the conversions $y_t$.}
We recall that we defined $h_0$ as the empty vector and $h_{t} = (\bx_s, a_s, \,y_s)_{s \leq t}$ for $t \geq 1$.
We introduce the filtration $\cF = (\cF_t)_{t \geq 0}$, with
$\cF_0 = \sigma(a_1,\bx_1)$ and for $t \geq 1$,
\[
\cF_{t} = \sigma\bigl( h_{t}, \, a_{t+1}, \, \bx_{t+1} \bigr)\,.
\]
Then, for all $t \geq 1$, the variables $r(a_t,\bx_t) \, y_t$ and $\bc(a_t,\bx_t) \, y_t$
are $\cF_t$--measurable, with conditional expectations with respect to $\cF_{t-1}$ equal to
\begin{align*}
& \E\bigl[ r(a_t, \bx_t) \, y_t \,\big|\, \cF_{t-1} \bigr] = r(a_t, \bx_t) \, P(a_t,\bx_t) \\
\mbox{and} \qquad &
\E\bigl[ \bc(a_t, \bx_t) \, y_t \,\big|\, \cF_{t-1} \bigr] = \bc(a_t, \bx_t) \, P(a_t,\bx_t)\,.
\end{align*}
Indeed, the conditioning by $\cF_{t-1}$ fixes $a_t$ and $\bx_t$, but not $y_t$,
and exactly means, when $a_t \ne \anull$, integrating over $y_t \sim P(a_t,\bx_t)$.
When $a_t = \anull$, all equalities above remain valid thanks to the abuse of notation
discussed in Section~\ref{sec:protocol}.
Given that $r$ takes values in $[0,1]$ and $\bc$ in $[0,1]^d$,
we may apply $d+1$ times the Hoeffding-Azuma inequality (once for $r$ and each component of $c$)
together with a union bound: there exist two events $\cE_{r,P,\delta}$ and $\cE_{\bc,P,\delta}$, each of probability
at least $1-\delta/4$, such that
\begin{align*}
& \mbox{on} \ \cE_{r,P,\delta}, \qquad
\sum_{t=1}^T r(a_t,\bx_t) \, y_t \geq \sum_{t=1}^T r(a_t, \bx_t) \, P(a_t, \bx_t) -
\sqrt{\frac{T}{2} \ln \frac{4}{\delta}} \\
\mbox{and} \qquad
& \mbox{on} \ \cE_{\bc,P,\delta}, \qquad
\sum_{t=1}^T \bc(a_t,\bx_t) \, y_t  \leq \sqrt{\frac{T}{2} \ln \frac{4d}{\delta}} \,\, \mathbf{1}
+ \sum_{t=1}^T  \bc(a_t, \bx_t) \, P(a_t, \bx_t)\,.
\end{align*}

\paragraph{Applications of the Hoeffding-Azuma inequality to use the properties of the $\bp_t(h_{t-1},\,\cdot\,)$.}
For this sub-step, we rather condition directly by $h_{t-1}$ (instead of $\cF_{t-1}$ as in the
previous step), where $t \geq 2$.
Conditioning by $h_{t-1}$ amounts to integrating first over $a_t \sim \bp_t(h_{t-1},\bx_t)$ and then over $\bx_t \sim \nu$:
this is because of the definition of the random draw of $a_t$ according to $\bp_t(h_{t-1},\bx_t)$
independently from everything else, and because $\bx_t$ is drawn independent from the past according to~$\nu$.
More precisely, for each $t \geq 2$, given that $U_{t-1}$ is $h_{t-1}$--measurable,
we thus have the following equalities:
\begin{align*}
\E\bigl[ r(a_t, \bx_t) \, U_{t-1}(a_t, \bx_t) \,\big|\, h_{t-1} \bigr]
& = \E\!\left[ \sum_{a \in \cA} r(a, \bx_t) \, U_{t-1}(a, \bx_t) \, p_{t,a}(h_{t-1},\bx_t) \,\bigg|\, h_{t-1} \right] \\
& = \E_{\bX \sim \nu} \!\left[ \sum_{a \in \cA} r(a, \bX) \, U_{t-1}(a, \bX) \, p_{t,a}(h_{t-1},\bX) \right] \\
\mbox{and} \quad\qquad
\E\bigl[ \bc(a_t, \bx_t) \, U_{t-1}(a_t, \bx_t) \,\big|\, h_{t-1} \bigr]
& = \E\!\left[ \sum_{a \in \cA} \bc(a, \bx_t) \, U_{t-1}(a, \bx_t) \, p_{t,a}(h_{t-1},\bx_t) \,\bigg|\, h_{t-1} \right] \\
& = \E_{\bX \sim \nu} \!\left[ \sum_{a \in \cA} \bc(a, \bX) \, U_{t-1}(a, \bX) \, p_{t,a}(h_{t-1},\bX) \right],
\end{align*}
where we recall that $\E_{\bX \sim \nu}$ denotes an integration solely over $\bX \sim \nu$.

By definition of $\bp_t(h_{t-1},\,\cdot\,)$, whenever the adaptive policy reaches Phase~2
at a given round~$t \geq 2$:
\begin{align*}
\E_{\bX \sim \nu} \!\left[ \sum_{a \in \cA} r(a, \bX) \, U_{t-1}(a, \bX) \, p_{t,a}(h_{t-1},\bX) \right]
& = \frac{\opt(\nu,U_{t-1},B_T)}{T} \\
\mbox{and} \quad\qquad
\E_{\bX \sim \nu} \!\left[ \sum_{a \in \cA} \bc(a, \bX) \, U_{t-1}(a, \bX) \, p_{t,a}(h_{t-1},\bX) \right]
& \leq \frac{B_T}{T}  \mathbf{1}\,.
\end{align*}
Otherwise, the adaptive policy is stuck in Phase~0, because at least one cost constraint is larger than $B-1$; in
this case, $\bp_{t}(h_{t-1},\bx) = \delta_{\anull}$ for all $\bx \in \cX$ and both expectations above are null.
We may summarize all cases by introducing the indicator function that all cost constraints are smaller than $B-1$,
\[
\mathds{1}_{ \{ \bC_{t-1} \leq (B-1)\mathbf{1} \} }\,, \qquad \mbox{where} \qquad
\bC_{t-1} = \sum_{s = 1}^{t-1} \bc(a_s,\bx_s) \, y_s\,,
\]
and stating that
\begin{align*}
\E_{\bX \sim \nu} \!\left[ \sum_{a \in \cA} r(a, \bX) \, U_{t-1}(a, \bX) \, p_{t,a}(h_{t-1},\bX) \right]
& \geq \mathds{1}_{\{ \bC_{t-1} \leq (B-1)\mathbf{1} \}} \frac{\opt(\nu,U_{t-1},B_T)}{T} \\
\mbox{and} \quad\qquad
\E_{\bX \sim \nu} \!\left[ \sum_{a \in \cA} \bc(a, \bX) \, U_{t-1}(a, \bX) \, p_{t,a}(h_{t-1},\bX) \right]
& \leq \frac{B_T}{T}  \mathbf{1}\,.
\end{align*}

Therefore, a second series of applications of the Hoeffding-Azuma inequality
(at times $2 \leq t \leq T$, i.e., excluding the first round)
entails the existence of two events $\cE_{r,U,\delta}$ and $\cE_{\bc,U,\delta}$, each of probability
at least $1-\delta/4$, such that
\begin{align*}
& \mbox{on} \ \cE_{r,U,\delta}, \qquad
\sum_{t=2}^T r(a_t, \bx_t) \, U_{t-1}(a_t, \bx_t)
\geq \sum_{t=2}^T \mathds{1}_{\{ \bC_{t-1} \leq (B-1)\mathbf{1} \}} \frac{\opt(\nu,U_{t-1},B_T)}{T} - \sqrt{\frac{T}{2} \ln \frac{4}{\delta}} \\
& \mbox{and on} \ \cE_{\bc,U,\delta}, \qquad
\sum_{t=2}^T \bc(a_t,\bx_t) \, U_{t-1}(a_t, \bx_t) \leq \left( B_T + \sqrt{\frac{T}{2} \ln \frac{4d}{\delta}} \right) \mathbf{1}\,.
\end{align*}

\paragraph{Appeal to results of Appendix~\ref{sec:ICP}.}
The inequalities~\eqref{eq:UgepP} and~\eqref{eq:csq1LmEdelta}
state that on the event $\cE_{\prob,\delta}$ of Lemma~\ref{lemma:pbound_conversion},
we have
\[
\forall t \geq 1, \ \ \forall a \in \cA \setminus \{ \anull \}, \ \ \forall \bx \in \cX, \qquad
P(a, \bx) \leq U_{t}(a,\bx) \leq P(a, \bx) + 2\varepsilon_{t}(a,\bx)\,.
\]
Thus, on $\cE_{\prob,\delta}$,
\begin{align*}
\sum_{t=1}^T r(a_t, \bx_t) \, P(a_t, \bx_t)
& \geq \sum_{t=2}^T r(a_t, \bx_t) \, U_{t-1}(a_t, \bx_t)
- 2 \sum_{t=2}^T \varepsilon_{t-1}(a_t,\bx_t) \Inotanull{t} \\
& \geq \sum_{t=2}^T r(a_t, \bx_t) \, U_{t-1}(a_t, \bx_t) - E_T \\
\mbox{and} \qquad
\sum_{t=1}^T  \bc(a_t, \bx_t) \, P(a_t, \bx_t)
& \leq \mathbf{1} + \sum_{t=2}^T \bc(a_t,\bx_t) \, U_{t-1}(a_t, \bx_t)\,,
\end{align*}
where we recall that the bound $E_T$ was defined in~\eqref{eq:defET}.

\paragraph{Conclusion of this step.}
We define
\[
\cE_{\haz,\delta} = \cE_{r,P,\delta} \cap \cE_{\bc,P,\delta} \cap \cE_{r,U,\delta} \cap \cE_{\bc,U,\delta}\,,
\]
which is an event of probability at least $1-\delta$.
On the intersection of $\cE_{\haz,\delta}$ and $\cE_{\prob,\delta}$, by collecting all bounds together,
we have
\[
\sum_{t=1}^T \bc(a_t,\bx_t) \, y_t \leq
\left( B_T + 1 + \sqrt{2T \ln \frac{4d}{\delta}} \right) \mathbf{1}
= (B-1) \mathbf{1}\,.
\]
This shows that on the intersection of $\cE_{\haz,\delta}$ and $\cE_{\prob,\delta}$,
the indicator functions $\mathds{1}_{\{ \bC_{t-1} \leq (B-1)\mathbf{1} \}}$
all equal~$1$, and that the policies $\bp_t(h_{t-1},\,\cdot\,)$ are always obtained
by solving the optimization problems of Phase~2. We conclude this
step by collecting the obtained bounds for rewards and by legitimately replacing
the indicator functions therein by~$1$.

\subsection{Fourth Step: Conclusion}

In this step, we merely combine the bounds exhibited in the first three steps
to obtain the closed-form expression of the regret bound. We then propose suitable
orders of magnitude for the parameters.

\paragraph{Collecting all bounds to get a closed-form regret bound.}
By considering $\bq = \pi^\star(\bx)$ for each $\bx \in \cX$ in~\eqref{eq:charact-pt-x-1bis},
where we recall that $\pi^\star$ is the optimal static policy,
and by the equality~\eqref{eq:charact-pt-x-1ter},
we note that, for all $t \geq 2$,
\begin{align*}
\forall \bx \in \cX, \qquad
\sum_{a \in \cA} \Bigl( r(a,\bx) - & \bigl(\bbeta^{\budg,\star}_t\bigr)^{\, \transp} \bc(a, \bx) \Bigr) \, U_{t-1}(a,\bx) \, p_{t,a}(h_{t-1},\bx) \\
& \geq \sum_{a \in \cA} \Bigl( r(a,\bx) - \bigl(\bbeta^{\budg,\star}_t\bigr)^{\, \transp} \bc(a, \bx) \Bigr)_+ \, U_{t-1}(a,\bx) \, \pi^\star_{a}(\bx)\,.
\end{align*}
On the event $\cE_{\prob,\delta}$ of Lemma~\ref{lemma:pbound_conversion},
the inequality~\eqref{eq:UgepP} states that $U_{t-1}(a,\bx) \geq P(a,\bx)$, which we may substitute in the inequality above
(thanks to the non-negative part in the right-hand side) to get
\begin{align*}
\forall \bx \in \cX, \qquad
\sum_{a \in \cA} \Bigl( r(a,\bx) - & \bigl(\bbeta^{\budg,\star}_t\bigr)^{\, \transp} \bc(a, \bx) \Bigr) \, U_{t-1}(a,\bx) \, p_{t,a}(h_{t-1},\bx) \\
& \geq \sum_{a \in \cA} \Bigl( r(a,\bx) - \bigl(\bbeta^{\budg,\star}_t\bigr)^{\, \transp} \bc(a, \bx) \Bigr)_+ \, P(a,\bx) \, \pi^\star_{a}(\bx) \\
& \geq \sum_{a \in \cA} \Bigl( r(a,\bx) - \bigl(\bbeta^{\budg,\star}_t\bigr)^{\, \transp} \bc(a, \bx) \Bigr) \, P(a,\bx) \, \pi^\star_{a}(\bx)\,.
\end{align*}
We replace the individual $\bx$ by a random variable $\bX \sim \nu$ and integrate over $\bX$: on $\cE_{\prob,\delta}$,
\begin{align}
\label{eq:ucb-opt-BT}
& \E_{\bX \sim \nu} \left[ \sum_{a \in \cA} r(a,\bX)  \, U_{t-1}(a,\bX) \, p_{t, a}(h_{t-1},\bX) \right] \\
\nonumber
& \qquad\qquad - \bigl(\bbeta^{\budg,\star}_t\bigr)^{\, \transp}
\E_{\bX \sim \nu} \left[ \sum_{a \in \cA}  \bc(a, \bX) \, U_{t-1}(a,\bX) \, p_{t, a}(h_{t-1},\bX) \right] \\
\nonumber
\geq \ \ & \underbrace{\E_{\bX \sim \nu} \left[  \sum_{a \in \cA} r(a,\bX) \, P(a,\bX) \, \pi^{\star}_a(\bX) \right]}_{=\opt(\nu,P,B)/T}
- \bigl( \bbeta^{\budg,\star}_t \bigr)^{\, \transp}
\underbrace{\E_{\bX \sim \nu} \left[ \sum_{a \in \cA}  \bc(a, \bX) \, P(a,\bX) \, \pi^{\star}_a(\bX) \right]}_{\leq (B/T) \mathbf{1}}.
\end{align}
The equality to $\opt(\nu,P,B)/T$ and the inequality $\leq (B/T) \mathbf{1}$ above come from the very definition of $\pi^\star$
as the static policy solving $\opt(\nu,P,B)$.
Similarly, Appendix~\ref{app:vhpc} shows that on the event $\cE_{\haz,\delta}$,
for all $2 \leq t \leq T$, the policies $p_{t, a}(h_{t-1},\,\cdot\,)$ are obtained by solving
$\opt(\nu,U_{t-1},B_T)$, so that, by definition,
\[
\E_{\bX \sim \nu} \left[ \sum_{a \in \cA} r(a,\bX)  \, U_{t-1}(a,\bX) \, p_{t, a}(h_{t-1},\bX) \right]
= \frac{\opt(\nu,U_{t-1},B_T)}{T}\,.
\]
Substituting this equality and
the KKT condition~\eqref{eq:kkt-complementary-slackness-budget}
into~\eqref{eq:ucb-opt-BT} and rearranging, we see that we proved so far that
on the intersection of $\cE_{\prob,\delta}$ and $\cE_{\haz,\delta}$,
\[
\forall \, 2 \leq t \leq T, \qquad
\opt(\nu,P,B) - \opt(\nu,U_{t-1}, B_T)
\leq
(B-B_T) \, \bigl(\bbeta^{\budg,\star}_t\bigr)^{\, \transp} \mathbf{1}\,.
\]
We may also substitute~\eqref{eq:property-strong-duality}, i.e.,
$(\bbeta^{\budg,\star}_t )^{\, \transp} \mathbf{1} \leq \opt(\nu,U_{t-1},B_T)/B_T$ and finally get that
on the intersection of $\cE_{\prob,\delta}$ and $\cE_{\haz,\delta}$,
\[
\forall \, 2 \leq t \leq T, \qquad
\opt(\nu,P,B) - \opt(\nu,U_{t-1}, B_T)
\leq \left( \frac{B}{B_T} - 1 \right) \opt(\nu,U_{t-1},B_T)\,.
\]
Summing over $2 \leq t \leq T$ and using that $\opt(\nu,P,B)/T \leq 1$
by definition and by the fact that $r$ takes values in $[0,1]$,
we obtain
\begin{align}
\nonumber
\opt(\nu,P,B) - \sum_{t=2}^T \frac{\opt(\nu,U_{t-1}, B_T)}{T}
& \leq 1 + \sum_{t=2}^T \frac{\opt(\nu,P,B) - \opt(\nu,U_{t-1}, B_T)}{T} \\
\label{eq:opt-diff}
& \leq 1 + \left( \frac{B}{B_T} - 1 \right) \sum_{t=2}^T \frac{\opt(\nu,U_{t-1}, B_T)}{T}\,.
\end{align}
Distinguishing the cases
\[
\opt(\nu,P,B) - \sum_{t=2}^T \frac{\opt(\nu,U_{t-1}, B_T)}{T} \leq 0 \quad \ \mbox{and} \ \quad
\opt(\nu,P,B) - \sum_{t=2}^T \frac{\opt(\nu,U_{t-1}, B_T)}{T} \geq 0\,,
\]
we see, based on~\eqref{eq:opt-diff}, that in both cases
\[
\opt(\nu,P,B) - \sum_{t=2}^T \frac{\opt(\nu,U_{t-1}, B_T)}{T}
\leq 1 + \left( \frac{B}{B_T} - 1 \right) \opt(\nu,P,B)\,.
\]
We finally substitute~\eqref{eq:ccl3rdstep-r} and have proved,
as claimed, that on the intersection of $\cE_{\prob,\delta}$ and $\cE_{\haz,\delta}$,
\begin{equation}
\label{eq:bdreg-almostfinal}
\opt(\nu,P,B) -
\sum_{t=1}^T r(a_t,\bx_t) \, y_t \leq
1 + \left( \frac{B}{B_T} - 1 \right) \opt(\nu,P,B)
+ E_T + \sqrt{2T \ln \frac{4}{\delta}}\,,
\end{equation}
where we recall that $E_T$ was defined in~\eqref{eq:defET}.

\paragraph{Improving readability and setting $\lambda$.}
As indicated, we require~\eqref{eq:defBT-cstr} with $c=2$.
By the definition~\eqref{eq:defBT} of $B_T$
and the inequality $1/(1-u) \leq 1+2u$ for $u \in (0,1/2)$,
we have
\begin{equation}
\label{eq:BBT}
\frac{B}{B_T} - 1
\leq \frac{2 \bigl( 2 + \sqrt{2 T \ln (4d/\delta)} \bigr)}{B}\,,
\end{equation}
which we may substitute in~\eqref{eq:bdreg-almostfinal}. It only
remains to deal with a bound on $E_T$ to conclude
the proof of Theorem~\ref{th:main}.

We set below a value of $\lambda$ larger than~$1$. Recalling that $\kappa \geq 4$, we may
already bound $E_T$ as
\begin{equation}
\label{eq:bdETpart1}
E_T \leq \gamma_{T,\lambda,\delta} \sqrt{\kappa \bigl(4 \Arrowvert \Theta \Arrowvert_{\maxn} +2\bigr)}
\sqrt{2m T \ln \biggl(1+\frac{T}{4 m}\biggr)}
= 2 \gamma_{T,\lambda,\delta} \sqrt{\kappa \bigl(2 \Arrowvert \Theta \Arrowvert_{\maxn} +1\bigr)}
\sqrt{m T \ln \biggl(1+\frac{T}{4 m}\biggr)}\,,
\end{equation}
where $\gamma_{T,\lambda,\delta}$, defined in the statement of Lemma~\ref{lemma:pbound_conversion},
may itself be bounded by
\[
\gamma_{T,\lambda,\delta} \leq \sqrt{\lambda} \bigl( \Arrowvert \Theta \Arrowvert_{\maxn} + 1/2 \bigr) + \frac{2}{\sqrt{\lambda}} \ln
\Biggl( \frac{2^m}{\delta} \biggl(1 + \frac{T}{4 m} \biggr)^{m/2} \Biggr)\,.
\]
For the sake of simplicity, we set the value of $\lambda$ by only optimizing
the orders of magnitude in $m$ and $T$ of (this upper bound on) $\gamma_{T,\lambda,\delta}$,
i.e., by considering
\[
\sqrt{\lambda} + \frac{m}{\sqrt{\lambda}} \ln T\,.
\]
We take $\lambda = m \ln(1+T/m)$, which is indeed larger than~$1$ given that $T \geq 2m$ and $\ln(1+T/m) \geq 1$.
We have
\[
\frac{2}{\sqrt{\lambda}} \ln \frac{1}{\delta} \leq \ln \frac{1}{\delta}
\qquad \mbox{and}
\qquad \frac{2}{\sqrt{\lambda}} \ln 2^m \leq (2 \ln 2) \sqrt{m} \leq 2\sqrt{m}\,,
\]
as well as
\[
\frac{2}{\sqrt{\lambda}} \ln \Biggl( \biggl(1 + \frac{T}{4 m} \biggr)^{m/2} \Biggr)
\leq \frac{\sqrt{m}}{\sqrt{\ln(1+T/m)}} \ln \biggl(1 + \frac{T}{4 m} \biggr) \leq \sqrt{m \ln \biggl(1 + \frac{T}{4 m} \biggr)}\,.
\]
All in all,
\begin{equation}
\label{eq:bdETpart2}
\gamma_{T,\lambda,\delta} \leq \bigl( \sqrt{m} + \Arrowvert \Theta \Arrowvert_{\maxn} + 1/2 \bigr)
\sqrt{\ln \biggl(1 + \frac{T}{m} \biggr)} + 2\sqrt{m} + \ln \frac{1}{\delta} \,.
\end{equation}
Combining~\eqref{eq:bdETpart1} and~\eqref{eq:bdETpart2}, we showed:
\begin{equation}
\label{eq:bdETpart3}
E_T \leq 2\sqrt{\kappa \bigl(2 \Arrowvert \Theta \Arrowvert_{\maxn} +1\bigr)}
\sqrt{m T \ln \biggl(1+\frac{T}{4 m}\biggr)}
\left( \bigl( \sqrt{m} + \Arrowvert \Theta \Arrowvert_{\maxn} + 1/2 \bigr)
\sqrt{\ln \biggl(1 + \frac{T}{m} \biggr)} + 2\sqrt{m} + \ln \frac{1}{\delta} \right).
\end{equation}

\clearpage
\section{Adaptation of the Logistic-UCB1 Strategy of~\citet{Faury2020ImprovedOA}: \\ ~~~~~~~Proof of Lemma~\ref{lemma:pbound_conversion}}
\label{app:Logistic-UCB1}

The proof is copied from the proof of \citet[Lemmas~1 and~2]{Faury2020ImprovedOA}, with minor adjustments.
We mostly provide it for the sake of self-completeness.

The adjustments are required
because the setting of~\citet{Faury2020ImprovedOA} is slightly different: their action set is a subset $\cX \subseteq \R^n$,
and when the learner picks an action $\bx_t \in \cX$ at round $t$, the obtained reward $r_t \in \{0,1\}$
is drawn independently at random according to a Bernoulli distribution with parameter
$\sig \bigl( \bx_t^{\transp} \btheta_\star \bigr)$, where $\btheta_\star \in \R^n$ is unknown. The learner
then only observes $r_t$ and not what would have been obtained with a different choice of action.

Therefore, the $\bx_t$ and $r_t$ of \citet{Faury2020ImprovedOA} correspond
to $\bphi(a_t, \bx_t)$ and $y_t$ in our setting. The main difference is that while
the learner has a full control over the choice of $\bx_t \in \cX$ in the setting of \citet{Faury2020ImprovedOA},
in our setting, $\bx_t \in \cX$ is drawn by the environment and the learner only picks $a_t \in \cA$;
the learner therefore does not have a full control over $\bphi(a_t, \bx_t)$.
This is why we carefully check in the present appendix that the results by \citet{Faury2020ImprovedOA}
(namely, their Lemmas~1 and~2) extend to our setting.

\paragraph{Reminder --- A tail inequality for self-normalized martingales.}
Theorem~1 of \citet{Faury2020ImprovedOA} reads as follows.
Let $\cF = (\cF_t)_{t \geq 0}$ be a filtration, $(\bU_t)_{t \geq 1}$ an $\cF$--adapted stochastic vector process
in $\R^m$ such that $\Arrowvert \bU_t \Arrowvert \leq 1$ a.s.\ for all $t \geq 1$, and
$(\Delta_t)_{t \geq 1}$ an $\cF$--martingale difference sequence with
$|\Delta_t| \leq 1$ a.s.\ for all $t \geq 1$; i.e., for all $t \geq 1$,
the random variable $\Delta_t$ is $\cF_t$--measurable and
\[
\E[ \Delta_t \, | \, \cF_{t-1} ] = 0 \qquad \mbox{a.s.}
\]
Denote $\sigma_t^2 = \E \bigl[ \Delta_t^2 \, | \, \cF_{t-1} \bigr]$, let $\lambda > 0$, and define, for $t \geq 1$:
\[
S_t = \sum_{s=1}^t \Delta_s \bU_s
\qquad \mbox{and} \qquad
M_t = \lambda \id_m + \sum_{s=1}^t \sigma_s^2 \bU_s \bU_s^{\,\,\transp}\,.
\]
Then, for all $\delta \in (0,1)$, with probability at least $1-\delta$:
\begin{equation}
\label{eq:th:1F}
\forall t \geq 1, \qquad
\Arrowvert S_t \Arrowvert_{M_t^{-1}} \leq \frac{\sqrt{\lambda}}{2} + \frac{2}{\sqrt{\lambda}} \ln
\biggr( \frac{2^m \text{det}(M_{t})^{1/2} \lambda^{-m/2}}{\delta} \biggr).
\end{equation}

The result above is proved by \citet{Faury2020ImprovedOA}
based on Laplace's method on supermartingales, which is a standard argument to provide confidence
bounds on self-normalized sums of conditionally centered random vectors and was
previously introduced, in the context of linear contextual bandits, by \citet[Theorem~2]{AbbasiYadkori2011ImprovedAF};
see also the monograph by \citet[Theorem 20.2]{Lattimore2020BA}.

\paragraph{Step~1 --- A martingale control.}
We apply~\eqref{eq:th:1F} to the following elements.
We take as filtration $\cF = (\cF_t)_{t \geq 0}$, with
$\cF_0 = \sigma(a_1,\bx_1)$ and for $t \geq 1$,
\[
\cF_{t} = \sigma\bigl( h_{t}, \, a_{t+1}, \, \bx_{t+1} \bigr)\,,
\]
where we recall $h_{t} = (\bx_s, a_s, \,y_s)_{s \leq t}$.
We set $\bU_t = \bphi(a_t, \bx_t)$, which is indeed $\cF_t$--measurable
and with Euclidean norm smaller than~$1$ (thanks to the normalization assumed in Section~\ref{sec:protocol}).
Finally, we set, for $t \geq 1$,
\[
\Delta_t =
\begin{cases*}
0 & if $a_t = \anull$, \\
y_t - {\sig}\bigl( \bphi(a_t, \bx_t)^{\transp} \btheta_\star \bigr) & if $a_t \ne \anull$,
\end{cases*}
\]
which we rewrite, by the abuses of notation indicated in Section~\ref{sec:protocol},
\[
\Delta_t = \Bigl( y_t - {\sig}\bigl( \bphi(a_t, \bx_t)^{\transp} \btheta_\star \bigr) \Bigr) \, \Inotanull{t}
= \bigl( y_t - P(a_t,\bx_t) \bigr) \, \Inotanull{t} \,.
\]
The conditioning by $\cF_{t-1}$ fixes $a_t$ and $\bx_t$, but not $y_t$,
and exactly means, when $a_t \ne \anull$, integrating over $y_t \sim P(a_t,\bx_t)$.
We therefore have that $\Delta_t$ is $\cF_t$--measurable, with
\[
\E[ \Delta_t \, | \, \cF_{t-1} \bigr] = 0 \qquad \mbox{a.s.};
\]
that is, $(\Delta_t)_{t \geq 1}$ appears as an $\cF$--martingale difference sequence, satisfying the boundedness-by-1 constraint.
We may therefore apply the result~\eqref{eq:th:1F}. To do so, we first compute the
conditional variances of the $\Delta_t$: for $t \geq 1$,
\begin{multline*}
\E \bigl[ \Delta_t^2 \, | \, \cF_{t-1} \bigr]
= \E\biggl[ \Bigl( y_t - {\sig}\bigl( \bphi(a_t, \bx_t)^{\transp} \btheta_\star \bigr) \Bigr)^2 \,\Big|\, \cF_{t-1} \biggr] \, \Inotanull{t} \\
= {\sig}\bigl( \bphi(a_t, \bx_t)^{\transp} \btheta_\star \bigr) \,\, \Bigl( 1 - {\sig}\bigl( \bphi(a_t, \bx_t)^{\transp} \btheta_\star \bigr) \Bigr)
\, \Inotanull{t} = \dot{\sig}\bigl( \bphi(a_t, \bx_t)^{\transp} \btheta_\star \bigr) \, \Inotanull{t}\,,
\end{multline*}
where we used the fact that $\eta(1-\eta) = \dot{\sig}$.
We rewrite
\[
S_t = \sum_{s=1}^t \Delta_t \, \bphi(a_s, \bx_s) =
\sum_{s=1}^t \Bigl( y_s - {\sig}\bigl( \bphi(a_t, \bx_t)^{\transp} \btheta_\star \bigr) \Bigr) \bphi(a_s, \bx_s) \Inotanull{s}
\]
and note that $M_t = W_t(\btheta)$ with the definition~\eqref{eq:Wttheta}.
The control~\eqref{eq:th:1F} may be rewritten as follows: for all $\delta \in (0,1)$, with probability at least $1-\delta$,
\[
\forall t \geq 1, \qquad \Arrowvert S_{t} \Arrowvert_{W_{t}(\btheta_\star)^{-1}}
\leq \frac{\sqrt{\lambda}}{2} \, + \frac{2}{\sqrt{\lambda}} \,  \ln \biggr( \frac{2^m \text{det}\bigr(W_{t}(\btheta_\star)\bigr)^{1/2} \lambda^{-m/2}}{\delta} \biggr).
\]
As $\dot{\eta} = \eta(1-\eta) \in [0,1/4]$ and as $\Arrowvert \bphi \Arrowvert \leq 1$,
we have, by a standard trace-determinant inequality (see, e.g., \citet[Lemma~10]{AbbasiYadkori2011ImprovedAF}),
\begin{equation*}
\begin{split}
\text{det}\bigr(W_{t}(\btheta_\star)\bigr) & \leq
\left( \lambda + \frac{1}{m} \sum_{s=1}^{t}
\dot{\sig} \bigl( \bphi(a_s, \bx_s)^{\transp} \btheta_\star \bigr) \,
\bigl\Arrowvert \bphi(a_s, \bx_s) \bigr\Arrowvert^2 \, \Inotanull{s} \right)^{\! m}  \\
& \leq \left( \lambda + \frac{1}{4m} \sum_{s=1}^{t} \Inotanull{s} \right)^{\! m}
\leq \left( \lambda + \frac{t}{4m} \right)^{\! m}.
\end{split}
\end{equation*}
Combining the two inequalities, we have proved that for all $\delta \in (0,1)$, with probability at least $1-\delta$,
\begin{equation}
\label{eq:ctrl-mart-Faury}
\forall t \geq 1, \qquad \Arrowvert S_{t} \Arrowvert_{W_{t}(\btheta_\star)^{-1}}
\leq \gamma_{t,\lambda,\delta} - \sqrt{\lambda} \Arrowvert \Theta \Arrowvert_{\maxn}\,,
\end{equation}
where~$\gamma_{t,\lambda,\delta}$ was defined in Lemma~\ref{lemma:pbound_conversion}.

\paragraph{Step~2 --- Application of the martingale control.}
The martingale control~\eqref{eq:ctrl-mart-Faury} is applied as follows.
We show below that the definition~\eqref{eq:besttheta} of~$\tilde{\btheta}_t$ entails that
\begin{equation}
\label{eq:StPsi}
S_t - \lambda\,{\btheta}_\star  = \Psi_t \bigl( \tilde{\btheta}_t \bigr) - \Psi_t( {\btheta}_\star )\,,
\end{equation}
where $\Psi_t$ was defined in~\eqref{eq:projtheta}.
Taking the $\Arrowvert\,\cdot\,\Arrowvert_{W_t(\btheta_\star)^{-1}}$ norms of both sides,
together with a triangle inequality (keeping in mind the boundedness of $\Theta$)
and noting that
\[
\Arrowvert \btheta \Arrowvert_{W_t(\btheta_\star)^{-1}} \leq
\Arrowvert \btheta \Arrowvert_{(\lambda \id_m)^{-1}} = \Arrowvert \btheta \Arrowvert / \sqrt{\lambda}\,,
\]
finally yields that for all $\delta \in (0,1)$, with probability at least $1-\delta$,
\begin{equation}
\label{eq:resultmartcontrol}
\forall t \geq 1,  \qquad
\Bigl\Arrowvert \Psi_t \bigl( \tilde{\btheta}_t \bigr) -
\Psi_t({\btheta}_\star) \Bigr\Arrowvert_{W_t(\btheta_\star)^{-1}}
\leq \Arrowvert S_t \Arrowvert_{W_t(\btheta_\star)^{-1}} + \sqrt{\lambda} \Arrowvert \btheta \Arrowvert
\leq \gamma_{t,\lambda,\delta}\,.
\end{equation}
We now show~\eqref{eq:StPsi}.
The gradient of the continuously differentiable function
\[
\btheta \in \R^m \longmapsto
\sum_{s=1}^{t} \Inotanull{s} \biggl( y_{s} \ln \sig \bigl( \bphi(a_s, \bx_s)^{\transp} \btheta \bigr)
+ (1- y_{s}) \ln \Bigl( 1 - \sig \bigl( \bphi(a_s, \bx_s)^{\transp} \btheta \bigr) \Bigr) \biggr)
- \frac{\lambda}{2} \Arrowvert \btheta \Arrowvert
\]
vanishes
at the point $\tilde{\btheta}_{t}$ where it achieves its maximum, i.e., $\tilde{\btheta}_{t}$
defined in~\eqref{eq:besttheta} satisfies
\begin{multline*}
\sum_{s=1}^{t} \Bigl( \sig \bigl(\bphi(a_s, \bx_s)^{\transp} \tilde{\btheta}_{t} \bigr) - y_s \Bigr) \bphi(a_s, \bx_s) \Inotanull{s} \\
= \sum_{s=1}^{t} \left( y_s \, \frac{\dot{\sig} \bigl(\bphi(a_s, \bx_s)^{\transp} \tilde{\btheta}_{t} \bigr) \, \bphi(a_s, \bx_s)}{\sig \bigl(\bphi(a_s, \bx_s)^{\transp} \tilde{\btheta}_{t} \bigr)} + (1-y_s) \, \frac{\dot{\sig} \bigl(\bphi(a_s, \bx_s)^{\transp} \tilde{\btheta}_{t} \bigr) \, \bphi(a_s, \bx_s)}{1-\sig \bigl(\bphi(a_s, \bx_s)^{\transp} \tilde{\btheta}_{t} \bigr)} \right) \Inotanull{s}
= \lambda \, \tilde{\btheta}_{t}\,,
\end{multline*}
where we used $\dot{\eta} = \eta(1-\eta)$ to get the first equality.
In particular,
\[
\Psi_t \bigl( \tilde{\btheta}_t \bigr) =
\sum_{s=1}^{t} \sig \bigl(\bphi(a_s, \bx_s)^{\transp} \tilde{\btheta}_t \bigr) \, \bphi(a_s, \bx_s) \Inotanull{s} + \lambda \, \tilde{\btheta}_t
= \sum_{s=1}^{t} y_s \, \bphi(a_s, \bx_s) \Inotanull{s}\,,
\]
hence the stated rewriting~\eqref{eq:StPsi}.

\paragraph{Step 3 --- Bound on prediction error.}
We now proceed to bounding, for all $a \in \cA \setminus \{ \anull \}$ and~$\bx$:
\[
\bigl| P(a,\bx)  - \hat{P}_{t}(a,\bx) \bigr|
= \Bigl| \sig \bigl( \bphi(a, \bx)^{\transp}
\btheta_\star \bigr) - \sig \bigl( \bphi(a, \bx)^{\transp}
\hat{\btheta}_{t} \bigr) \Bigr|\,.
\]
As $\eta$ is $1/4$--Lipschitz (given $\dot{\eta} = \eta(1-\eta) \in [0,1/4]$),
\begin{equation}
\label{eq:err-bound-start}
\bigl| \hat{P}_{t}(a,\bx) - P(a,\bx) \bigr|
\leq \frac{1}{4} \, \Bigl| \bphi(a, \bx)^{\transp} (\btheta_\star - \hat{\btheta}_{t} ) \Bigr|\,.
\end{equation}
For two $m \times m$ symmetric definite positive matrices $M$ and $M'$, we write
$M \succeq M'$ whenever $\Arrowvert \bv \Arrowvert_M \geq \Arrowvert \bv \Arrowvert_{M'}$
for all $\bv \in \R^m$. This inequality entails $(M')^{-1} \succeq M^{-1}$.
We introduce below a symmetric definite positive matrix $G_t$ such that
\begin{equation}
\label{eq:Gtsucceq}
G_t \succeq \kappa^{-1} V_t\,, \qquad
G_t \succeq \bigl(1 + 2 \Arrowvert \Theta \Arrowvert_{\maxn} \bigr)^{-1} \, W_t\bigl( \hat{\btheta}_t \bigr)\,, \qquad
G_t \succeq \bigl(1 + 2 \Arrowvert \Theta \Arrowvert_{\maxn} \bigr)^{-1} \, W_t({\btheta}_\star)
\end{equation}
and
\begin{equation}
\label{eq:PsieqGt}
\Psi_t \bigl( \hat{\btheta}_t \bigr) - \Psi_t\bigl( {\btheta}_\star \bigr) = G_t \bigl( \hat{\btheta}_t - {\btheta}_\star \bigr)\,.
\end{equation}
Based on all these properties, we get, by a Cauchy-Schwarz inequality
in the norms $\Arrowvert \,\cdot\, \Arrowvert_{G_t}$ and $\Arrowvert \,\cdot\, \Arrowvert_{G_t^{-1}}$:
\begin{align*}
\Bigl| \bphi(a, \bx)^{\transp} & (\btheta_\star - \hat{\btheta}_{t} ) \Bigr| \\
& \leq
\bigl\| \bphi(a, \bx) \bigr\|_{G_t^{-1}} \, \bigl\| \btheta_\star - \tilde{\btheta}_{t} \bigr\|_{G_t} \\
& = \bigl\| \bphi(a, \bx) \bigr\|_{G_t^{-1}} \, \Bigl\| G_t \bigl( \btheta_\star - \tilde{\btheta}_{t} \bigr) \Bigr\|_{G_t^{-1}} \\
& = \bigl\| \bphi(a, \bx) \bigr\|_{G_t^{-1}} \, \Bigl\| \Psi \bigl( \hat{\btheta}_t \bigr) - \Psi\bigl( {\btheta}_\star \bigr) \Bigr\|_{G_t^{-1}}
\leq  \sqrt{\kappa} \, \bigl\| \bphi(a, \bx) \bigr\|_{V_t^{-1}} \, \Bigl\| \Psi \bigl( \hat{\btheta}_t \bigr) - \Psi\bigl( {\btheta}_\star \bigr) \Bigr\|_{G_t^{-1}}\,.
\end{align*}
A triangle inequality for the first inequality, followed by applying~\eqref{eq:Gtsucceq} for
the second inequality, and applying the definition of~\eqref{eq:projtheta} as a projection
for the third inequality, shows that
\begin{align*}
\Bigl\| \Psi \bigl( \hat{\btheta}_t \bigr) - \Psi\bigl( {\btheta}_\star \bigr) \Bigr\|_{G_t^{-1}}
& \leq \Bigl\| \Psi\bigl( \hat{\btheta}_t \bigr) - \Psi \bigl( \tilde{\btheta}_t \bigr) \Bigr\|_{G_t^{-1}}
+ \Bigl\| \Psi\bigl( {\btheta}_\star \bigr) - \Psi \bigl( \tilde{\btheta}_t \bigr) \Bigr\|_{G_t^{-1}}\\
& \leq \sqrt{1 + 2 \Arrowvert \Theta \Arrowvert_{\maxn}} \biggl(
\Bigl\| \Psi\bigl( \hat{\btheta}_t \bigr) - \Psi \bigl( \tilde{\btheta}_t \bigr) \Bigr\|_{W_t(\hat{\btheta}_t)^{-1}} +
\Bigl\| \Psi({\btheta}_\star) - \Psi \bigl( \tilde{\btheta}_t \bigr) \Bigr\|_{W_t({\btheta}_\star)^{-1}}
\biggr) \\
& \leq 2 \sqrt{1 + 2 \Arrowvert \Theta \Arrowvert_{\maxn}} \,\,
\Bigl\| \Psi({\btheta}_\star) - \Psi \bigl( \tilde{\btheta}_t \bigr) \Bigr\|_{W_t({\btheta}_\star)^{-1}}\,.
\end{align*}
Substituting the martingale control~\eqref{eq:resultmartcontrol} and collecting all bounds together
finally yields
\[
\bigl| \hat{P}_{t}(a,\bx) - P(a,\bx) \bigr|
\leq \frac{\sqrt{\kappa}}{2} \, \bigl\| \bphi(a, \bx) \bigr\|_{V_t^{-1}} \, \sqrt{1 + 2 \Arrowvert \Theta \Arrowvert_{\maxn}}
\,\, \gamma_{t,\lambda,\delta}  \,,
\]
as desired.

Note in particular that~\eqref{eq:resultmartcontrol} holds for all $t \geq 1$
with probability $1 - \delta$,
and that we took care of the dependencies on $a$ and $\bx$ through the
$\bigl\| \bphi(a, \bx) \bigr\|_{V_t^{-1}}$ term.
This explains why the result of Lemma~\ref{lemma:pbound_conversion}
holds with probability $1 - \delta$ for all $t \geq 1$, all $a \in \cA \setminus \{ \anull \}$ and all $\bx \in \cX$.

\paragraph{Step 4 --- Construction of the matrix $G_t$.}
It only remains to show that there exists a symmetric definite positive matrix $G_t$ such
that~\eqref{eq:Gtsucceq} and~\eqref{eq:PsieqGt} hold.
We define
\begin{align*}
& G_t = \int_{[0,1]} W_t\bigl( v \, \hat{\btheta}_{t} + (1-v)\btheta_\star \bigr) \d v \\
& = \lambda \id_m +
\sum_{s=1}^t \left( \int_{[0,1]} \dot{\sig} \bigl( v \, \bphi(a_s, \bx_s)^{\transp} \hat{\btheta}_{t} +
(1-v) \bphi(a_s, \bx_s)^{\transp} \btheta_\star \bigr) \d v \right)
\bphi(a_s, \bx_s) \bphi(a_s, \bx_s)^{\transp} \Inotanull{s} \,.
\end{align*}
The thus defined matrix $G_t$ is indeed symmetric definite positive matrix.
By definition of $\kappa$ and the fact that $\Theta$ is convex, we have,
for all $v \in [0,1]$,
\[
\dot{\sig} \bigl( v \, \bphi(a_s, \bx_s)^{\transp} \hat{\btheta}_{t} +
(1-v) \bphi(a_s, \bx_s)^{\transp} \btheta_\star \bigr)
= \dot{\sig} \Bigl( \bphi(a_s, \bx_s)^{\transp} \bigl( v \hat{\btheta}_{t} +
(1-v) \btheta_\star \bigr) \Bigr) \geq \kappa^{-1}\,,
\]
which also immediately entails
the first inequality of~\eqref{eq:Gtsucceq}.
To prove~\eqref{eq:PsieqGt},
we introduce, for $s \geq 1$ such that $a_s \ne \anull$, the continuously differentiable function
\[
f_{s,t} :
v \in [0,1] \longmapsto f_{s,t}(v) = \sig \bigl( v \, \bphi(a_s, \bx_s)^{\transp} \hat{\btheta}_{t} + (1-v) \bphi(a_s, \bx_s)^{\transp} \btheta_\star \bigr)\,,
\]
with derivative
\[
\dot{f}_{s,t}(v) =
\dot{\sig} \bigl( v \, \bphi(a_s, \bx_s)^{\transp} \hat{\btheta}_{t} + (1-v) \bphi(a_s, \bx_s)^{\transp} \btheta_\star \bigr)
\,\, \bphi(a_s, \bx_s)^{\transp} \! \bigl( \hat{\btheta}_{t} - \btheta_\star \bigr)\,,
\]
and we have
\[
\sig \bigl( \bphi(a_s, \bx_s)^{\transp} \hat{\btheta}_{t} \bigr)
- \sig \bigl( \bphi(a_s, \bx_s)^{\transp} \btheta_\star \bigr)
= f_{s,t}(1) - f_{s,t}(0) = \int_{[0,1]} \dot{f}_{s,t}(v) \d v\,,
\]
These facts, combined with the definition of $G_t$, immediately entail~\eqref{eq:PsieqGt}.

It only remains to prove the third inequality of~\eqref{eq:Gtsucceq}, namely,
\[
G_t \succeq \bigl(1 + 2 \Arrowvert \Theta \Arrowvert_{\maxn} \bigr)^{-1} \, W_t( {\btheta}_\star )\,,
\]
as the second one is obtained by symmetry from there, by exchanging the roles of ${\btheta}_\star$ and~$\hat{\btheta}_t$.
To do so, we rely on the following inequality:
for all $z_1, z_2 \in \R$,
\begin{equation}
\label{eq:selfconc}
\int_{[0,1]} \dot{\sig} \bigl( z_1 + v (z_2 - z_1) \bigr) \d v \geq \frac{\dot{\sig} (z_1)}{1 + \arrowvert z_1 - z_2 \arrowvert}\,.
\end{equation}
This inequality is proved, in the case $z_1 \ne z_2$, by noting that the second-order derivative of $\eta$ equals
\[
\ddot{\eta}(x) = \frac{\e^{-x} -1}{1 + \e^x} \,\,\dot{\eta}(x)\,,
\qquad \mbox{where} \qquad \frac{\e^{-x} -1}{1 + \e^x} \in [-1,1]\,,
\]
so that for all $z,z' \in \R$,
\[
\ln \dot{\eta}(z') - \ln \dot{\eta}(z) = \int_{z}^{z'}
\frac{\ddot{\eta}(\tau)}{\dot{\eta}(\tau)} \d \tau \geq - |z'-z|\,, \qquad \mbox{i.e.}, \qquad
\dot{\eta}(z') \geq \dot{\eta}(z) \, \e^{- |z'-z|}\,.
\]
Therefore,
\[
\int_{[0,1]} \dot{\sig} \bigl( z_1 + v (z_2 - z_1) \bigr) \d v
\geq
\int_{[0,1]} \dot{\eta}(z_1) \, \e^{- v|z_1-z_2|} \d v
= \dot{\eta}(z_1) \, \frac{1 - \e^{- |z_1-z_2|}}{|z_1-z_2|} \geq \frac{\dot{\eta}(z_1)}{1+|z_1-z_2|}\,,
\]
where we applied the inequality $(1-\e^{-x})/x \geq 1/(1+x)$, which holds for all $x \geq 0$.

We go back to proving the third inequality of~\eqref{eq:Gtsucceq}.
With \eqref{eq:selfconc} for the first inequality, followed by an application of
the Cauchy-Schwarz inequality for the second inequality, and
the fact that $\btheta_{\star}, \, \hat{\btheta}_{t} \in \Theta$
have Euclidean norms smaller than~$\Arrowvert \Theta \Arrowvert_{\maxn}$,
together with $\Arrowvert \bphi\Arrowvert \leq 1$, we have,
for all $s \geq 1$ with $a_s \ne \anull$,
\begin{align*}
\int_{[0,1]} \dot{\sig} \bigl( v \, \bphi(a_s, \bx_s)^{\transp} \hat{\btheta}_{t} & +
(1-v) \bphi(a_s, \bx_s)^{\transp} \btheta_\star \bigr) \d v \\
& \geq \dot{\sig} \bigl( \bphi(a_s, \bx_s)^{\transp} \btheta_\star \bigr) \Bigl( 1 +  \bigl  \arrowvert \bphi(a_s, \bx_s)^{\transp}  (\btheta_\star - \hat{\btheta}_{t} )  \bigr \arrowvert \Bigr ) ^{\! -1} \\
& \geq \dot{\sig} \bigl( \bphi(a_s, \bx_s)^{\transp} \btheta_\star \bigr) \Bigl( 1 +  \bigl  \Arrowvert \bphi(a_s, \bx_s) \bigr \Arrowvert \;  \bigl  \Arrowvert \btheta_\star - \hat{\btheta}_{t} \bigr \Arrowvert \Bigr ) ^{\! -1} \\
& \geq \dot{\sig} \bigl( \bphi(a_s, \bx_s)^{\transp} \btheta_\star \bigr) \bigl(1 + 2 \Arrowvert \Theta \Arrowvert_{\maxn} \bigr)^{-1}\,.
\end{align*}
As $\bigl(1 + 2 \; \Arrowvert \Theta \Arrowvert_{\maxn} \bigr)^{-1} \leq 1$, we can then conclude,
from the definition of $G_t$, that
\[
\begin{split}
G_t & \succeq
(1 + 2 \; \Arrowvert \Theta \Arrowvert_{\maxn})^{-1}
\left( \lambda \id_m + \sum_{s=1}^{t}
\dot{\sig} \bigl( \bphi(a_s, \bx_s)^{\transp} \btheta_{\star} \bigr)
\, \bphi(a_s, \bx_s) \bphi(a_s, \bx_s)^{\transp} \Inotanull{s} \right) \\
& = (1 + 2 \; \Arrowvert \Theta \Arrowvert_{\maxn})^{-1} \,
W_t({\btheta}_\star)
\,,
\end{split}
\]
as announced. This concludes the proof.

\clearpage
\section{Proof of the Regret Bound in Case of an Unknown Distribution~$\nu$: \\ ~~~~~~~Proof of Theorem~\ref{th:main-unknown}}
\label{app:proof-nu-unknown}

We rather explain the differences to and modifications with respect to the proof of Appendix~\ref{app:proof-nu-known}.

To best do so, we use $\hat{~}$ superscripts to index various quantities defined based on the estimations $\hat{\nu}_t$,
even though these quantities are not themselves estimators. In particular, the budget parameter is denoted by $\hat{B}_T$;
we refer to the policies computed at rounds $t \geq 2$ by $\hat{\bp}_t(h_{t-1},\,\cdot\,)$,
and by $\hat{\cE}_\delta$ the event of Lemma~\ref{lemma:pbound_conversion}
for the sampling strategy pulling actions $a_t \sim \hat{\bp}_t(\bx_t)$,
which is exactly the strategy that we are analyzing here;
and finally, the optimal dual variables linked to the budget
are denoted by $\hat{\bbeta}^{\budg,\star}_t$.

\subsection{Key New Building Block: Uniform Deviation Inequality}

Throughout this appendix, we will need to relate quantities defined based on
$\hat{\nu}_t$ to the target quantities defined based on $\nu$. All these quantities will be
of the form: for $1 \leq t \leq T$ and for various functions $f : \cX \to [0,1]$,
\[
\E_{\bX \sim \hat{\nu}_t} \bigl[ f(\bX) \bigr]
\qquad \mbox{and} \qquad
\E_{\bX \sim \hat{\nu}_t} \bigl[ f(\bX) \bigr]\,.
\]
A simple (but probably slightly suboptimal) way to do so is to apply $T |\cX|$ times the
Hoeffding-Azuma inequality together with a union bound. We get that on an event $\cEnu$
of probability at least $1-\delta$,
\[
\forall \, 1 \leq t \leq T, \quad \forall \bx \in \cX, \qquad
\bigl| \hat{\nu}_t(\bx) - \nu(\bx) \bigr| \leq \sqrt{\frac{1}{2t} \ln \frac{2 T |\cX|}{\delta}}\,.
\]
In particular, with probability at least $1-\delta$,
for all functions $f : \cX \to [0,1]$ and all $1 \leq t \leq T$,
\begin{equation}
\label{eq:step0-HAz-f}
\Bigl| \E_{\bX \sim \hat{\nu}_t} \bigl[ f(\bX) \bigr]
- \E_{\bX \sim \hat{\nu}_t} \bigl[ f(\bX) \bigr] \Bigr|
\leq \sum_{x \in \cX} \bigl| \hat{\nu}_t(\bx) - \nu(\bx) \bigr|
\leq |\cX| \sqrt{\frac{1}{2t} \ln \frac{2 T |\cX|}{\delta}}\,.
\end{equation}

We now explain the adaptations required (or not required) for each of the four steps of the proof
provided in Appendix~\ref{app:proof-nu-known}.

\subsection{First and Second Steps: No Adaptation Required}
\label{app:recap-step2}

These two steps do not require any adaptation; we merely re-state the useful results
extracted therein, with the corresponding adapted notation.

The first step (Appendix~\ref{sec:ICP}) held for any sampling
strategy. Therefore,
the same upper-confidence bonuses~\eqref{eq:def:vareps}
and Lemma~\ref{lemma:pbound_conversion}
entail that on an event $\hat{\cE}_\delta$ of probability at least $1-\delta$,
for all $t \geq 1$, all $a \in \cA \setminus \{ \anull \}$, and all $\bx \in \cX$:
\begin{equation}
\label{eq:UP-unknown}
P(a, \bx) \leq U_{t}(a,\bx) \leq P(a, \bx) + 2\varepsilon_{t}(a,\bx)\,.
\end{equation}
The bound~\eqref{eq:defET} also still holds, as it was obtained in a deterministic manner
not using any specific feature of the sampling strategy; namely,
\begin{equation}
\label{eq:UP-unknown-suite}
2 \sum_{t=2}^T \varepsilon_{t-1}(a_t,\bx_t) \Inotanull{t} \leq E_T\,.
\end{equation}

Similarly, the second step (Appendix~\ref{app:second-step})
actually yielded general results between the primal and the dual formulations of the $\opt$
problems considered.
The equality~\eqref{eq:kkt-complementary-slackness-budget}, the characterizations~\eqref{eq:charact-pt-x-1ter}
and~\eqref{eq:charact-pt-x-1bis}, as well as the inequality~\eqref{eq:property-strong-duality}
may be instantiated with $\hat{\nu}_t$ (in lieu of $\nu$) and $\hat{B}_{T}$ (in lieu of $B_T$) as follows.
For each $t \geq 2$ such that
the cost constraints of Phase~0 of the adaptive policy are not violated
and the optimization problem $\opt\bigl(\hat{\nu}_t,U_{t-1},\hat{B}_{T}\bigr)$ is to be solved,
there exists a vector $\hat{\bbeta}^{\budg,\star}_t \geq \mathbf{0}$ such that
first,
\begin{equation}
\label{eq:kkt-complementary-slackness-budget-unknown-nu}
\bigl( \hat{\bbeta}^{\budg,\star}_t \bigr)^{\, \transp}
\E_{\bX \sim \hat{\nu}_t} \! \left[ \sum_{a \in \cA} \bc(a, \bX) \, U_{t-1}(a,\bX) \, \hat{p}_{t, a} (\bX) \right]
=  \frac{\hat{B}_{T}}{T} \, \bigl( \hat{\bbeta}^{\budg,\star}_t \bigr)^{\, \transp} \mathbf{1}\,.
\end{equation}
Second, for all $\bx \in \cX$,
\begin{multline}
\label{eq:charact-pt-x-1ter-unknown-nu}
\max_{\bq \in \cP(\cA)}
\sum_{a \in \cA} \Bigl( r(a,\bx) - \bigl(\hat{\bbeta}^{\budg,\star}_t\bigr)^{\, \transp} \bc(a, \bx) \Bigr) \, U_{t-1}(a,\bx) \, q_a(\bx) \\
= \max_{\bq \in \cP(\cA)}
\sum_{a \in \cA} \Bigl( r(a,\bx) - \bigl(\hat{\bbeta}^{\budg,\star}_t\bigr)^{\, \transp} \bc(a, \bx) \Bigr)_+ \, U_{t-1}(a,\bx) \, q_a(\bx)
\end{multline}
and
\begin{equation}
\label{eq:charact-pt-x-1bis-unknown-nu}
\hat{\bp}_t(h_{t-1},\bx) \in \argmax_{\bq \in \cP(\cA)}
\sum_{a \in \cA} \Bigl( r(a,\bx) - \bigl(\hat{\bbeta}^{\budg,\star}_t\bigr)^{\, \transp} \bc(a, \bx) \Bigr)_+ \, U_{t-1}(a,\bx) \, q_{a} \,.
\end{equation}
Third,
\begin{equation}
\label{eq:property-strong-duality-unknown-nu}
\opt\bigl(\hat{\nu}_t,U_{t-1},\hat{B}_{T}\bigr) \geq \hat{B}_{T} \bigl( \hat{\bbeta}^{\budg,\star}_t \bigr)^{\, \transp} \mathbf{1}\,.
\end{equation}

\subsection{Third Step: Most of the Adaptations Lie Here}

We show below that on the intersection between the events $\hat{\cE}_{\prob,\delta}$
and $\cEnu$, introduced above, and another event
$\hat{\cE}_{\haz,\delta}$, also of probability at least $1-\delta$, we have that
\[
\sum_{t=1}^T \bc(a_t,\bx_t) \, y_t \leq (B-1) \, \mathbf{1} \,.
\]
Therefore, on this intersection of three events,
for all rounds $t \geq 2$, the policy $\hat{\bp}_t(h_{t-1},\,\cdot\,)$
is obtained by Phase~2, i.e., by solving $\opt\bigl(\nu_t,U_{t-1},B_T\bigr)$.

Similarly, we then prove below that on $\hat{\cE}_{\prob,\delta}
\cap \hat{\cE}_{\haz,\delta} \cap \cEnu$,
\begin{multline}
\label{eq:ccl-third-unknown}
\sum_{t=1}^T r(a_t,\bx_t) \, y_t \geq
\sum_{t=2}^T \overbrace{\E_{\bX \sim \hat{\nu}_t} \!\left[ \sum_{a \in \cA} r(a, \bX) \, U_{t-1}(a, \bX) \, \hat{p}_{t,a}(h_{t-1},\bX)
 \right]}^{= \opt(\hat{\nu}_t,U_{t-1},\hat{B}_T)/T} \\
- \left( E_T + \sqrt{2T \ln \frac{4d}{\delta}} + |\cX| \sqrt{2T \ln \frac{2 T |\cX|}{\delta}} \right).
\end{multline}

\paragraph{The inequalities that may be extracted from Appendix~\ref{app:vhpc}.}
Several applications of the Hoeffding-Azuma inequality, together with a union bound,
show that there exist an event $\hat{\cE}_{\haz,\delta}$ of probability at least $1-\delta$
such that, simultaneously:
\begin{align}
\label{eq:HAz1-unknown}
\sum_{t=1}^T r(a_t,\bx_t) \, y_t & \geq \sum_{t=1}^T r(a_t, \bx_t) \, P(a_t, \bx_t) -
\sqrt{\frac{T}{2} \ln \frac{4}{\delta}}\,, \\
\label{eq:HAz2-unknown}
\sum_{t=1}^T \bc(a_t,\bx_t) \, y_t & \leq \sum_{t=1}^T  \bc(a_t, \bx_t) \, P(a_t, \bx_t)
+ \sqrt{\frac{T}{2} \ln \frac{4d}{\delta}} \, \mathbf{1}\,, \\
\label{eq:HAz3-unknown}
\sum_{t=2}^T r(a_t,\bx_t) \, U_{t-1}(a_t, \bx_t) & \geq \sum_{t=2}^T
\E_{\bX \sim \nu} \!\left[ \sum_{a \in \cA} r(a, \bX) \, U_{t-1}(a, \bX) \, \hat{p}_{t,a}(h_{t-1},\bX) \right] -
\sqrt{\frac{T}{2} \ln \frac{4}{\delta}}\,, \\
\label{eq:HAz4-unknown}
\sum_{t=2}^T \bc(a_t,\bx_t) \, U_{t-1}(a_t, \bx_t) & \leq \sum_{t=2}^T
\E_{\bX \sim \nu} \!\left[ \sum_{a \in \cA} \bc(a, \bX) \, U_{t-1}(a, \bX) \, \hat{p}_{t,a}(h_{t-1},\bX) \right] +
\sqrt{\frac{T}{2} \ln \frac{4d}{\delta}} \, \mathbf{1}\,.
\end{align}

\paragraph{Dealing with the cost constraints.}
Now, for each $t \geq 2$, the definition of $\hat{\bp}_t(h_{t-1},\,\cdot\,)$, no matter whether it is provided by Phase~0 or
Phase~2 of the adaptive policy, ensures that
\[
\E_{\bX \sim \hat{\nu}_t} \!\left[ \sum_{a \in \cA} \bc(a, \bX) \, U_{t-1}(a, \bX) \, \hat{p}_{t,a}(h_{t-1},\bX) \right]
\leq \frac{\hat{B}_T}{T} \, \mathbf{1} \,.
\]
By~\eqref{eq:step0-HAz-f} and the fact that the sums at stake below lie in $[0,1]$, we have, on $\cEnu$,
that for all $t \geq 2$,
\begin{multline*}
\E_{\bX \sim {\nu}} \!\left[ \sum_{a \in \cA} \bc(a, \bX) \, U_{t-1}(a, \bX) \, \hat{p}_{t,a}(h_{t-1},\bX) \right] \\
\leq
\E_{\bX \sim \hat{\nu}_t} \!\left[ \sum_{a \in \cA} \bc(a, \bX) \, U_{t-1}(a, \bX) \, \hat{p}_{t,a}(h_{t-1},\bX) \right]
+ |\cX| \sqrt{\frac{1}{2t} \ln \frac{2 T |\cX|}{\delta}} \, \mathbf{1} \,.
\end{multline*}
Combining the two inequalities above with~\eqref{eq:HAz2-unknown} and~\eqref{eq:HAz4-unknown},
as well as with the bounds $P \leq U_{t-1}$ of~\eqref{eq:UP-unknown},
we proved so far that
\[
\mbox{on} \ \hat{\cE}_{\prob,\delta} \cap \hat{\cE}_{\haz,\delta} \cap \cEnu\,, \quad
\sum_{t=1}^T \bc(a_t,\bx_t) \, y_t
\leq \left( \hat{B}_T + 1 + \sqrt{2T \ln \frac{4d}{\delta}}
+ \sum_{t=2}^T |\cX| \sqrt{\frac{1}{2t} \ln \frac{2 T |\cX|}{\delta}} \right) \mathbf{1}\,.
\]
Since $\displaystyle{\sum_{t\leq T}} 1/\sqrt{t} \leq 2 \sqrt{T}$ and by the definition of $\hat{B}_T$
(see the statement of Theorem~\ref{th:main-unknown}), we proved that
\[
\mbox{on} \ \hat{\cE}_{\prob,\delta} \cap \hat{\cE}_{\haz,\delta} \cap \cEnu\,, \qquad
\sum_{t=1}^T \bc(a_t,\bx_t) \, y_t \leq (B-1) \mathbf{1}\,,
\]
as claimed.
Therefore, on $\hat{\cE}_{\prob,\delta} \cap \hat{\cE}_{\haz,\delta} \cap \cEnu$,
the adaptive policy of Box~B of Section~\ref{sec:policy}
never stays in Phase~0 and instead solves, at each round $t \geq 2$, the Phase-2 problem
$\opt\bigl(\hat{\nu}_t,U_{t-1},\hat{B}_T\bigr)$.

\paragraph{Dealing with the rewards.}
By~\eqref{eq:HAz1-unknown} and~\eqref{eq:HAz3-unknown}, by
the bound $U_{t-1} \leq P + 2\epsilon_{t-1}$ of~\eqref{eq:UP-unknown}
together with the bound $E_T$ of~\eqref{eq:UP-unknown-suite},
and by the uniform control~\eqref{eq:step0-HAz-f},
we have similarly that on $\hat{\cE}_{\prob,\delta} \cap \hat{\cE}_{\haz,\delta} \cap \cEnu$,
\begin{multline*}
\sum_{t=1}^T r(a_t,\bx_t) \, y_t
\geq \sum_{t=2}^T \E_{\bX \sim \hat{\nu}_t} \!\left[ \sum_{a \in \cA} r(a, \bX) \, U_{t-1}(a, \bX) \, \hat{p}_{t,a}(h_{t-1},\bX) \right] \\
- \left( E_T + \sqrt{2T \ln \frac{4d}{\delta}} + |\cX| \sqrt{2T \ln \frac{2 T |\cX|}{\delta}} \right).
\end{multline*}
Given that on the intersection of events considered,
the adaptive policy solves $\opt\bigl(\hat{\nu}_t,U_{t-1},\hat{B}_T\bigr)$ for all $t \geq 2$, we have
\begin{equation}
\label{eq:rewritingrewardsnuunknown}
\E_{\bX \sim \hat{\nu}_t} \!\left[ \sum_{a \in \cA} r(a, \bX) \, U_{t-1}(a, \bX) \, \hat{p}_{t,a}(h_{t-1},\bX) \right]
= \frac{\opt\bigl(\hat{\nu}_t,U_{t-1},\hat{B}_T\bigr)}{T}\,.
\end{equation}
This concludes the proof of~\eqref{eq:ccl-third-unknown}, and hence, the adaptation of the third step.

\subsection{Fourth Step: Some Adaptations are Also Required}

In this final step, we collect the bounds from the previous three steps.
Some (rather minor) adaptations are required, e.g., we would like
to integrate~\eqref{eq:charact-pt-x-1ter-unknown-nu} and~\eqref{eq:charact-pt-x-1bis-unknown-nu}
over $\hat{\nu}_t$ in the left-hand sides and $\nu$ in the right-hand sides.

\paragraph{Main modification.}
We consider some $t \geq 2$.
For each $\bx \in \cX$, we apply~\eqref{eq:charact-pt-x-1ter-unknown-nu} and~\eqref{eq:charact-pt-x-1bis-unknown-nu} with
$\bq = \pi^\star(\bx)$ and get
\begin{multline*}
\sum_{a \in \cA} \Bigl( r(a,\bx) - \bigl(\hat{\bbeta}^{\budg,\star}_t\bigr)^{\, \transp} \bc(a, \bx) \Bigr) \, U_{t-1}(a,\bx) \, \hat{p}_{t,a}(h_{t-1},\bx) \\
\geq
\sum_{a \in \cA} \Bigl( r(a,\bx) - \bigl(\hat{\bbeta}^{\budg,\star}_t\bigr)^{\, \transp} \bc(a, \bx) \Bigr)_+ \, U_{t-1}(a,\bx) \, \pi^\star_a(\bx)\,.
\end{multline*}
The bound $U_{t-1} \geq P$ of~\eqref{eq:UP-unknown} and the non-negative parts taken in the right-hand side then entail that
\begin{multline*}
\mbox{on} \ \hat{\cE}_{\prob,\delta}, \qquad \forall \bx \in \cX,
\qquad
\sum_{a \in \cA} \Bigl( r(a,\bx) - \bigl(\hat{\bbeta}^{\budg,\star}_t\bigr)^{\, \transp} \bc(a, \bx) \Bigr) \, U_{t-1}(a,\bx) \, \hat{p}_{t,a}(h_{t-1},\bx) \\
\geq
\sum_{a \in \cA} \Bigl( r(a,\bx) - \bigl(\hat{\bbeta}^{\budg,\star}_t\bigr)^{\, \transp} \bc(a, \bx) \Bigr)_+ \, P(a,\bx) \, \pi^\star_a(\bx)\,.
\end{multline*}
We replace the individual $\bx$ by a random variable $\bX \sim \hat{\nu}_t$ and take expectations with respect
to~$\bX$:
\begin{multline*}
\mbox{on} \ \hat{\cE}_{\prob,\delta}, \qquad
\E_{\bX \sim \hat{\nu}_t} \! \left[ \sum_{a \in \cA} \Bigl( r(a,\bX) -
\bigl(\hat{\bbeta}^{\budg,\star}_t\bigr)^{\, \transp} \bc(a, \bX) \Bigr) \, U_{t-1}(a,\bX) \, \hat{p}_{t,a}(h_{t-1},\bX) \right] \\
\geq \E_{\bX \sim \hat{\nu}_t} \! \left[
\sum_{a \in \cA} \Bigl( r(a,\bX) - \bigl(\hat{\bbeta}^{\budg,\star}_t\bigr)^{\, \transp} \bc(a, \bX) \Bigr)_+ \, P(a,\bX) \, \pi^\star_a(\bX)
\right].
\end{multline*}
Thanks to the non-negative parts in the right-hand side, we identify some function $f(\bX)$ where $f$ takes values in $[0,1]$,
so that we may apply the uniform control~\eqref{eq:step0-HAz-f} and get
\begin{align*}
\mbox{on} \ \cEnu, \qquad
& \E_{\bX \sim \hat{\nu}_t} \! \left[
\sum_{a \in \cA} \Bigl( r(a,\bX) - \bigl(\hat{\bbeta}^{\budg,\star}_t\bigr)^{\, \transp} \bc(a, \bX) \Bigr)_+ \, P(a,\bX) \, \pi^\star_a(\bX)
\right] \\
\geq \ & \E_{\bX \sim \nu} \! \left[
\sum_{a \in \cA} \Bigl( r(a,\bX) - \bigl(\hat{\bbeta}^{\budg,\star}_t\bigr)^{\, \transp} \bc(a, \bX) \Bigr)_+ \, P(a,\bX) \, \pi^\star_a(\bX)
\right] - |\cX| \sqrt{\frac{1}{2t} \ln \frac{2 T |\cX|}{\delta}} \\
\geq \ & \E_{\bX \sim \nu} \! \left[
\sum_{a \in \cA} \Bigl( r(a,\bX) - \bigl(\hat{\bbeta}^{\budg,\star}_t\bigr)^{\, \transp} \bc(a, \bX) \Bigr) \, P(a,\bX) \, \pi^\star_a(\bX)
\right] - |\cX| \sqrt{\frac{1}{2t} \ln \frac{2 T |\cX|}{\delta}}\,,
\end{align*}
where for the second inequality, we simply dropped the non-negative parts.

\paragraph{The rest of the proof for this final step is basically unchanged.}
Combining all the bounds exhibited so far in this updated fourth step, we have, for each $t \geq 2$,
\begin{align*}
\mbox{on} \ \hat{\cE}_{\prob,\delta} \cap \cEnu, \qquad
& \overbrace{\E_{\bX \sim \hat{\nu}_t} \left[ \sum_{a \in \cA} r(a,\bX) \, U_{t-1}(a,\bX) \, \hat{p}_{t, a}(h_{t-1},\bX)
\right]}^{=\opt(\hat{\nu}_t,U_{t-1},\hat{B}_T)/T} \\
\nonumber
& \qquad - \bigl(\hat{\bbeta}^{\budg,\star}_t\bigr)^{\, \transp}
\E_{\bX \sim \hat{\nu}_t} \left[ \sum_{a \in \cA}  \bc(a, \bX) \, U_{t-1}(a,\bX) \, \hat{p}_{t, a}(h_{t-1},\bX) \right] \\
\nonumber
\geq \ \ & \overbrace{\E_{\bX \sim \nu} \left[  \sum_{a \in \cA} r(a,\bX) \, P(a,\bX) \, \pi^{\star}_a(\bX) \right]}^{=\opt(\nu,P,B)/T} \\
& \qquad - \bigl( \hat{\bbeta}^{\budg,\star}_t \bigr)^{\, \transp}
\underbrace{\E_{\bX \sim \nu} \left[ \sum_{a \in \cA}  \bc(a, \bX) \, P(a,\bX) \, \pi^{\star}_a(\bX) \right]}_{\leq (B/T) \mathbf{1}}
- |\cX| \sqrt{\frac{1}{2t} \ln \frac{2 T |\cX|}{\delta}}\,,
\end{align*}
where we substituted the inequalities stemming from the definition of $\pi^\star$
as well as the rewriting~\eqref{eq:rewritingrewardsnuunknown}.
Rearranging the inequality above and substituting~\eqref{eq:kkt-complementary-slackness-budget-unknown-nu},
we get that on $\hat{\cE}_{\prob,\delta} \cap \cEnu$,
\begin{align}
\nonumber
\lefteqn{\frac{\opt(\nu,P,B)}{T} - \frac{\opt\bigl(\hat{\nu}_t,U_{t-1},\hat{B}_T\bigr)}{T}} \\
\nonumber
& \leq \frac{B \bigl(\hat{\bbeta}^{\budg,\star}_t\bigr)^{\transp} \mathbf{1}}{T}
- \bigl(\hat{\bbeta}^{\budg,\star}_t\bigr)^{\transp}\E_{\bX \sim \hat{\nu}_t}
\left[ \sum_{a \in \cA} \bc(a, \bX) \, U_{t-1}(a,\bX) \, \hat{p}_{t, a}(\bX)\right] +
|\cX| \sqrt{\frac{1}{2t} \ln \frac{2 T |\cX|}{\delta}} \\
\label{eq:Step4ingr2-unknown-nu}
& = \frac{B-\hat{B}_{T}}{T} \, \bigl(\hat{\bbeta}^{\budg,\star}_t\bigr)^{\transp} \mathbf{1} + |\cX| \sqrt{\frac{1}{2t} \ln \frac{2 T |\cX|}{\delta}} \,.
\end{align}
Summing this bound over $2 \leq t \leq T$ and combining it with~\eqref{eq:property-strong-duality-unknown-nu}, we obtain
that on $\hat{\cE}_{\prob,\delta} \cap \cEnu$,
\[
\opt(\nu,P,B) - \sum_{t=2}^T \frac{\opt\bigl(\hat{\nu}_t,U_{t-1},\hat{B}_T\bigr)}{T}
\leq 1 + \frac{B-\hat{B}_{T}}{\hat{B}_{T}} \sum_{t=2}^T \frac{\opt\bigl(\hat{\nu}_t,U_{t-1},\hat{B}_T\bigr)}{T}
+ |\cX| \sqrt{2T \ln \frac{2 T |\cX|}{\delta}} \,.
\]
By distinguishing the cases
\[
\opt(\nu,P,B) - \sum_{t=2}^T \frac{\opt\bigl(\hat{\nu}_t,U_{t-1},\hat{B}_T\bigr)}{T} \leq 0 \quad \ \mbox{and} \ \quad
\opt(\nu,P,B) - \sum_{t=2}^T \frac{\opt\bigl(\hat{\nu}_t,U_{t-1},\hat{B}_T\bigr)}{T} \geq 0\,,
\]
the inequality above entails that on $\hat{\cE}_{\prob,\delta} \cap \cEnu$,
\[
\opt(\nu,P,B) - \sum_{t=2}^T \frac{\opt\bigl(\hat{\nu}_t,U_{t-1},\hat{B}_T\bigr)}{T}
\leq \frac{B-\hat{B}_{T}}{\hat{B}_{T}} \opt(\nu,P,B)
+ |\cX| \sqrt{2T \ln \frac{2 T |\cX|}{\delta}} + 1 \,.
\]
Substituting this upper bound into~\eqref{eq:ccl-third-unknown}, we finally
obtain that on $\hat{\cE}_{\prob,\delta} \cap \hat{\cE}_{\haz,\delta} \cap \cEnu$,
\begin{align*}
\opt(\nu,P,B) - \sum_{t=1}^T r(a_t,\bx_t) \, y_t
\leq \frac{B-\hat{B}_{T}}{\hat{B}_{T}} & \opt(\nu,P,B) \\
& + E_T + \underbrace{\sqrt{2T \ln \frac{4d}{\delta}} + 2 |\cX| \sqrt{2T \ln \frac{2 T |\cX|}{\delta}} + 1}_{\leq 2 b_T}\,. \\
\end{align*}
We conclude the proof by the same modifications to improve readability
as at the end of the proof of Theorem~\ref{th:main}: namely, since the definition of $E_T$
did not change, the bound~\eqref{eq:bdETpart3} is still applicable,
while
\[
\frac{B-\hat{B}_{T}}{\hat{B}_{T}} \leq
\frac{2 b_T}{B} = \frac{1}{B} \left(
4 + 2\sqrt{2T \ln \frac{4d}{\delta}}
+ 2 |\cX| \sqrt{2T \ln \frac{2 T |\cX|}{\delta}} \right)
\]
is obtained with the same techniques and similar conditions as
for~\eqref{eq:BBT}.

\clearpage
\section{Details for the Material of Section~\ref{sec:LinearCBwL}}
\label{app:linCBwK}

In this section, we first recall (Appendix~\ref{app:setting-lin})
the setting of linear CBwK introduced by \citet{Agrawal2016LinearCB}---and actually,
slightly generalize it to match the setting of CBwK for a logistic-regression conversion model.
We then state the adaptive policy considered (Appendix~\ref{app:policy-lin}),
which relies on upper-confidence estimates of the rewards and lower-confidence estimates of the cost
vectors. We also state the corresponding estimation guarantees in a key lemma (Lemma~\ref{lemma:linucb}).
The heart of this section is to
state, discuss (Appendix~\ref{sec:th-lin}) and prove (Appendix~\ref{sec:proof-sketch-lin})
a regret bound, matching the one of \citet[Theorem~3]{Agrawal2016LinearCB}, with a slight improvement
consisting of a relaxation of the budget constraints.
For the sake of completeness, we finally recall (Appendix~\ref{sec:algo-AD}) the statement of the adaptive policy of
\citet{Agrawal2016LinearCB}.

\subsection{Setting}
\label{app:setting-lin}

The setting is the following. We consider a
finite action set $\cA$ including a no-op action $\anull$, a finite context set $\cX \subseteq \R^{n}$,
a number $T$ of rounds and a total budget constraint $B > 0$. All these parameters are known.
Contexts $\bx_1,\bx_2,\ldots,\bx_T$ are drawn i.i.d.\ according to some distribution $\nu$.
At round $t \geq 1$, the learner observes the context $\bx_t$, picks an action $a_t$ and,
conditionally to $\bx_t$ and $a_t$, when $a_t \ne \anull$,
obtains a reward $r_t \in [0,1]$ drawn independently at random according to a distribution
with expectation $\olr(a_t,\bx_t)$, where
\[
\forall a \in \cA \setminus \{\anull\}, \ \forall \bx \in \cX, \qquad
\olr(a,\bx) = \bphi(a,\bx)^{\transp} \bmu_{\star}\,,
\]
and suffers a vector cost $\bc_t \in [0,1]^d$
drawn independently at random according to a distribution
with vector of expectations
$\olbc(a_t,\bx_t)$, where each component $\olc_i$ of $\olbc$, for $i \in \{1,\ldots,d\}$, is given by
\[
\forall a \in \cA \setminus \{\anull\}, \ \forall \bx \in \cX, \qquad
\olc_i(a,\bx) = \bphi(a,\bx)^{\transp} \btheta_{\star,i} \,.
\]
In the definitions above, $\bphi : \cA \setminus \{ \anull \} \times \cX \to \R^m$ is a known transfer function,
with $\Arrowvert \bphi \Arrowvert \leq 1$,
and $\bmu_{\star}$ and the $\btheta_{\star,i}$ are unknown parameters in $\R^m$.
We assume that these unknown parameters lie in some bounded set $\Theta$, with
maximal norm still denoted by $\Arrowvert \Theta \Arrowvert$.
When $a_t = \anull$, the obtained reward and suffered costs are null:
$r_t = 0$ and $\bc_t = \mathbf{0}$.

\paragraph{Comparison to the canonical setting of linear CBwK.}
Note that in the original formulation of \citet{Agrawal2016LinearCB}, we have (where $\bx_a$ also denote vectors):
\[
\bx = \bigl( \bx_a \bigr)_{a \in \cA \setminus \{\anull\}} \qquad \mbox{and} \qquad \bphi(a,\bx) = \bx_a\,.
\]

\paragraph{Benchmark and regret.}
The goal is still to maximize the accumulated rewards while controlling the costs:
\[
\mbox{maximize} \quad \sum_{t \leq T} r_t \qquad \mbox{while controlling} \quad
\sum_{t \leq T} \bc_t \leq B \mathbf{1}\,.
\]
The goal can be equivalently defined as the minimization of the regret
while controlling the costs, where the regret equals
\[
R_T = \opt(\nu,\olr,\olbc,B) - \sum_{t \leq T} r_t
\]
for the benchmark given by the static policy $\pi^\star$ achieving
the largest expected
cumulative rewards under the condition that its cumulative vector
costs abide by the budget constraints in expectation, i.e.,
\begin{equation}
\label{eq:opt-def-linear}
\begin{split}
\opt(\nu,\olr,\olbc,B) =
&  \max_{\pi : \cX \to \cP(\cA)}
T \, \E_{\bX \sim \nu} \! \left[ \sum_{a \in \cA} \olr(a, \bX)\right] \\
&  \text{under} \qquad T \, \E_{\bX \sim \nu} \! \left[ \sum_{a \in \cA} \olbc(a, \bx) \right]
\leq B \mathbf{1} \,.
\end{split}
\end{equation}
In the sequel, we will use the definition~\eqref{eq:opt-def-linear} of $\opt$
with different quadruplets of parameters; see, for instance, the definition of
the adaptive policy of Box~C.

\subsection{Statement of an Adaptive Policy}
\label{app:policy-lin}

The considered adaptive policy is stated in Box~C. It is adapted
from the adaptive policy of Section~\ref{sec:policy}. The (almost only)
changes lie in Phase~1, which depends heavily of the model.
Here, we resort (as \citet{Agrawal2016LinearCB}) to
a LinUCB-type estimation of the parameters of the stochastic linear bandits
yielding rewards and costs. Based on these estimated parameters,
we are able to issue, at each round $t \geq 2$, an
upper-confidence expected reward function $U_{t-1}$
and a lower-confidence expected vector-cost function $\bL_{t-1}$.
We also use the empirical estimate of the context distribution.
In Phase~2, we solve the $\opt$ problem on these
estimates and with the conservative budget $\hat{B}_T$.

The adaptive policy of Box~C bears links, and actually generalizes,
the one by \citet{Zhen2021ReoptimCBKC}. The setting of the latter reference
is more limited as more information is provided to the learner, such as
the costs for taking each action and the distribution $\nu$ of contexts (clients in their
case).

For the sake of completeness, we state in Appendix~\ref{sec:algo-AD}
the adaptive policy introduced by~\citet{Agrawal2016LinearCB}.

\begin{figure}[p]
\begin{nbox}[title={Box C: LinUCB for direct solutions to OPT problems}]
\textbf{Parameters:} regularization parameter $\lambda > 0$;
conservative-budget parameter $\hat{B}_T$;
upper-confidence bonuses $\varepsilon_{s}(a,\bx) > 0$, for $s \geq 1$
and $(a,\bx) \in \bigl( \cA \setminus \{\anull\} \bigr) \times \cX$. \medskip

\textbf{Round} $t=1$: play an arbitrary action $a_1 \in \cA \setminus \{\anull\}$  \medskip

\textbf{At rounds} $t \geq 2$: \smallskip
\begin{enumerate}
\item[\underline{Phase 0}] If $\displaystyle{\sum_{s \leq t-1} \bc_s \leq (B-1) \mathbf{1}}$ is violated, then
$\hat{\bp}_t(h_{t-1},\bx) = \delta_{\anull}$ for all $\bx$ \smallskip
\item[\underline{Phase 1}] Otherwise, estimate the parameters by
\begin{align*}
& \bmu_{t-1} = X_{t-1}^{-1} \sum_{s=1}^{t-1} \Inotanull{s} \bphi(a_s, \bx_s) \, r_s \\
\mbox{and} \qquad
& \hat{\btheta}_{t-1,i} = X_{t-1}^{-1} \sum_{s=1}^{t-1} \Inotanull{s} \bphi(a_s, \bx_s) \, c_{s,i} \\
\mbox{where} \qquad
& X_{t} = \sum_{s=1}^{t} \Inotanull{s} \bphi(a_s, \bx_s) \bphi(a_s, \bx_s)^{\transp} + \lambda \id_m
\end{align*}
Define the expected reward function $\hat{r}$ and cost function $\hat{\bc}_{t-1} = \bigl(\hat{c}_{t-1,i}\bigr)_{1 \leq i \leq d}$ as
\begin{align*}
\forall a \in \cA \setminus \{\anull\}, \ \ \forall \bx \in \cX, \qquad\qquad
& \hat{r}_{t-1}(a,\bx) = \bphi(a,\bx)^{\transp} \bmu_{t-1} \\
\mbox{and} \quad \forall 1 \leq i \leq d, \quad\qquad
& \hat{c}_{t-1,i} = \bphi(a,\bx)^{\transp} \btheta_{t-1,i}
\end{align*}
Build the upper-confidence expected reward function $U_{t-1}$
and the lower-confidence expected vector-cost function $\bL_{t-1}$ as
\begin{align*}
\forall a \in \cA & \setminus \{\anull\}, \ \ \forall \bx \in \cX, \\
& U_{t-1}(a,\bx) = \max \Bigl\{ \min \bigl\{ \hat{r}_{t-1}(a,\bx) + \epsilon_{t-1}(a,\bx), \, 1 \bigr\}, \, 0 \Bigr\} \\
& \bL_{t-1}(a,\bx) = \max \Bigl\{ \min \bigl\{ \hat{\bc}_{t-1} - \epsilon_{t-1}(a,\bx) \mathbf{1}, \, \mathbf{1} \bigr\}, \, \mathbf{0} \Bigr\}
\end{align*}
where the maximum and minimum are taken pointwise in the definition of $\bL_{t-1}$  \smallskip \\
Set $U_{t-1}(\anull,\bx) = 0$ and $L_{t-1}(\anull,\bx) = \mathbf{0}$ for all $\bx \in \cX$ \smallskip \\
Also estimate the context distribution by \quad $\displaystyle{\hat{\nu}_t = \frac{1}{t} \sum_{s=1}^t \delta_{\bx_s}}$
\item[\underline{Phase 2}] Compute the solution $\hat{\bp}_t(h_{t-1},\,\cdot\,)$ of
\begin{align*}
\hspace{-.75cm}
\opt\bigl(\hat{\nu}_t, U_{t-1}, \bL_{t-1}, B_T\bigr) =
&  \max_{\pi : \cX \to \cP(\cA)}
T \, \E_{\bX \sim \hat{\nu}_t} \! \left[ \sum_{a \in \cA} U_{t-1}(a,\bX) \, \pi_a(\bX) \right] \\
&  \text{under} \qquad T \, \E_{\bX \sim \hat{\nu}_t} \! \left[ \sum_{a \in \cA} \bL_{t-1}(a, \bX) \, \pi_a(\bX) \right]
\leq B_T \mathbf{1}
\end{align*}
Draw an arm $a_t \sim \hat{\bp}_t(h_{t-1},\bx_t)$
\end{enumerate}
\end{nbox}
\end{figure}

The main additional ingredient in the analysis of the policy of Box~C, compared to the
analyses of the adaptive policy of Section~\ref{sec:policy}, is a guarantee on the outcomes
of Phase~1. We recall that we assumed that the parameters $\bmu_\star$ and $\btheta_{\star,i}$
lie in a bounded set $\Theta$ with maximal norm denoted by $\Arrowvert \Theta \Arrowvert$.

\begin{lemma}[{direct adaptation from~\citet[Theorem~2]{AbbasiYadkori2011ImprovedAF}}]
\label{lemma:linucb}
Fix any sampling strategy and consider the version of LinUCB given by
Phase~1 of Box~C.
For all $\delta \in (0,1)$, there exists an event $\cE_{\lin,\delta}$ with probability at least $1-\delta$
and such that over $\cE_{\lin,\delta}$:
\begin{align*}
\forall t \geq 1, \ \forall a \in \cA \setminus \{ \anull \}, \ \forall \bx \in \cX, \qquad
& \bigl| \hat{r}_{t}(a,\bx) - \olr(a,\bx) \bigr| \leq
\gamma_{t,\lambda,\delta} \, \bigl\Arrowvert \bphi(a, \bx) \bigr\Arrowvert_{X_{t}^{-1}} \\
\mbox{and} \qquad
& \bigl| \hat{\bc}_{t}(a,\bx) - \olbc(a,\bx) \bigr| \leq
\gamma_{t,\lambda,\delta} \, \bigl\Arrowvert \bphi(a, \bx) \bigr\Arrowvert_{X_{t}^{-1}} \, \mathbf{1} \,,
\end{align*}
where
\[
\gamma_{t,\lambda,\delta} = \frac{1}{4} \sqrt{m \ln \frac{1 + t/(\lambda m)}{\delta/(d+1)}} + \sqrt{\lambda} \Arrowvert \Theta \Arrowvert.
\]
\end{lemma}

\begin{proof}[\emph{\textbf{Proof sketch.}}~~]
We explain why the bound for $\olr$ holds with probability at least $1-\delta/(d+1)$.
The lemma follows by repeating the argument for the components of $\olbc$ and resorting to a union bound.

Given that rewards lie in $[0,1]$ and are thus $1/4$--sub-Gaussian, the martingale analysis by
\citet[Theorem~2]{AbbasiYadkori2011ImprovedAF}, with the same adaptations as the ones carried out
in Appendix~\ref{app:Logistic-UCB1} to take into account the rounds when $a_t = \anull$, shows that
with probability at least $1-\delta/(d+1)$,
\[
\forall t \geq 1, \qquad \bigl\Arrowvert \bmu_t - \bmu_\star \bigr\Arrowvert_{X_t} \leq
\frac{1}{4} \sqrt{m \ln \frac{1 + t/(\lambda m)}{\delta/(d+1)}} + \sqrt{\lambda} \Arrowvert \Theta \Arrowvert = \gamma_{t,\lambda,\delta}\,.
\]
We then proceed similarly to the Cauchy-Schwarz inequalities following~\eqref{eq:PsieqGt}:
for all $a \in \cA \setminus \{ \anull \}$ and $\bx \in \cX$,
\[
\bigl| \hat{r}_{t}(a,\bx) - \olr(a,\bx) \bigr|
= \Bigl| \bphi(a,\bx)^{\transp} \bigl( \bmu_{t-1} - \bmu_\star \bigr) \Bigr|
\leq \bigl\Arrowvert \bphi(a, \bx) \bigr\Arrowvert_{X_{t}^{-1}} \,
\bigl\Arrowvert \bmu_t - \bmu_\star \bigr\Arrowvert_{X_t} \,.
\]
This concludes the proof.
\end{proof}

As a consequence, we set, when defining the adaptive policy of Box~C,
\[
\epsilon_s(a,\bx) = \gamma_{t,\lambda,\delta} \, \bigl\Arrowvert \bphi(a, \bx) \bigr\Arrowvert_{X_{s}^{-1}}
\]
for all $s \geq 1$,
and denote by $\hat{\cE}_{\lin,\delta}$ the event of Lemma~\ref{lemma:linucb}
for the sampling policy of Box~C. This event is of probability at least $1-\delta$.
We have:
\begin{align}
\nonumber
\mbox{on} \ \hat{\cE}_{\lin,\delta}, \qquad
\forall t \geq 1, \ \ \forall a \in \cA \setminus \{ \anull \}, \ \ & \forall \bx \in \cX, \\
\label{eq:U-lin}
& \olr(a,\bx) \leq U_t(a,\bx) \leq \olr(a,\bx) + 2 \epsilon_t(a,\bx) \\
\label{eq:L-lin}
\mbox{and} \qquad & \bL_t(a,\bx) \leq \olbc(a,\bx) \leq \bL_t(a,\bx)
+ 2 \epsilon_t(a,\bx) \mathbf{1}\,.
\end{align}

\subsection{Regret Bound}
\label{sec:th-lin}

We sketch in Appendix~\ref{sec:proof-sketch-lin} below the proof of the following result.
We use the $\epsilon_s(a,\bx)$ indicated by Lemma~\ref{lemma:linucb}.

\begin{theorem}
\label{th:lin}
In the setting of Appendix~\ref{app:setting-lin},
we consider the adaptive policy of Box~C of Appendix~\ref{app:policy-lin}.
We set a confidence level $1-\delta \in (0,1)$
and use parameters $\lambda = m \ln(1+T/m)$, a working budget of
$\hat{B}_T = B - b_T$, where
\[
b_T = 2 +
m \bigl( 2 \sqrt{2} \Arrowvert \Theta \Arrowvert + 1 \bigr)
\sqrt{T} \ln \frac{1 + T/m}{\delta/(d+1)}
+ \sqrt{2T \ln \frac{4d}{\delta}} + |\cX| \sqrt{2T \ln \frac{2 T |\cX|}{\delta}}\,,
\]
and $\epsilon_t(a,\bx) = \gamma_{t,\lambda,\delta} \, \bigl\Arrowvert \bphi(a, \bx) \bigr\Arrowvert_{X_{t}^{-1}}$.
Then, provided that $T \geq 2m$ and $B > 2b_T$,
we have, with probability at least $1 - 3\delta$,
\[
\opt(\nu,\olr,\olbc,B) - \sum_{t=1}^T r_t \leq
2 b_T \left( 1 + \frac{\opt(\nu,\olr,\olbc,B)}{B} \right)\,.
\]
\end{theorem}

The order of magnitude of the regret bound, in terms of $T$, $m$, and $B$ is
\[
\frac{\opt(\nu,\olr,\olbc,B)}{B} \, m \sqrt{T} \ln T\,.
\]
This matches the regret bound achieved by \citet[Theorem~3]{Agrawal2016LinearCB},
except that the latter reference required a budget $B$ of order $T^{3/4}$ up to logarithmic terms, while
we relax the budget amount to $B \geq 2b_T$, i.e., $B$ of order $\sqrt{T}$ up to logarithmic terms.

Also, and more importantly, we provide a natural strategy in Box~C, whose parameters are easy to tune,
while the fully adaptive algorithm underlying \citet[Theorem~3]{Agrawal2016LinearCB}
has to estimate a critical parameter $Z$ to trade off rewards and costs (the equivalent of our
$\hat{\bbeta}^{\budg,\star}_t$ dual optimal variable). This $Z$ should be of order $\opt(\nu,\olr,\olbc,B)/B$
and $\sqrt{T}$ initial rounds of the strategy underlying \citet[Theorem~3]{Agrawal2016LinearCB}
are devoted to computing a suitable value of~$Z$.

We also provide, in the analysis of Appendix~\ref{sec:proof-sketch-lin},
a rigorous treatment of the use of the no-op action $\anull$.

However, the main advantage of \citet[Theorem~3]{Agrawal2016LinearCB} over
Theorem~\ref{th:lin} above lies in the absence of finiteness restriction on the context set $\cX$,
which we have to (somewhat artificially) introduce to ensure that the linear program of Phase~2 of
the adaptive policy of Box~C is tractable.

\subsection{Proof Sketch of Theorem~\ref{th:lin}}
\label{sec:proof-sketch-lin}

We use the same $\hat{~}$ conventions as in Appendix~\ref{app:proof-nu-unknown}.
The main (but rather minor) changes with respect to the proofs of Appendices~\ref{app:proof-nu-known}
and~\ref{app:proof-nu-unknown} are specifically underlined below.
The reason why it is handy to consider instead a lower-confidence bound on the vector costs
is to be found in Step~4 below.

\paragraph{Step~1.} The first step corresponds to
Lemma~\ref{lemma:linucb} above, together with the introduction of the bound $E_T$.
Given that we pick $\lambda = m \ln(1+T/m) \geq 1 \geq 1$, we get the following counterpart of~\eqref{eq:defET},
by replacing $\kappa$ by~$1$:
\[
2 \sum_{t=2}^T \varepsilon_{t-1}(a_t,\bx_t) \Inotanull{t}
\leq 2 \gamma_{T,\lambda,\delta}
\sqrt{2m T \ln \biggl(1+\frac{T}{\lambda m}\biggr)} \defeq E_T \,.
\]
Substituting the value of $\lambda$ and
upper bounding $\gamma_{T,\lambda,\delta}$ by
\[
\gamma_{T,\lambda,\delta} \leq
\left( \Arrowvert \Theta \Arrowvert + \frac{1}{4} \right) \sqrt{m \ln \frac{1 + T/m}{\delta/(d+1)}} \,,
\]
we get
\[
E_T \leq m \bigl( 2 \sqrt{2} \Arrowvert \Theta \Arrowvert + 1 \bigr)
\sqrt{T} \ln \frac{1 + T/m}{\delta/(d+1)}\,.
\]

\paragraph{Step~2.}
There are three main outcomes of Step~2 (see the summary in Appendix~\ref{app:recap-step2}).
Up to considering the new $U_t(a,\bx)$ and $\bL_t(a,\bx)$
in lieu of the $r(a,\bx)\,U_t(a,\bx)$ and $\bc(a,\bx)\,U_t(a,\bx)$, respectively, we
have the following counterparts to~\eqref{eq:kkt-complementary-slackness-budget},
\eqref{eq:argmax-ucb-main-simple} and~\eqref{eq:property-strong-duality}.
For each $t \geq 2$ such that
the cost constraints of Phase~0 of the adaptive policy of Box~C are not violated
and the optimization problem $\opt\bigl(\hat{\nu}_t, U_{t-1}, \bL_{t-1}, \hat{B}_T\bigr)$ is to be solved,
there exists a vector $\hat{\bbeta}^{\budg,\star}_t \geq \mathbf{0}$ such that
the complementary slackness condition of KKT reads
\begin{equation}
\label{eq:kkt-complementary-slackness-budget-LIN}
\bigl( \hat{\bbeta}^{\budg,\star}_t \bigr)^{\, \transp}
\E_{\bX \sim \hat{\nu}_t} \! \left[ \sum_{a \in \cA} \bL_{t-1}(a, \bX) \, \hat{p}_{t, a}(h_{t-1},\bX) \right]
=  \frac{\hat{B}_{T}}{T} \, \bigl( \hat{\bbeta}^{\budg,\star}_t \bigr)^{\, \transp} \mathbf{1}\,,
\end{equation}
the policy $\hat{\bp}_t(h_{t-1},\,\cdot\,)$ satisfies
\begin{equation}
\label{eq:property-strong-duality-LIN}
\hat{\bp}_t(h_{t-1},\bx) \in \argmax_{\bq \in \cP(\cA)}
\sum_{a \in \cA} \Bigl( U_{t-1}(a,\bx) - \bigl(\hat{\bbeta}^{\budg,\star}_t\bigr)^{\, \transp} \bL_{t-1}(a, \bx) \Bigr)\, q_{a} \,,
\end{equation}
and the value of the optimization problem is larger than
\begin{equation}
\label{eq:argmax-ucb-main-simple-LIN}
\opt\bigl(\hat{\nu}_t, U_{t-1}, \bL_{t-1}, \hat{B}_T\bigr)
\geq \hat{B}_{T} \bigl( \hat{\bbeta}^{\budg,\star}_t \bigr)^{\, \transp} \mathbf{1}\,.
\end{equation}

\paragraph{Step~3.}
The uniform deviation argument~\eqref{eq:tobeimproved}, formulated equivalently as~\eqref{eq:step0-HAz-f},
still holds, on an event referred to as $\cEnu$.
Also, we assumed that rewards $r_t$ and cost vectors $\bc_t$ are bounded in $[0,1]$ and
$[0,1]^d$, respectively.
Several applications of the Hoeffding-Azuma inequality, together with a union bound,
show that there exist an event $\hat{\cE}_{\haz,\delta}$ of probability at least $1-\delta$
such that, simultaneously, various high-probability controls
similar to~\eqref{eq:HAz1-unknown}--\eqref{eq:HAz4-unknown} hold. We do not rewrite them explicitly.

On the intersection $\hat{\cE}_{\haz,\delta} \cap \cE_{\lin,\delta} \cap \cEnu$, we have
\begin{align*}
\lefteqn{\sum_{t=1}^T \bc_t} \\
& \leq \sum_{t=1}^T \olbc(a_t,\bx_t) \Inotanull{t} + \sqrt{\frac{T}{2} \ln \frac{4d}{\delta}} \, \mathbf{1} \\
& \leq 1 + \sum_{t=2}^T \bL_{t-1}(a_t,\bx_t) \Inotanull{t} + \left( E_T + \sqrt{\frac{T}{2} \ln \frac{4d}{\delta}} \right) \mathbf{1} \\
& \leq 1 + \sum_{t=2}^T \E_{\bX \sim \nu} \!\left[ \sum_{a \in \cA} \bL_{t-1}(a, \bX) \, \hat{p}_{t,a}(h_{t-1},\bX) \right]
+ \left( E_T + \sqrt{2T \ln \frac{4d}{\delta}} \right) \mathbf{1} \\
&\leq 1 + \sum_{t=2}^T \E_{\bX \sim \hat{\nu}_t} \!\left[ \sum_{a \in \cA} \bL_{t-1}(a, \bX) \, \hat{p}_{t,a}(h_{t-1},\bX) \right]
+ \left( E_T + \sqrt{2T \ln \frac{4d}{\delta}} + |\cX| \sqrt{2T \ln \frac{2 T |\cX|}{\delta}} \right) \mathbf{1}\,,
\end{align*}
where, among others, we used~\eqref{eq:L-lin} and the definition of $E_T$ for the second inequality.
Note that the $E_T$ term was not necessary in Appendices~\ref{app:proof-nu-known}
and~\ref{app:proof-nu-unknown} as we were then using an upper-confidence bound on the vector costs,
obtained thanks to an upper-confidence bound on the conversion rate.
At each round $t \geq 2$, whether the strategy picks $\hat{p}_{t,a}(h_{t-1},\,\cdot\,)$ in
Phase~0 (in which case the left-hand side in the display below equals~$\mathbf{0}$) or Phase~2 (in which case
we have an equality in the display below), it holds that
\[
\E_{\bX \sim \hat{\nu}_t} \!\left[ \sum_{a \in \cA} \bL_{t-1}(a, \bX) \, \hat{p}_{t,a}(h_{t-1},\bX) \right]
\leq \frac{\hat{B}_T}{T} \mathbf{1}\,.
\]
Substituting the value of $\hat{B}_T$, we proved that on $\hat{\cE}_{\haz,\delta} \cap \cE_{\lin,\delta} \cap \cEnu$,
which is an event of probability at least $1-3\delta$,
\[
\sum_{t=1}^T \bc_t \leq (B-1) \mathbf{1}\,.
\]
This shows that on this intersection of events, the adaptive policy of Box~C always resorts to Phase~2.
We will consider this event for the rest of the proof, so that the results of Step~2 may be applied.

We may proceed similarly for rewards, based on~\eqref{eq:U-lin}:
on $\hat{\cE}_{\haz,\delta} \cap \cE_{\lin,\delta} \cap \cEnu$
\begin{align}
\nonumber
\lefteqn{\sum_{t=1}^T r_t} \\
\nonumber
& \geq \sum_{t=1}^T \olr(a_t,\bx_t) \Inotanull{t} - \sqrt{\frac{T}{2} \ln \frac{4}{\delta}} \\
\nonumber
& \geq \sum_{t=2}^T U_{t-1}(a_t,\bx_t) \Inotanull{t} - \left( E_T + \sqrt{\frac{T}{2} \ln \frac{4}{\delta}} \right)  \\
\nonumber
& \geq \sum_{t=2}^T \E_{\bX \sim \nu} \!\left[ \sum_{a \in \cA} U_{t-1}(a, \bX) \, \hat{p}_{t,a}(h_{t-1},\bX) \right]
- \left( E_T + \sqrt{2T \ln \frac{4}{\delta}} \right)  \\
\label{eq:rt-lin-finalbd}
&\geq  \sum_{t=2}^T \underbrace{\E_{\bX \sim \hat{\nu}_t} \!\left[ \sum_{a \in \cA} U_{t-1}(a, \bX) \, \hat{p}_{t,a}(h_{t-1},\bX) \right]}_{=
\opt(\hat{\nu}_t, U_{t-1}, \bL_{t-1}, \hat{B}_T)/T}
- \Biggl( \underbrace{E_T + \sqrt{2T \ln \frac{4}{\delta}} + |\cX| \sqrt{2T \ln \frac{2 T |\cX|}{\delta}}}_{\leq b_T} \Biggr)\,,
\end{align}
where the indicated equality to $\opt\bigl(\hat{\nu}_t, U_{t-1}, \bL_{t-1}, \hat{B}_T\bigr)$ follows from the fact
that on the considered intersection of events, the adaptive policy always resorts to Phase~2. We also
used the piece of notation $b_T$ introduced in the statement of Lemma~\ref{lemma:linucb}.

\paragraph{Step~4.}
We build on~\eqref{eq:property-strong-duality-LIN} as follows. By the existence of $\anull$,
the maximum in~\eqref{eq:property-strong-duality-LIN} can be taken with non-negative parts.
We also substitute the upper confidence control~\eqref{eq:U-lin} and the lower confidence control~\eqref{eq:L-lin}---this
piece of the proof is the very reason why such upper and lower confidence estimates were picked. We get:
on $\hat{\cE}_{\haz,\delta} \cap \cE_{\lin,\delta} \cap \cEnu$, for all $t \geq 2$,
for all $\bx \in \cX$,
\begin{align*}
\lefteqn{\sum_{a \in \cA} \Bigl( U_{t-1}(a,\bx) - \bigl(\hat{\bbeta}^{\budg,\star}_t\bigr)^{\, \transp} \bL_{t-1}(a, \bx) \Bigr)\,
\hat{p}_{t,a}(h_{t-1},\bx)} \\
& = \sum_{a \in \cA} \Bigl( U_{t-1}(a,\bx) - \bigl(\hat{\bbeta}^{\budg,\star}_t\bigr)^{\, \transp} \bL_{t-1}(a, \bx) \Bigr)_+ \,
\hat{p}_{t,a}(h_{t-1},\bx) \\
& \geq \sum_{a \in \cA} \Bigl( U_{t-1}(a,\bx) - \bigl(\hat{\bbeta}^{\budg,\star}_t\bigr)^{\, \transp} \bL_{t-1}(a, \bx) \Bigr)_+ \,
\pi^\star_a(\bx) \\
& \geq \sum_{a \in \cA} \Bigl( \olr(a,\bx) - \bigl(\hat{\bbeta}^{\budg,\star}_t\bigr)^{\, \transp} \olbc(a, \bx) \Bigr)_+ \,
\pi^\star_a(\bx)\,.
\end{align*}
The rest of the proof follows the exact same logic as in Appendices~\ref{app:proof-nu-known} and~\ref{app:proof-nu-unknown}.
By replacing the $\bx$ by a random variable $\bX \sim \hat{\nu}_t$ and integrating, we have,
on $\hat{\cE}_{\haz,\delta} \cap \cE_{\lin,\delta} \cap \cEnu$, that for all $t \geq 2$,
\begin{align*}
\lefteqn{\E_{\bX \sim \hat{\nu}_t} \left[\sum_{a \in \cA} \Bigl( U_{t-1}(a,\bx) -
\bigl(\hat{\bbeta}^{\budg,\star}_t\bigr)^{\, \transp} \bL_{t-1}(a, \bx) \Bigr) \, \hat{p}_{t,a}(h_{t-1},\bx) \right]} \\
& \geq
\E_{\bX \sim \hat{\nu}_t} \left[\sum_{a \in \cA} \Bigl(\olr(a,\bx) -
\bigl(\hat{\bbeta}^{\budg,\star}_t\bigr)^{\, \transp} \olbc(a, \bx) \Bigr)_+ \, \pi^\star_{a}(\bx) \right] \\
& \geq \E_{\bX \sim \nu} \left[\sum_{a \in \cA} \Bigl(\olr(a,\bx) -
\bigl(\hat{\bbeta}^{\budg,\star}_t\bigr)^{\, \transp} \olbc(a, \bx) \Bigr)_+ \, \pi^\star_{a}(\bx) \right]
- |\cX| \sqrt{\frac{t}{2} \ln \frac{2 T |\cX|}{\delta}} \,,
\end{align*}
where the second inequality follows by~\eqref{eq:step0-HAz-f},
which is legitimately applied thanks the fact that the sum over $a$ in the right-hand side takes values in $[0,1]$,
given the non-negative parts and the fact that $\olr(a,\bx) \leq 1$ by definition.
Now, with~\eqref{eq:kkt-complementary-slackness-budget-LIN} and the definition of $\pi^\star$:
on $\hat{\cE}_{\haz,\delta} \cap \cE_{\lin,\delta} \cap \cEnu$,
for all $t \geq 2$,
\begin{align*}
\lefteqn{\frac{\opt(\hat{\nu}_t, U_{t-1}, \bL_{t-1}, \hat{B}_T)}{T} -
\frac{\hat{B}_{T}}{T} \, \bigl( \hat{\bbeta}^{\budg,\star}_t \bigr)^{\, \transp} \mathbf{1}} \\
& = \E_{\bX \sim \hat{\nu}_t} \left[\sum_{a \in \cA} \Bigl( U_{t-1}(a,\bx) -
\bigl(\hat{\bbeta}^{\budg,\star}_t\bigr)^{\, \transp} \bL_{t-1}(a, \bx) \Bigr) \, \hat{p}_{t,a}(h_{t-1},\bx) \right] \\
& \geq \E_{\bX \sim \nu} \left[\sum_{a \in \cA} \Bigl(\olr(a,\bx) -
\bigl(\hat{\bbeta}^{\budg,\star}_t\bigr)^{\, \transp} \olbc(a, \bx) \Bigr)_+ \, \pi^\star_{a}(\bx) \right]
- |\cX| \sqrt{\frac{t}{2} \ln \frac{2 T |\cX|}{\delta}} \\
& \geq \frac{\opt(\nu,\olr,\olbc,B)}{T} - \frac{B}{T} \, \bigl( \hat{\bbeta}^{\budg,\star}_t \bigr)^{\, \transp} \mathbf{1}
- |\cX| \sqrt{\frac{t}{2} \ln \frac{2 T |\cX|}{\delta}}\,.
\end{align*}
Based on these inequalities, we have
\begin{align*}
\lefteqn{\opt(\nu,\olr,\olbc,B) - \sum_{t=2}^T \frac{\opt(\hat{\nu}_t, U_{t-1}, \bL_{t-1}, \hat{B}_T)}{T}} \\
& \leq |\cX| \sqrt{2T \ln \frac{2 T |\cX|}{\delta}} + 1 + \sum_{t=2}^T
\frac{B - \hat{B}_{T}}{T} \, \bigl( \hat{\bbeta}^{\budg,\star}_t \bigr)^{\, \transp} \mathbf{1} \\
& \leq |\cX| \sqrt{2T \ln \frac{2 T |\cX|}{\delta}} + 1 + \frac{B - \hat{B}_{T}}{\hat{B}_{T}} \sum_{t=2}^T
\frac{\opt(\hat{\nu}_t, U_{t-1}, \bL_{t-1}, \hat{B}_T)}{T}\,,
\end{align*}
where we substituted~\eqref{eq:argmax-ucb-main-simple-LIN} for the second inequality.
By a case analysis, we finally proved that
on $\hat{\cE}_{\haz,\delta} \cap \cE_{\lin,\delta} \cap \cEnu$,
\begin{multline*}
\opt(\nu,\olr,\olbc,B) - \sum_{t=2}^T \frac{\opt(\hat{\nu}_t, U_{t-1}, \bL_{t-1}, \hat{B}_T)}{T} \\
\leq \left( \frac{B}{\hat{B}_{T}} - 1 \right) \opt(\nu,\olr,\olbc,B)
+ \underbrace{1 + |\cX| \sqrt{2T \ln \frac{2 T |\cX|}{\delta}}}_{\leq b_T} \, .
\end{multline*}
The proof is concluded by combining the inequality above with
the bound~\eqref{eq:rt-lin-finalbd} on cumulative rewards:
\[
\opt(\nu,\olr,\olbc,B) - \sum_{t=1}^T r_t \leq 2b_T
+ \biggl( \underbrace{\frac{B}{\hat{B}_{T}} - 1}_{\leq 2b_T} \biggr) \opt(\nu,\olr,\olbc,B)\,,
\]
where the bound on $B/\hat{B}_{T} - 1$
is obtained with the same techniques and similar conditions as for~\eqref{eq:BBT}.

\subsection{Reminder: Algorithm 1 from~\citet{Agrawal2016LinearCB}}
\label{sec:algo-AD}

We recall (and actually slightly adapt to our setting) the
adaptive policy of~\citet{Agrawal2016LinearCB} titled Algorithm~1 therein.
We describe it in Box~D.
One of the adaptations is to state it with general upper-confidence bonuses $\varepsilon_{s}(a,\bx) > 0$.
As in~\citet{Agrawal2016LinearCB}, who proceed as in the proof of Lemma~\ref{lemma:linucb}, we will use the same values for
$\varepsilon_{s}(a,\bx)$ as in Theorem~\ref{th:lin}. The same comment applies to $\lambda$.
Another adaptation is that we specified the online convex optimization algorithm to be used
and picked a simple strategy (instead of other possible choices discussed in \citet{Agrawal2016LinearCB}),
namely, the projected gradient descent introduced by~\citet{Zink03}. The latter relies on a learning rate $\eta > 0$.
The drawback of the projected gradient descent is however that its dependency in the ambient dimension is suboptimal.

The final parameter of the adaptive policy of Box~D is a parameter $Z$ to
trade off between rewards and costs. A recommended choice is, for instance, $Z = \opt(\nu,\olr,\olbc,B)/B$,
the issue being that, of course, the latter value is unknown. In the simulation
study of Appendix~\ref{app:simu}, we will provide a good value of $Z$ to the adaptive policy,
even though \citet{Agrawal2016LinearCB} introduce a variant with a preliminary exploration
phase meant to provide in an automatic way such a good value for $Z$.

\begin{figure}[p]
\begin{nbox}[title={Box~D: Adaptation of Algorithm 1 from~\citet{Agrawal2016LinearCB}}]
\textbf{Parameters:} regularization parameter $\lambda > 0$;
trade-off parameter $Z$ between reward and costs;
upper-confidence bonuses $\varepsilon_{s}(a,\bx) > 0$, for $s \geq 1$
and $(a,\bx) \in \bigl( \cA \setminus \{\anull\} \bigr) \times \cX$;
learning rate $\eta > 0$. \medskip

\textbf{Round} $t=1$: play an arbitrary action $a_1 \in \cA \setminus \{\anull\}$; pick
$\bzeta_{1} = \mathbf{0}$ \medskip

\textbf{At rounds} $t \geq 2$: \smallskip
\begin{enumerate}
\item[\underline{Phase 0}] If $\displaystyle{\sum_{s \leq t-1} \bc_s \leq (B-1) \mathbf{1}}$ is violated, then
$a_t = \anull$ \smallskip
\item[\underline{Phase 1}] Otherwise, estimate the parameters by
\begin{align*}
& \bmu_{t-1} = X_{t-1}^{-1} \sum_{s=1}^{t-1} \bphi(a_s, \bx_s) \, r_s \\
\mbox{and} \qquad
& \hat{\btheta}_{t-1,i} = X_{t-1}^{-1} \sum_{s=1}^{t-1} \bphi(a_s, \bx_s) \, c_{s,i} \\
\mbox{where} \qquad
& X_{t} = \sum_{s=1}^{t} \bphi(a_s, \bx_s) \bphi(a_s, \bx_s)^{\transp} + \lambda \id_m
\end{align*}
Define the expected reward function $\hat{r}$ and cost function $\hat{\bc}_{t-1} = \bigl(\hat{c}_{t-1,i}\bigr)_{1 \leq i \leq d}$ as
\begin{align*}
\forall a \in \cA \setminus \{\anull\}, \ \ \forall \bx \in \cX, \qquad\qquad
& \hat{r}_{t-1}(a,\bx) = \bphi(a,\bx)^{\transp} \bmu_{t-1} \\
\mbox{and} \quad \forall 1 \leq i \leq d, \quad\qquad
& \hat{c}_{t-1,i} = \bphi(a,\bx)^{\transp} \btheta_{t-1,i}
\end{align*}
Build the upper-confidence expected reward function $U_{t-1}$
and the lower-confidence expected vector-cost function $\bL_{t-1}$ as
\begin{align*}
\forall a \in \cA & \setminus \{\anull\}, \ \ \forall \bx \in \cX, \\
& U_{t-1}(a,\bx) = \hat{r}_{t-1}(a,\bx) + \epsilon_{t-1}(a,\bx) \\
& \bL_{t-1}(a,\bx) = \hat{\bc}_{t-1} - \epsilon_{t-1}(a,\bx) \mathbf{1}
\end{align*}
\item[\underline{Phase 2}]
Play
\begin{align*}
a_t \in \argmax_{a \in \cA \setminus \{\anull\}} U_{t-1}(a,\bx) - Z \left( \bzeta_{t-1}^{\transp}  \bL_{t-1}(a,\bx) \right)
\end{align*}
Compute
\begin{align*}
\bzeta_{t} = \proj \Bigl( \bzeta_{t-1} + \eta \bigl( \bc_{t-1} - (B/T) \mathbf{1} \bigr) \Bigr)
\end{align*}
where $\proj$ denotes the Euclidean projection onto
the unit $\ell_1$--ball
\[
\bigl\{ \bzeta \in \R^d : \bzeta \geq \mathbf{0} \ \mbox{and} \ \zeta_1 + \ldots + \zeta_d \leq 1 \bigr\}
\]
\end{enumerate}
\end{nbox}
\end{figure}

\clearpage
\section{Simulation Study}
\label{app:simu}

This appendix reports numerical simulations performed on partially simulated but realistic data (Appendix~\ref{sec:data}),
for the sake of illustration only. We describe (Appendix~\ref{sec:spec-sett-CBwK})
the specific experimental setting of CBwK for a conversion model considered---i.e., the features
available, the parameters of the logistic regression, the reward and cost functions.
A key point is that continuous variables are used to define rewards, costs, and even the
conversion rate, while the adaptive policy of Box~C must discretize these variables to solve
its Phase~2 linear program.
Though the experimental setting introduced is not a setting of linear CBwK, we may still
apply the Box~D adaptive policy (Appendix~\ref{sec:AD-howto}), with the underlying idea
that it provides linear approximations to non-linear reward and cost functions.
We carefully explain how the hyper-parameters picked (Appendix~\ref{sec:hyperp})
before providing and discussing the outcomes of the simulations (Appendix~\ref{sec:outcomes-simu-app}).
The main outcome is that, as expected, the ad hoc adaptive policy of Box~C outperforms
the adaptive policy of Box~D, which essentially linearly approximates non-linear rewards and costs.
We end with a note (Appendix~\ref{sec:comput-time}) on the computation environment and time.

\subsection{Data Preparation and Available Contexts}
\label{sec:data}

The underlying dataset for the simulations is the standard ``default of credit card clients''
dataset of~\citet{UCI2016DefaultCR}, initially provided by~\citet{UCIarticle}.
(It may be used under a Creative Commons Attribution 4.0 International [CC BY 4.0] license.)
This dataset is originally for comparing algorithms predicting default probability of credit card clients. It includes some socio-demographic data, debt level, and payment/default history of the clients. It also includes a target measuring whether the client will default in the future (1-month ahead).
We transform it to match our motivating application of market share expansion for loans, described in Appendix~\ref{app:industrialmotivation}.
To do so, we consider each line of the dataset as a loan application. We then discard some variables (e.g., the target) and create new ones (requested amount, standard interest rate offered,
risk score).
Below, we begin with describing the variables that we keep as they are and explain next how we created the additional variables,
based on existing ones.

\paragraph{Variables retrieved.}
The variable \var{Age} provides the age of a given client at the time of the loan application, in years. We
discretize it into 5 levels with similar numbers of loan requests in each level. The cutoffs for each level are 27, 31, 37, and 43, respectively. This gives rise to a variable referred to as
\var{Age--discrete}.

The variable \var{Education} reports the education level of a client; in the data there are 4 levels:
``others'' (level 1, representing $2\%$ of clients), ``high school'' (level 2, with a share of $16\%$),
``university degree'' (level 3, with $47\%$), and ``graduate school degree'' (level~4, with $35\%$).

Finally, the variable \var{Marital status} provides the marital status of a client: ``others'' (level~1,
accounting for $1.3\%$ of clients), ``single'' (level~2, for $53.3\%$ of clients), and ``married'' (level~3,
for $45.4\%$).

\paragraph{Variables created based on existing ones.}
We create a variable \var{Requested amount}, in dollars ({\$}),
based on a variable provided in the data set that measures the current debt level, in dollars, of the clients:
we do so by multiplying the debt level by $0.2$. We cap the value of \var{Requested amount} to 100K{\$}.
We then discretize \var{Requested amount} into 5 levels with similar numbers of loan requests in each level; the obtained variable is referred to
as \var{Requested amount--discrete}.
The cutoffs for each level are 10K{\$}, 20K{\$}, 36K{\$}, and 54K{\$}, respectively.

For the final two variables, \var{Standard interest rate} and \var{Risk score}, we
first build a probability-of-default model with the variables from the raw database as predictors and the occurrence of a default
within the next month as a target. This probability-of-default
model is based on XGBoost (\citet{Chen2016XGBoost}), run
with no penalization, depth $3$, learning rate $0.01$, subsample parameter $0.8$, min child weight $10$,
and number of trees $1,\!176$. We only set the number of trees by cross validation, while the rest of the hyper-parameters were set arbitrarily.
As the default target is on a credit card, the predicted default rate seems high compared to what we deem as typical default rates on loans. We therefore
divide the predicted probability of default by factor of $4$ and cap this probability to $20\%$. This gives us a working variable called \var{PD},
for probability of default.

We build a \var{Risk score} rating the risk of a client's default, with 5 levels,
coded from A (level~1) to E (level~5), where E represents the highest risk.
It is created based on the $20\%$ - $40\%$ - $60\%$ - $80\%$ quantiles of \var{PD}.

We finally set the \var{Standard interest rate} variable as $0.9$ times the \var{PD}, with a maximal value of $18\%$ and a minimum one of $1\%$.
This constitutes an oversimplification of risk-based pricing, since we do not take into account any loss given default; but we do not
have enough information in the dataset to do so, which is why we basically assume that the loss given default is constant. Then, theory
has it that the \var{Standard interest rate} can be considered proportional to \var{PD}, and we carefully picked the factor $0.9$
to get realistic values. In the dataset, the thus created \var{Standard interest rate} variable exhibits an
average of $4.9\%$, with median $3.3\%$.
An important note is that this variable is continuous and does not take finitely many values, while
the setting described in Section~\ref{sec:protocol} imposes such a restriction
(the linear setting of Section~\ref{sec:LinearCBwL} does not require it).
We discuss below how we take this fact into account: by only using this variable for the conversion model
(i.e., in Phase~1 of the adaptive policy of Box~C),
but not to pick actions (i.e., not in Phase~2 of the adaptive policy of Box~C).

Lastly, we filter the database to remove the outliers: the lines for which \var{Standard interest rate} times \var{Requested amount} is larger than 10K,
which happens for $133$ clients. Our final database contains $29,\!867$ loan applications, out of which we will
bootstrap $T=50,\!000$ applications.

\paragraph{Additional comments.}
Note that for \var{Requested amount} and \var{Age}, the discretizations performed
aim to get five balanced classes; however, as some requests are with some specific boundary values,
we do not get exact equal distributions over the classes.

All the parameters, constants, and cutoff/filter thresholds used in this data preparation step were decided arbitrarily and
were not based on any real information. The context variables here were also selected somewhat arbitrarily (based on their availability),
and solely for illustration purposes. In reality, the variables that can be used for commercial discounts
need to comply with relevant laws, regulations and company's internal compliance rules.

\paragraph{Summary.}
The context $\bx$ for a given client thus contains the following variables:
\var{Age},
\var{Age-discrete}, \var{Education}, \var{Marital status},
\var{Requested amount}, \var{Requested amount--discrete},
\var{Risk score}, and \var{Standard interest rate}.
Categorical variables are hot-one encoded via binary variables.

\subsection{Specific Setting of CBwK for Logistic Conversions}
\label{sec:spec-sett-CBwK}

We recall that our aim is to provide simulations matching the motivating example of
market share expansion for loans described in~Appendix \ref{app:industrialmotivation}.
We take as action set, i.e., as possible discount rates, $\cA = \{\anull, 0.1, 0.2, 0.35, 0.55, 0.8\}$.

\paragraph{Features.} The feature vectors $\varphi_{\conv}(a,\bx)$ used are composed of only some of the variables
defining the context $\bx$, namely,
\var{Age-discrete}, \var{Education}, \var{Marital status},
\var{Requested amount--discrete},
\var{Risk score}, and
\var{Standard interest rate} (we recall that this variable is not discrete
but will not use it in the linear program of Phase~2), with the addition
of a new variable called \var{Final interest rate} equal to
the discounted standard interest rate offered, i.e.,
\var{Final interest rate} = \var{Standard interest rate} $\times (1-a)$.

\paragraph{Reward and vector cost functions.}
We use the following (normalized) reward and cost functions, inspired from Appendix \ref{app:industrialmotivation}.
We set a common duration for all loans, say, 2~years, so that the requested amount equals the outstanding.
For all $a \in \cA \setminus \{\anull\}$ and $\bx$,
\begin{equation}
\label{eq:rbc-exp}
r(a,\bx) = x_{\am}/M_{\am} \qquad \mbox{and} \qquad
\bc(a,\bx) = \bigl( a/M_{\disc}, \, x_{\irs} x_{\am} / M_{\irs,\am} \bigr)
\end{equation}
where $x_{\am}$ and $x_{\irs}$ denote the components of the context $\bx$ containing the \var{Requested amount}
and \var{Standard interest rate}, respectively, and the normalization factors
equal $M_{\am} = 10^5$, $M_{\disc} =7$, and $M_{\irs,\am} = 9,\!996$.

Note that the definition of the first cost here is different from that in Appendix~\ref{app:industrialmotivation}:
we use $a/M_{\disc}$ here instead of $\mathds{1}_{\{ a \ne 0 \}}$ there. We do so not to disfavor the Box~D policy,
see details in Appendix~\ref{sec:AD-howto} below.

\paragraph{Conversion rate function $P$.}
We model the conversion rate function~$P$ with the logistic-regression model stated in~\eqref{eq:P-logistic},
with $\varphi = \varphi_{\conv}$,
and only need to provide the numerical value of~$\btheta_\star$, which we do in Table~\ref{table:conversion-coef}.
This model, as well as the Phase~1 learning of $\btheta_\star$ described by~\eqref{eq:besttheta}--\eqref{eq:Wttheta}
in Section~\ref{sec:policy}, holds for possibly non-discrete contexts.

\begin{table}[t]
\caption{Coefficients picked for the logistic-regression model of $P$.}
\label{table:conversion-coef}
\centering
\begin{tabular}{lccccc}
\toprule
Intercept                              &  \multicolumn{5}{c}{$0.8177$} \\
\midrule
Continuous variable                   &  \multicolumn{5}{c}{Single coefficient} \\
\var{Final interest rate}     &  \multicolumn{5}{c}{$-13.1101$} \\
\midrule
Discrete variables                  &  \multicolumn{5}{c}{Coefficients for each level} \\
&  Level 1  &  Level 2   &  Level 3     &  Level 4     &  Level 5 \\
\cline{2-6} \\
\var{Risk score}                 &  $-0.3045$       &  $-0.0383$       &   $0.0515$  &  $0.1261$   &  $0.1636$ \\
\var{Requested amount--discrete}       &  $0.7093$         &  $0.4703$         &  $0.1113$    &  $-0.2748$  &  $-1.0179$ \\
\var{Age--discrete}                          &  $-0.1837$        &  $-0.1392$       & $-0.0476$   &  $0.1096$  &   $0.2592$ \\
\var{Education}                 &  $0.1836$         & $0.0126$        & $-0.0896$    & $-0.1084$   &   \\
\var{Marital status}           &  $0.0799$        &  $0.0102$      &    $-0.0918$  &  & \\
\bottomrule
\end{tabular}
\end{table}

The numerical values picked for $\btheta_\star$ were so in some arbitrary way, to get somewhat realistic outcomes with a simple model structure.
We imposed monotonicity constraints, as these are most natural: for instance, the conversion rate increases with
the level of \var{Risk score} and \var{Age--discrete} increase,
and decreases with the level of \var{Education}, \var{Requested amount--discrete}, and
\var{Final interest rate}. The coefficients for \var{Marital status} indicates that conversions
are more likely for clients that are single than for married clients.

The average conversion rate in the case $a = 0$ of no discount (i.e,
by replacing \var{Final interest rate} by \var{Standard interest rate}) is around $50\%$.

\paragraph{Adaptive policy: based only on the discrete variables.}
As indicated above, the logistic-regression model and the learning of its parameters
apply to continuous variables.
The restriction that the context set $\cX$ should be finite only came from
Phase~2 of the Box~C adaptive policy, i.e., the linear program---in particular,
for it to be computationally tractable. Here, we thus restrict our attention
to policies that map the discrete variables in $\bx$ to distributions over $\cA$:
policies that ignore the variables \var{Age}, \var{Requested amount}, and \var{Standard interest rate}.
For the first two variables, they may use their discretized versions
\var{Age--discrete} and \var{Requested amount--discrete}.
For \var{Standard interest rate}, given how it was constructed,
\var{Risk Score} appears as its discretized version.

The aim of these simulations is to show, among others, that using
discretizations only in Phase~2 is relevant and efficient.

Note that, on the theoretical side, the proof sketches provided in
Sections~\ref{sec:analysis-known} and~\ref{sec:analysis-unknown}
reveal that the errors $\varepsilon_t(a,\bx)$ for learning $\theta_\star$ and $P$,
obtained as outcomes of the first step of the analyses, are carried over in the subsequent steps,
where the optimization part is evaluated. Using discretizations only in Phase~2 does therefore not
come at the price of loosing theoretical guarantees.

\subsection{Consideration of the Box~D Adaptive Policy for Linear CBwK}
\label{sec:AD-howto}

In these experiments, we also consider the Box~D adaptive policy of Appendix~\ref{sec:algo-AD},
which was introduced by~\citet{Agrawal2016LinearCB} in a different setting.
To be as fair as possible to this adaptive policy,
we do so with the extended features $\varphi_{\lin}(a,\bx)$ consisting
of the features $\varphi_{\conv}(a,\bx)$ described above and three additional components:
the discount $a$,
the \var{Requested amount} $x_{\am}$, and
the product $x_{\irs} x_{\am}$ of the \var{Requested amount}
by the \var{Standard interest rate}. Actually, to ensure that $\varphi_{\lin}(a,\bx) \in [0, 1]^m$, the last two components added are normalized:
we rather use $x_{\am}/M_{\am}$ and $ x_{\irs} x_{\am} / M_{\irs,\am}$. The reward and vector cost functions introduced in~\eqref{eq:rbc-exp}
are linear in $\varphi_{\lin}(a,\bx)$. Even better, each component of $r(a,\bx)$ and $\bc(a,\bx)$
is given directly by a component of $\varphi_{\lin}(a,\bx)$---an extremely simple
linear dependency on $\varphi_{\lin}(a,\bx)$.

However, the expected reward and cost functions
\[
\olr(a,\bx) = r(a,\bx) \, P(a,\bx) \qquad \mbox{and} \qquad \olbc(a,\bx) = \bc(a,x)\,P(a,\bx)\,,
\]
which are the ones that should be linear in $\varphi_{\lin}(a,\bx)$ according to the
setting described in Appendix~\ref{app:setting-lin}, are not linear in these features.
This is due to the $P(a,\bx)$ terms, which are given by logistic regressions.

The Box~D adaptive policy is therefore disadvantaged. This is even more true as it models rewards and costs independently, while they are coupled through conversions.
We nonetheless consider this linear-modeling policy because a typical justification for linear approximations is that they offer
a typical and efficient first-stage approach to possibly complex problems.
Another reason was the desire to have some competitor to our policy in the simulations, and
Box~D adaptive policy was an easy-to-implement strategy---unlike the policies by
\citet{Badanidiyuru2014ResourcefulCB} and \citet{Agrawal2016AnEA}, which rely on considering finitely many benchmark policies.

All in all, we report
the performance of the Box~D policy as well in our experiments, though, as expected,
the ad-hoc Box~C policy outperforms it.

\subsection{Hyper-parameters Picked}
\label{sec:hyperp}

We actually set the hyper-parameters based on the budget $B$, and therefore, first explain how we set its possible values.
It turns out that setting $B \geq 3,\!650$ is equivalent to not imposing any constraint, while setting $B \geq 2,\!900$ is equivalent
to no second budget constraint. To get meaningful results, we therefore picked $B = 1,\!600$ and $B = 2,\!200$
as the two possible values for $B$. We now describe how we tune each of the two adaptive policies considered.

\paragraph{Hyper-parameters common to the two adaptive policies.}
We take $T=50,\!000$ clients in the experiments, by bootstrapping them from the enriched dataset prepared in Appendix~\ref{sec:data}. We
set initial $50$ rounds as a warm start for the sequential logistic regression and sequential linear regression
carried out in Phase~1 of the adaptive policies.

Both adaptive policies use upper-confidence bonuses $\varepsilon_{s}(a,\bx)$, which are roughly of the form
(considering $\lambda$ as a constant)
\[
C \bigl( 1 + \ln s \bigr) \bigl\Arrowvert \bphi(a, \bx) \bigr\Arrowvert_{X_{s}^{-1}}\,,
\]
where the matrix $X_s$ was defined in Box~C; for simplicity,
we set $\kappa = 1$, so that the matrices $V_s$ of Lemma~\ref{lemma:pbound_conversion} and $X_s$ are equal,
which explains the common form of the upper-confidence bonuses $\varepsilon_{s}(a,\bx)$.
The hyper-parameter $C$ controls the exploration: the higher $C$, the more exploration. We report in the simulations
the results achieved for $C$ in the range $\{0.025, 0.1, 0.3\}$. That range was
set so that at round $s = 51$, which is the first round after the warm start, the
$\varepsilon_{51}(a,\bx) $ take values around $0.05$, $0.3$, and $0.9$, respectively.

For simplicity, we set $B_T = B$ as a working budget.

\paragraph{Hyper-parameters for the Box~C adaptive policy.}
We feed this adaptive policy with a good value of $\lambda$, namely, $\lambda = 0.0129$. We obtained it
by cross-validation on an independent $T$--sample of data, using the Phase~1 estimation. In the $T$--sample for estimating $\lambda$,
at each round $s$, we take action from the optimal static policy and use the associated conversion $y_s$ as target for estimation.
We omit the projection step in Phase~1 by considering that a large enough set $\Theta$ was picked.

\paragraph{Hyper-parameters for the Box~D adaptive policy.}
As discussed in Appendix~\ref{sec:algo-AD},
we set $Z = \opt(\nu,P,B)/B$, that is, $Z = 5.16$ for $B = 1,\!600$ and  $Z =3.87$ for $B = 2,\!200$.
We also set $\lambda  = 0.2452$ for $B = 1,\!600$ and $\lambda  = 0.2765$ for $B = 2,\!200$.
These values were obtained as weighted averages:
the sum of $0.5$ times the optimal $\lambda$ for rewards and $0.25$ times the optimal $\lambda$ for each of the two cost components.
These optimal $\lambda$s for rewards and cost components were set by cross validation on an independent $T$--sample;
with actions $a_s$ taken at each round $s$ from the optimal static policy and associated rewards $r_s$ and costs $\bc_s$ as targets.

Finally, the learning parameter $\eta$ was selected in the range $\{0.005, 0.01, 0.05, 0.1, 0.2\}$.
We did so given the other choices, by picking in hindsight the $\eta$ with best performance; this of course,
just like the clever choice of $Z$, should give an advantage to the Box~D adaptive policy.
Namely, when $B = 1,\!600$, for $C$ equal $0.025$, $0.1$, and $0.3$, we selected $\eta$ equal to $0.05$, $0.01$, and $0.1$, respectively; and when $B = 2,\!200$, we selected $\eta$ equal to $0.01$, $0.01$, and $0.05$, respectively.
When performing these retrospective choices, we however noted that
the performance was not significantly impacted by the choice of $\eta$.

\subsection{Outcomes of the Simulations}
\label{sec:outcomes-simu-app}

We were limited by the computational power (see Appendix~\ref{sec:comput-time})
and could only perform $10$ simulations for each pair of $B \in \{ 1,\!600, \, 2,\!200 \}$
and $C \in \{ 0.025, \, 0.1, \, 0.3 \}$. We report averages (strong lines) as well as
$\pm 2$ times the standard errors (shaded areas).

Figure~\ref{fig:num-res} reports, in the first line of graphs, the regret achieved
with respect to what achieves the optimal static policy, i.e.,
\[
t \longmapsto \frac{t}{T} \opt(\nu,P,B) - \sum_{s=1}^t r_{s}\,,
\]
where $\opt(\nu,P,B)$ is larger than $8,\!000$ for both values of~$B$.
This regret can take negative or positive values, but in expectation, it is non-negative.
This is not immediately clear from the figure, which reports the empirical averages of the regret over $10$ runs:
these empirical averages are sometimes negative, but they always lie in confidence intervals containing the value~$0$.

The figures also reports, in the second and third lines, the difference between a constant linear increase
of the costs (between a $0$ initial cost and a final $B$ cost) and the costs actually
achieved by the adaptive policies. I.e., these graphs report the averages and standard errors
of the following quantities: for each cost component $i \in \{1,2\}$,
\[
t \longmapsto \sum_{s=1}^t c_{s,i} - \frac{t}{T} B\,;
\]
by design, the difference above must be non-positive.
The second line of Figure~\ref{fig:num-res} deals with the first cost component,
and its third line reports the results for the second cost component.

\begin{figure}[p]
\vspace{-1.5cm}

\hspace{-2.8cm} \includegraphics[width=1.35\textwidth]{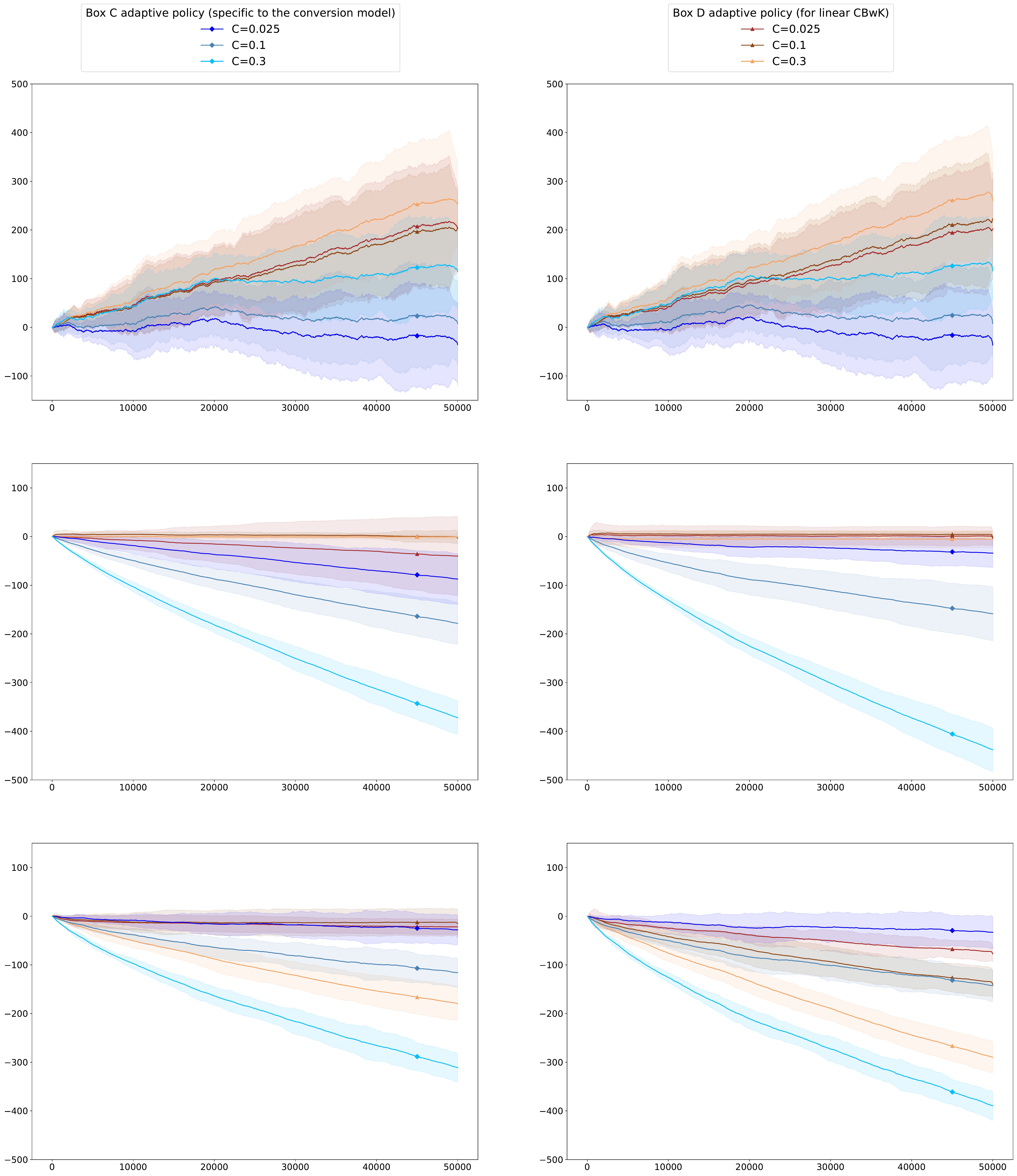} \\ \ \\
\caption{\label{fig:num-res} \emph{Averages (\emph{solid lines}) and $\pm 2$ times standard errors (\emph{shaded areas}),
achieved on $10$ runs by the Box~C (\emph{blue}) and Box~D (\emph{orange}) adaptive policies:
of the regret (\emph{first line}), of the difference of the first cost component to $tB/T$ (\emph{second line}), and
of the difference of the second cost component to $tB/T$ (\emph{third line}), by the values of the budget
(B = $1,\!600$ in the \emph{first column}, B = $2,\!200$ in the \emph{second column}). \ \\ \ \\}}
\end{figure}

The experiments reveal that while both adaptive policies seem to achieve sublinear regret,
the Box~D adaptive policy, which is suited for the CBwK setting for a conversion model, performs better
than the Box~C adaptive policy in terms of rewards: it achieves a smaller, sometimes negative, regret.
In terms of costs, we globally see the same trend, with, for a given value of $C$, the Box~D adaptive policy
suffering smaller costs than the Box~C adaptive policy while achieving higher rewards. This hints at
a better use of the discounts.

The hyper-parameter~$C$ has an interesting impact: the lower $C$, the lower the regret (the higher the rewards)
and the lower the costs. Rewards and costs go hand in hand: for a given adaptive policy,
higher rewards are associated with higher costs.

\subsection{Computation Time and Environment}
\label{sec:comput-time}

As requested by the NeurIPS checklist, we provide details on the computation time and environment.
Our experiments were ran on the following hardware environment: no GPU was required, CPU is $2.7$ GHz Quad-Core with total of 8 threads and RAM is 16 GB 2133 MHz LPDDR3. We ran $5$ simulations with different seeds on parallel each time.
In the setting and for the data described above,
it took us $8$ hours for each such bunch of $5$ runs of the adaptive policy of Box~C, and $1.5$ hours for Box~D.

\end{document}